%% file: 00main.tex
\documentclass{article}

\usepackage[english]{babel}

\usepackage[letterpaper,top=2cm,bottom=2cm,left=3cm,right=3cm,marginparwidth=1.75cm]{geometry}
\usepackage{amsmath}
\usepackage{amssymb}
\usepackage{amsthm}
\usepackage{amsfonts}
\usepackage{algorithm}
\usepackage{algpseudocode}
\usepackage{array}
\usepackage{blindtext}
\usepackage{booktabs}
\usepackage{comment}
\usepackage{forest}
\usepackage{graphicx}
\graphicspath{{images/}}
\usepackage[colorlinks=true, allcolors=blue]{hyperref}

\usepackage{multirow}
\usepackage{listings}
\usepackage{longtable}
\usepackage{siunitx}
\usepackage{subfiles} 
\usepackage[backend=biber, style=numeric-comp, sorting=none]{biblatex}
\usepackage{tikz}
\usetikzlibrary{positioning}

\newtheorem{theorem}{Theorem}
\newtheorem{lemma}[theorem]{Lemma}

\newtheorem{corollary}[theorem]{Corollary}
\newtheorem{definition}[theorem]{Definition} 
\theoremstyle{remark}
\newtheorem{remark}[theorem]{Remark}

\definecolor{panelrule}{gray}{0.55}
\definecolor{filebg}{gray}{0.97}

\lstdefinelanguage{NeuralCertPseudo}{
  morekeywords={PROCEDURE,FUNCTION,INPUT,OUTPUT,RETURN,FOR,EACH,IN,IF,ELSE,
    WHILE,REPEAT,UNTIL,REQUIRE,REJECT,ACCEPT,STORE,WITH,END,TRUE,FALSE},
  sensitive=true,
  morecomment=[l]{//}
}

\lstdefinestyle{neuralcert}{
  language=NeuralCertPseudo,
  basicstyle=\ttfamily\footnotesize,
  keywordstyle=\bfseries,
  commentstyle=\itshape,
  columns=fullflexible,
  keepspaces=true,
  breaklines=true,
  showstringspaces=false,
  numbers=none,
  frame=single,
  float,
  floatplacement=tbp,
  framerule=0.3pt,
  xleftmargin=0.4em,
  xrightmargin=0.4em,
  aboveskip=0.7\baselineskip,
  belowskip=0.7\baselineskip,
  captionpos=t
}

\lstdefinestyle{certfile}{
  basicstyle=\ttfamily\scriptsize,
  backgroundcolor=\color{filebg},
  breaklines=true,
  columns=fullflexible,
  frame=single,
  rulecolor=\color{panelrule},
  framesep=4pt,
  showstringspaces=false,
  literate={->}{{$\rightarrow$}}2 {>=}{{$\geq$}}2
}

\newcommand{\panel}[1]{%
  \medskip\noindent\textbf{#1}\par\nobreak\smallskip}

\title{NeuralCert: certified computational discovery of extremal mathematical constructions}

\author{%
  Mark Patrick Roeling\thanks{%
    Corresponding author: \href{mailto:your.email@domain.nl}{mp.roeling@mindef.nl}.\\
    Code and certified data: \url{https://github.com/mproeling/neuralcert}%
  }\\
  Netherlands Defence Academy; Data Science Centre of Excellence, NL Ministry of Defence
}

\date{}

\begin{document}

\maketitle

\begin{abstract}
Neural networks are becoming popular in solving mathematical problems, but stochastic models do not provide mathematical exactness by themselves. This study introduces a discovery-to-certification framework in which high-dimensional variational trial functions are learned in a compact separable representation, spectrally diagnosed and pruned, and then certified exactly through multimodular evaluation. Exact certification makes the numerical proofs fully explicit and independently verifiable. This framework can be run on a standard personal computer. 

Across three extremal problems, we show that neural optimization can contribute to rigorous mathematics in three distinct ways: by discovering improved constructions, by exposing empirical invariants that lead to proofs, and by revealing optimization barriers whose geometry motivates new analytic or numerical representations.

More broadly, these results suggest a path toward AI-assisted mathematics in which flexible computational discovery and exact certification become complementary components of a single rigorous workflow.
\end{abstract}


\subfile{01introduction}

\subfile{Section1/02results}

\subfile{Section2/02results}

\subfile{Section3/02results}

\section*{Discussion}

\subfile{04discussion-v3}

\section*{Methods}
\subfile{Section1/03methods}

\subfile{Section2/03methods_Delsarte_v5}

\subfile{Section3/03methods}

\subfile{05methods-llm}


\printbibliography[
  heading=bibintoc,
  title={References}
]

\clearpage
\begin{refsection}

\subfile{Section1/Maynard_SuppMethods_v8}

\subfile{Section1/input004asymptotic_v4}

\newpage

\subfile{Section2/04Delsarte_SUPPLEMENTARY_v5}

\newpage
\subfile{Section3/04supplementary-methods}

\newpage

\subfile{06supplementary-compute-package}

\subfile{Section1/supplementary_certificate_example}

\newpage

\subfile{Tables}

\newpage

\printbibliography[
  heading=subbibliography,
  title={Supplementary References}
]

\end{refsection}

\end{document}

%% file: 01introduction.tex
\noindent
Machine learning can contribute to mathematical discovery by identifying latent relations between mathematical objects, guiding search over vast combinatorial or algebraic spaces, and generating candidate constructions or algorithms. Davies \emph{et al.} used learned predictors and attribution methods to expose latent mathematical structure and guide conjecture formation \cite{davies2021advancing}. AlphaTensor showed that reinforcement learning can search large spaces of tensor decompositions and recover exact matrix-multiplication algorithms, including previously unknown constructions \cite{fawzi2022alphatensor}. FunSearch coupled neural generation to an executable evaluator, demonstrating that iterative proposal, evaluation and selection can produce new mathematical constructions \cite{romeraparedes2024funsearch}. Also, neural function approximators can transform a mixed discrete--continuous geometric problem into a differentiable optimization problem \cite{mundinger2025neural}.

Despite these advances, a methodological gap persists for problems in which the unknown object is a function or extremizer in a large or infinite-dimensional space. Such problems are commonly reduced computationally by choosing a finite ansatz,
\begin{equation}
    f(x)=\sum_{j=1}^{m} a_j \phi_j(x),
\end{equation}
and optimizing over the coefficients $a_j$. While this makes the problem tractable, it also imposes an a priori structural restriction,
\begin{equation}
    f^\star \approx \operatorname{span}\{\phi_1,\ldots,\phi_m\},
\end{equation}
so that improvements in optimization accuracy cannot remove approximation error induced by an inadequate representation. However, neural parameterizations offer a complementary strategy: instead of fixing the shape of a candidate solution in advance, they provide flexible differentiable families in which large-scale functional search can be conducted. For example, in geometric colouring a mixed discrete--continuous search space was replaced by a probabilistic differentiable representation amenable to gradient optimization \cite{mundinger2025neural}. 

Here we introduce NeuralCert, a discovery-to-certification framework that deliberately separates the representation used for computational search from the representation used for mathematical proof. 
Flexible neural parameterizations are used to learn one-dimensional factors within a separable trial family, without fixing the polynomial dictionary subsequently used for exact certification. Discovered candidates are subsequently compressed into explicit mathematical representations from which rigorous certificates can be constructed and independently verified. Here, certification means recording a claimed bound together with the exact data needed to recompute it independently, so that the inequality can be confirmed without trusting, or re-running, any part of the discovery pipeline. Across three extremal problems, we show that this separation enables neural optimization to contribute to rigorous mathematics through three distinct modes: discovering improved constructions, exposing structure that leads to analytic results, and identifying optimization barriers that motivate alternative representations.

Since a floating-point neural candidate does not constitute a proof, the discovery candidate is reduced to an independently checkable object \cite{fawzi2022alphatensor,romeraparedes2024funsearch} such that verification establishes the mathematical claim. This idea builds on previous work where neural search produced tensor decompositions which could be checked algebraically \cite{fawzi2022alphatensor}, and candidate programs judged by an explicit executable evaluator \cite{romeraparedes2024funsearch}. These examples suggest that discovery and verification impose fundamentally different computational requirements. Discovery benefits from (over)parameterization, approximation and exploration, whereas verification benefits from low-dimensional structure, exactness and controlled numerical error.

Our central methodological hypothesis is not that neural optimization should replace mathematical proof, but that the representation best suited for discovery need not be the representation best suited for proof. A flexible computational model can search a substantially larger functional space than is convenient for exact symbolic treatment, after which the useful structure can be transferred into an explicit mathematical object. Conversely, repeated failure or instability of such a search can reveal invariants or unsuitable coordinates and thereby guide analytic reformulation \cite{davies2021advancing}. NeuralCert is designed around this asymmetry: discovery is deliberately flexible and expendable, whereas every reported mathematical claim must be supported by a representation that can be reconstructed and verified independently.

\section*{NeuralCert overview}

NeuralCert is a Python framework for neural discovery of candidate functions followed by independent certification and verification. The architecture is deliberately separated into three components; discovery, certification and verification, which do not share numerical evaluation routines. Discovery uses flexible neural parameterizations to explore problem-specific function spaces and exports candidate functions and metadata, typically as NumPy \texttt{.npz} files. Certification is independent of the neural model: no network parameters, optimizer state or training data are required. Instead, the candidate is reconstructed in a problem-specific mathematical representation and converted into an explicit certificate. Verification is implemented as a standalone minimal code path that ingests only this certificate, typically as \texttt{.json}, and independently reconstructs the inequalities underlying the claimed result. Thus, neural optimization serves only as a proposal mechanism; final mathematical claims depend exclusively on independently reconstructed and verifiable proof objects.

%% file: Section1/02results.tex

\section*{Discovery and certification of sieve constants improve prime gap bounds}

The Maynard--Tao sieve reduces bounded gaps between primes to lower bounds on a variational constant \(M_k\), which we approached through two independent computational routes. The first uses neural search over separable functional families, followed by reconstruction in a polynomial representation and rigorous Rayleigh--Ritz certification. At small \(k\), this approach recovers known extremal behaviour (\(M_5\geq2.0071443298\) and \(M_{20}\geq3.12755795\)); at \(k=25\) it improves the Polymath8b reference (certified \(3.3221511\) versus \(3.3221426\)). Thus, high-quality extremizers can be discovered without prescribing the polynomial ansatz subsequently used for certification (Supplementary Tables~\ref{tab:vanilla-results} and~\ref{tab:epsilon-results}). Neural discovery further benefits from overparameterization: a redundant channel representation supplies extra variational degrees of freedom during optimization, after which unnecessary channels are removed by backward elimination with re-optimization at each step, subject to a cumulative relative Rayleigh-quotient loss of at most \(\tau_{\mathrm{prune}}=10^{-7}\).

At large \(k\), an early neural run near \(k=3600\) produced a numerically stable Rayleigh value above \(8\) that remained below the Cauchy--Schwarz ceiling and was stable under grid refinement. Exact reconstruction nevertheless failed: the optimizer had collapsed onto a sharply localized single channel whose discovery score was inflated by an inconsistent quadrature frame. The discovered profile approaches \(1/(\log k\cdot t)\), a form the polynomial channels cannot represent across the increasingly flat Rayleigh landscape (Supplementary Methods~\ref{sup:large-k-failure}). The artifact survived the available heuristic gates because its value did not cross the upper bound; only independent reconstruction of the candidate for certification exposed the inconsistency, which motivated NeuralCert's scalar cross-checks and support-aware thresholds.

To access much larger \(k\), we developed a second computational route using the rational profile \(g(t)=1/(c+(k-1)t)\), already employed in Polymath8b (Theorem 6.7). Our contribution is a direct, rigorously certified Fourier-domain evaluation of the associated product trial function on the simplex. This makes it possible to optimize and certify the construction at a computational cost nearly independent of \(k\), obtain improved explicit prime-gap bounds, and analyse its asymptotic deficit through a stable-law limit. This yields
\[
M_{3655}\geq8.000064869,\quad
M_{208910}\geq12.00000199,\quad
M_{11655069}\geq16.0000003924,\quad
M_{644589002}\geq20.0000009143,
\]
for \(3\), \(4\), \(5\) and \(6\) primes in a tuple, together with the uniform bound \(M_k\geq\log k-0.307\) for \(100\leq k\leq1.88\times10^9\). These results require only the Bombieri--Vinogradov theorem, assuming neither the Elliott--Halberstam conjecture nor Deligne-type distributional input.

Combining these thresholds with explicit admissible tuples yields
\[
H_2\le 33\,118,\quad H_3\le 2\,718\,108,\quad
H_4\le 214\,099\,720,\quad H_5\le 14\,541\,349\,288,
\]
where the tuple for \(k=3655\) is taken from the Engelsma--Sutherland database \cite{Sutherland-database}, the tuple for \(k=208910\) was constructed by a hybrid shifted-Schinzel/greedy sieve, and the two largest use the first \(k\) primes exceeding \(k\); all four are checked by a standalone verifier that shares no code with the construction (Supplementary Methods). Measured against the corresponding records established without the Elliott--Halberstam conjecture and without Deligne's theorems \cite{polymath_wiki,polymath2014variantsselbergsievebounded}, these improve the published bounds by factors of \(14.3\), \(11.9\), \(9.5\) and \(8.6\) respectively. They also improve on the stronger records that do invoke Deligne-type input \cite{stadlmann2025primes}, by factors of \(12.0\), \(9.0\), \(6.5\) and \(5.3\) (Figure~\ref{fig:rplot-asymptote} and Supplementary Data). This distinction matters because, as noted in \cite{stadlmann2026bounded}, the previously best-known upper bounds \(H_m\) for \(m\geq2\) relies on equidistribution estimates of Zhang type; the bounds above are obtained without them.

An asymptotic analysis of the single-channel rational family identifies its limiting convolution profile with a spectrally positive \(1\)-stable law and proves \(M_k\ge\log k-C_\ast-o(1)\), where \(C_\ast=\inf_{a\in\mathbb R}C(a)\); numerical evaluation suggests \(C_\ast\approx0.3343\) (Supplementary Methods). This explains why the rank-one structure remains asymptotically competitive, attaining the optimal \(\log k\) growth rate with only a constant-order deficit, without implying global optimality of rank one. Consequently
\[
0\leq\liminf_{k\to\infty}(\log k-M_k)\leq C_\ast,
\]
where the lower endpoint follows from the classical bound \(M_k\leq\frac{k}{k-1}\log k\). The finite-range constant \(0.307\) is smaller than \(C_\ast\) because, within this rational family, the finite-range deficit is numerically consistent with \(C_\ast-1.047/\log k\); the directly evaluated deficit remains below \(0.29\) throughout the certified range.

Despite sharing no numerical evaluation, the two computational pipelines agree throughout their overlapping range. While the finite-\(k\) results remain below Bogaert's independently computed (Krylov subspace) values for \(k\neq25\), \(k\leq30\), indicating that the Krylov method is powerful for modest \(k\), NeuralCert improves \(31\leq k\leq100\). Every reported finite-\(k\) inequality is accompanied by a machine-checkable certificate that can be independently verified.

\begin{figure}[t]
    \centering
    \includegraphics[width=0.85\linewidth]{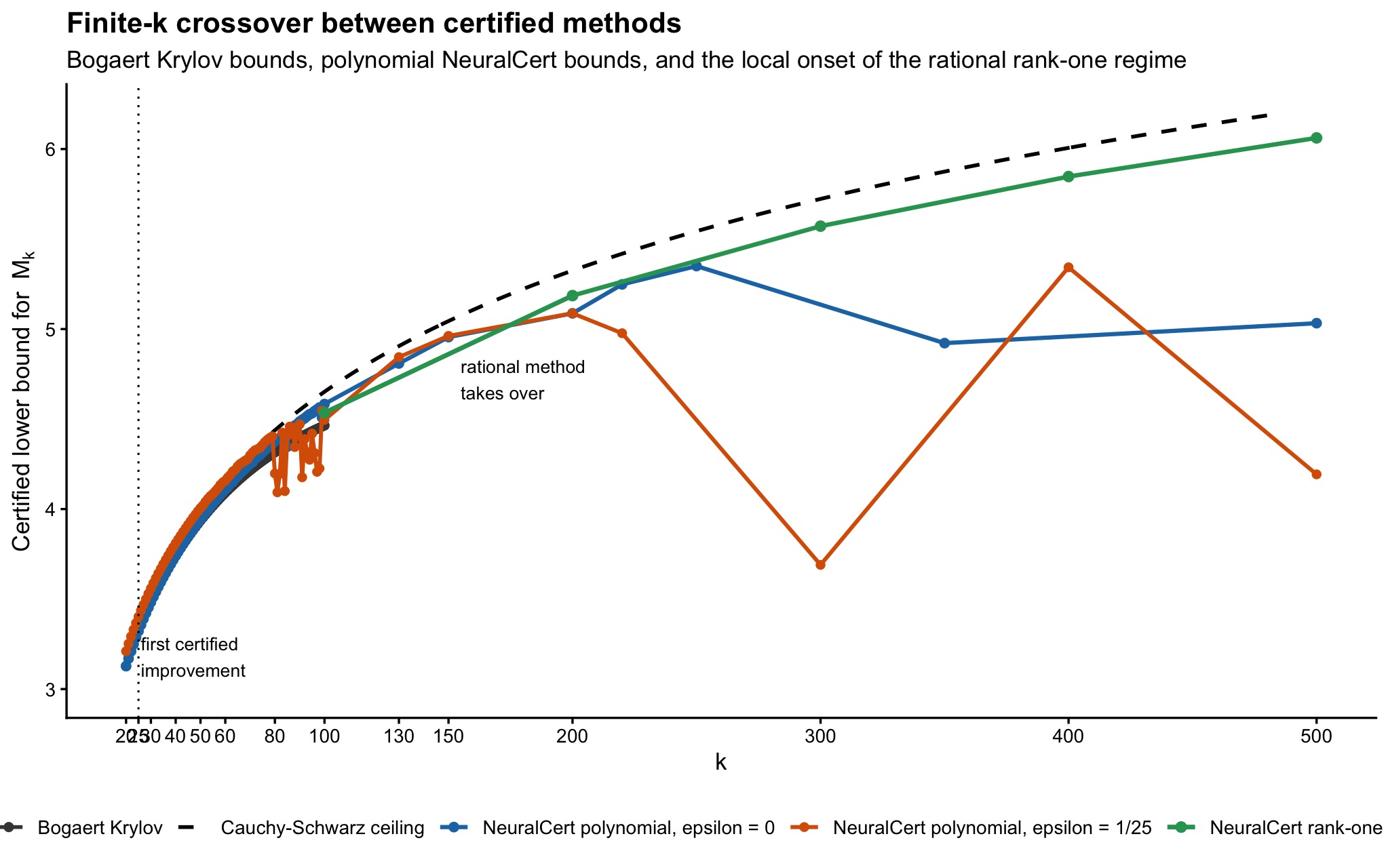}
\caption{\textbf{Certified lower bounds for the Maynard variational constant across computational regimes}. Certified values obtained with Bogaert's Krylov method are compared with NeuralCert polynomial constructions for the vanilla problem (\(\varepsilon=0\)) and enlarged support (\(\varepsilon=1/25\)), together with the certified rank-one rational family. The polynomial construction first exceeds the Krylov bounds at \(k=25\), but loses efficiency as \(k\) increases, whereas the rank-one rational representation becomes competitive and provides the continuation to larger \(k\). The Cauchy--Schwarz bound is shown as a reference. Lines connecting discrete certified values are guides to the eye and do not represent interpolated certificates. Beyond \(k=500\), the rank-one rational family is the only representation that remained reliably optimizable and certifiable in the regime studied here.}
    \label{fig:rplot-asymptote}
\end{figure}

We also implemented the \(\varepsilon\)-enlarged variant proposed in Polymath8b. At \(k=49\), optimization within our certified trial family saturates at \(3.98867\). The certified value is a lower bound, so this does not establish \(M_{49}\le4\); it shows that enlarging the sieve weight alone, at Bombieri--Vinogradov level, did not reach the threshold in our search. Concurrent work is consistent with this reading: \(H_1\leq240\) has since been established by combining the Bombieri--Vinogradov support with equidistribution estimates for smooth moduli \cite{stadlmann2026bounded}, and subsequent refinements of that construction give \(H_1\leq212\) \cite{charton2026new} and \(H_1\leq186\) \cite{openai2026improved}; all three use distributional input beyond Bombieri--Vinogradov. The additional ingredient is therefore on the arithmetic side rather than the variational one, which is where our saturation result locates it. We report the first certified \(\varepsilon\)-values at \(k=49,52\) and above \(k=54\).

%% file: Section2/02results.tex
\section*{Higher-order Delsarte discovery exposes a tensorization obstruction}

Higher-order Delsarte hierarchies strengthen the classical LP for linear codes and show substantial finite-length gains, but asymptotic progress requires an explicit dual family rather than additional numerical optimization \cite{loyfer2023new,coregliano2025higher}. We therefore treated the higher-order dual as a discovery problem, initially asking whether a low-complexity correction to the lifted level-one solution could improve the MRRW bound. Extending the symmetry-reduced hierarchy to $r=3$ confirmed that higher levels contain genuine information, 
including the level-3 bound $V_3(12,5)\le 16$, against the level-2 value
$V_2(12,5)=24.260255\ldots$ located numerically.
However, comparison with the explicit lift showed that the source of the improvement changes with hierarchy level: at $r=2$ it can reside entirely in the partial-Fourier constraints, whereas at $r=3$ the full-transform rows already improve on the lift. A single learned correction was therefore not a stable asymptotic target.

We redirected discovery to the spectral construction of \cite{coregliano2025higher} where the blocklength dependence can be removed analytically. For a normalized configuration $G$ on
$\mathbb F_2^\ell$, the relevant objective is
$J_\ell(G)= \mathsf H(G) / \ell $.
Although we initially intended to parameterize $G$ neurally, symmetry and analytic reduction collapsed the problem faster than a generic network became useful: for $\ell\le4$ the largest search space has only $15$ free probability coordinates and no remaining dependence on $n$. 
We therefore optimized the reduced configuration directly. Across $28$ instances and $13{,}600$ randomized local starts, no configuration improved MRRW; instead, the recovered optima repeatedly matched the quasirandom tensor-product family to numerical precision.

This repeated null result exposed a structural obstruction. Defining the translation affinity
\[
\Phi(G,v)=\sum_u\sqrt{G(u)G(u+v)},
\]
the finite spectral walk condition implies $\Phi(G,v)\ge\varepsilon$ on a
spanning set of shifts (Methods, hypotheses (A1)--(A2)).
For
\[
h(\varepsilon)=\mathsf H\!\left(\frac{1-\sqrt{1-\varepsilon^2}}{2}\right),
\]
convexity of $h$, conditioning and Jensen's inequality give
\[
\mathsf H(G)
\ge
\sum_{i=1}^{\ell} h\!\left(\Phi(G,e_i)\right)
\ge
\ell h(\varepsilon).
\]
Hence, with $\delta=(1-\varepsilon)/2$,
\[
J_\ell(G)\ge
h(\varepsilon)
=
\mathsf H\!\left(\frac12-\sqrt{\delta(1-\delta)}\right),
\]
exactly the first MRRW expression. The tensor-product configuration $\mathrm{Bernoulli}(p_\varepsilon)^{\otimes\ell}$ saturates the entropy--affinity inequality, identifying the object repeatedly recovered computationally. Hence, within the fixed single-orbit spectral ansatz, optimizing the configuration $G$ cannot improve MRRW Theorem~\ref{thm:obstruction} and Corollary~\ref{cor:no-improvement}).

The null result is informative because it is highly structured rather than merely negative. Across 13.600 local optimizations, the recovered configurations repeatedly collapse onto the same tensor-product family to numerical precision (Supplementary Methods \ref{sup:search}). This suggested that the observed optimum reflected an invariant of the reduced problem instead of optimization failure. The entropy–affinity argument supports this interpretation: the tensor-product configuration saturates the inequality and the numerically recovered object is an extremizer of the constrained configuration problem.

Furthermore, within the single-orbit spectral construction, additional optimization of $G$ cannot
improve the first MRRW expression; any improvement must instead exploit degrees of freedom not
controlled by the obstruction, such as alternative sign functions, support on multiple configuration
orbits, the partial-Fourier dual families or more general dual feasible solutions (Supplementary Methods \ref{sup:obstruction} and \ref{sup:verification}). Thus, the computation eliminates an apparently natural search direction while identifying
the components in which non-trivial asymptotic freedom remains.

%% file: Section3/02results.tex
\section*{Convex discovery improves sign-uncertainty bounds}

For the $+1$ Bourgain--Clozel--Kahane sign-uncertainty problem, radial polynomial--Gaussian Fourier eigenfunctions are the 
canonical ansatz.
Cohn, Dong and Gon\c{c}alves \cite{cohn2022sign} showed that sublinear-degree members of this family encounter an asymptotic ceiling, motivating our initial search for more flexible Fourier-eigenfunction structures. We therefore treated the choice of envelope, scales and contact locations as a discovery problem, with learned proposals followed by independent high-precision evaluation.

The flexible search improved rapidly but stalled in dimension $d=1$. Gaussian-mixture constructions reached $\rho=0.578532$, above the published
$A_{+}(1)\le0.572990$ bound. Neural optimization did not remove the plateau: apparently improving trajectories systematically entered coefficient systems with condition numbers $10^{17}$--$10^{18}$, where no reliable digits could be guaranteed from the floating-point solve. Multiprecision reevaluation rejected these candidates. 
Multistart runs also converged to distinct plateaus rather than a common value, consistent with changes in root multiplicity and contact structure in the last-sign-change functional. We therefore treated the failure as a diagnostic of the search representation rather than evidence that the function class itself had been exhausted.

The decisive simplification was to return to the Fourier-eigenfunction constraint. With $u=\pi |x|^2$, every radial $+1$ polynomial--Gaussian eigenfunction in the truncated Laguerre space has the form
\[
 f(x)=P(u)e^{-u},\qquad
 P\in\operatorname{span}
\left\{L_{0}^{(d/2-1)}(2u),L_{2}^{(d/2-1)}(2u),
\ldots,L_{2N}^{(d/2-1)}(2u)\right\}.
\]
The Fourier constraint is linear in the Laguerre coefficients and
$e^{-u}>0$. For fixed $u_0$, admissibility therefore reduces to
\[
 P(0)=0,\qquad P(u)\ge0\quad(u\ge u_0).
\]
To exclude the zero polynomial and fix the otherwise arbitrary scale,
we impose the linear normalization
\[
 \int_{u_0}^{\infty}P(u)e^{-u}\,du=1.
\]
Thus the search can be reformulated as a convex semi-infinite feasibility
problem in the coefficients, with bisection in $u_0$, rather than a
non-convex optimization over forced contacts. An adaptive exchange LP locates
the active tangencies; these are then rationalized, used to reconstruct an
exact Laguerre polynomial over $\mathbb Q$, and verified on the complete ray
by exact factorization and a Sturm chain over $\mathbb Z$. The numerical search determines the quality of the construction found, whereas admissibility of each quoted bound follows independently from the rational certificate. No global optimality claim is made.

In $d=1$ the degree ladder first reproduces the published configuration
and then improves it, with approximate last-sign-change radii
\[
\rho_{22}\approx0.572989678,\qquad
\rho_{30}\approx0.572706699,\qquad
\rho_{38}\approx0.572588699.
\]
The degree-$22$ search numerically resolves the same five-contact configuration as the published construction to its stated precision; the new improvement occurs at higher degree. The degree-$38$ candidate is certified by an explicit rational polynomial, giving
$
\boxed{A_{+}(1)\le0.572588700}
$.
The same pipeline gives
$
A_{+}(2)\le0.756206237
$,
which is approximately $7.6\times10^{-7}$ below the published $0.756207$ value and is therefore treated as a certified reproduction with a marginal gain rather than a substantive improvement (see Figure \ref{fig:signchange-landscape3d}).

%
%
%

\begin{figure}[t]
  \centering
  \includegraphics[width=0.86\textwidth]{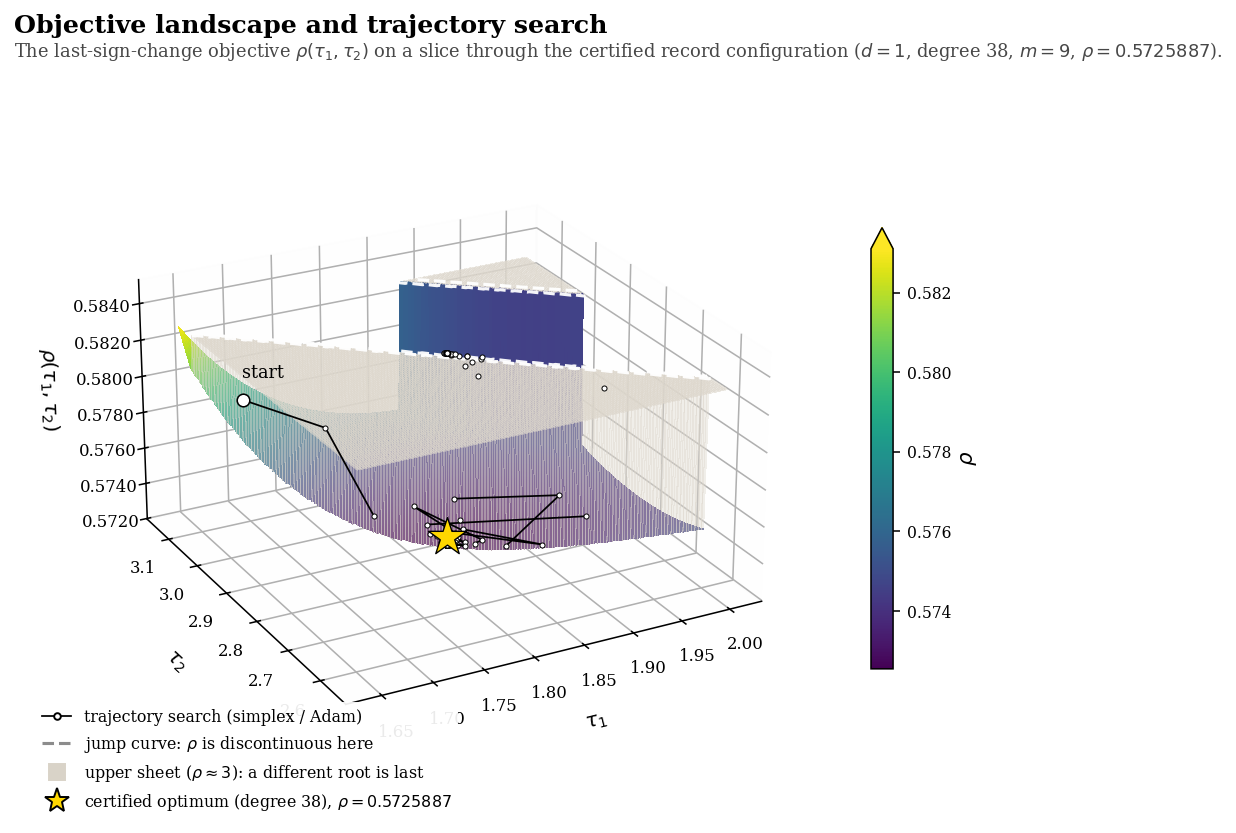}
    \caption{%
    \textbf{Objective landscape and trajectory search.}
    The last-sign-change functional $\rho(\tau_1,\tau_2)$ on a slice through
    the certified record configuration ($d=1$, degree~38, $m=9$ contacts,
    $\rho=0.5725887$), obtained by varying the two innermost contacts with
    the remaining seven held at $\tau^{\ast}$.  The surface is drawn from the
    raw grid without smoothing or interpolation, and is evaluated in the
    damped Laguerre basis $\psi_n(u)=L^{(\alpha)}_n(2u)e^{-u}$: at this
    degree the underlying polynomial reaches ${\sim}10^{69}$ on the grid, so
    a monomial evaluation loses every digit at the outer tangencies and
    splits them into spurious sign changes.  The landscape has two sheets, a
    narrow valley (colour) and an upper sheet at $\rho\approx3$
    (translucent), with no intermediate values~--- crossing the dashed
    curve, a different root of $P$ becomes the last sign change and $\rho$
    leaps.  Black markers show a simplex/Adam trajectory, its polyline
    broken wherever the objective jumps. The observed jumps illustrate the sensitivity of direct optimization to changes in the last-sign-change root. The convex reformulation avoids differentiating this discontinuous objective by using bisection in $u_0=\pi\rho^2$ and solving a coefficient-feasibility problem at each step. The dashed curve is drawn as a high-gradient contour and indicates where $\rho$ jumps rather than
    resolving the discontinuity set exactly.%
  }
  \label{fig:signchange-landscape3d}
\end{figure}

Higher-dimensional runs approach, rather than cross, the known polynomial--Gaussian asymptotic ceiling; their interpretation is limited by the numerical degree ceiling of the floating-point discovery stage (Methods). Thus, the new small-dimensional bounds do not evade the Cohn--Dong--Gon\c{c}alves obstruction. Instead, they show that the finite polynomial--Gaussian family had not been numerically exhausted.

%% file: 04discussion-v3.tex
Can modern neural optimization rediscover and improve extremal constructions in hard analytic and combinatorial mathematics? Across the three problems studied here, the answer takes several distinct forms. In the Maynard--Tao sieve, neural optimization identifies stronger extremizers, but at large \(k\) the neural representation ultimately gives way to the rational family \(g(t)=(c+(k-1)t)^{-1}\). The large-\(k\) results illustrate how computational progress can arise from a new evaluation and certification strategy for an established mathematical family.
In the higher-order Delsarte hierarchy, computational failure is itself informative: the search collapses onto a tensor-product family, and the resulting invariant yields an obstruction to improving the first MRRW expression through the entropy-based rate estimate under the spanning-affinity and finite-walk hypotheses considered here. The result leaves open constructions that bypass these hypotheses or use alternative spectral supports or dual feasible solutions.
In sign uncertainty, neural optimization encounters a plateau; independent verification reveals severe ill-conditioning, and the failure exposes a convex reformulation that produces new certified bounds.
In each case, the discovery representation is provisional: its purpose is to expose mathematical structure rather than to become part of the final proof \cite{davies2021advancing, raayoni2021generating, udrescu2020ai}. Computation therefore contributes not only by finding better objects, but also by identifying obstructions and more appropriate representations. The numerical degree ceiling encountered in the sign uncertainty problem is a case in point: asymptotic theory indicates that the strongest bounds lie in the high-degree regime, yet that regime is precisely where the discovery representation loses numerical resolution, so progress required changing the representation rather than pushing the degree further. 

Separating heuristic proposal generation from mathematically controlled inference also underlies neuro-symbolic systems such as AlphaGeometry, in which learned construction proposals are coupled to a symbolic deduction engine \cite{trinh2024solving}. NeuralCert applies a related separation not to theorem search, but to continuous variational discovery: the learned component proposes mathematical objects, whereas the final inequality depends only on the reconstructed certificate. The computational model need not itself be interpretable, exact or even retained; it can use floating-point arithmetic, stochastic optimization, approximate quadrature and overparameterized models.
Discovery output is instead regarded as an exploratory and untrusted proposal that must cross into a mathematically explicit representation before it can support a claim. Certification and verification then operate on that representation independently of the optimizer. This design follows certificate-based principles in verified computation \cite{tan2024formally}. The failed large-\(k\) Maynard candidate demonstrated the value of this design. A candidate near \(k=3600\) remained stable under grid refinement and below the ceiling, yet independent reconstruction revealed a scale inconsistency with inflation of the objective. The optimizer had exploited the numerical error, generating a plausible extremizer that failed certification. Conversely, the later rational rank-one construction was accepted because its Rayleigh quotient can be independently reconstructed and certified. Hence, structural similarity discovered numerically becomes mathematically relevant only when it survives the certification boundary, in line with the certificate discipline of computer-assisted proof \cite{hales2017formal}. The methodological contribution is therefore the transition between search representations and independently verifiable mathematical objects, with neural optimization serving as one component of that process.

The same separation has appeared concurrently on this very problem. Recently, the bound \(H_1\leq246\) was improved to \(240\) \cite{stadlmann2026bounded}, then to \(212\) \cite{charton2026new} and \(186\) \cite{openai2026improved}, the last two accompanied by Lean certificates and, in one case, a proof attributed to an automated system. The present work automates different objects. Instead of a proof generated through a proof assistant at a single small dimension (\(k=49\), \(45\) and \(40\) respectively), we discovered  a trial function numerically and then removed from the argument entirely: what is certified is the Rayleigh quotient of one fixed, exactly specified function, at dimensions up to \(k\approx6\times10^{8}\) and uniformly to \(k=1.88\times10^{9}\). Formal verification is the stronger guarantee, and extending certificates into that setting is a potent direction. Conversely, the variational constant is shared infrastructure: earlier work requires \(k\approx3.5\times10^4\) to reach \(m=2\) because the available lower bound on \(M_k\) falls well short of \(\log k\), and combining the sharper bound \(M_k\ge\log k-0.307\) with the improved equidistribution estimates used in \cite{stadlmann2026bounded,charton2026new,openai2026improved} would be expected to improve \(H_m\) for \(m\ge2\) beyond either ingredient alone. 

The framework also has clear limitations. NeuralCert is not a general mechanism for converting numerical optimization into proof. Certification requires that a discovery can be transferred to an explicit mathematical representation, and this transfer may itself become a bottleneck. A failed reconstruction is not hard evidence that the discovered function is mathematically uninteresting; it only establishes that the proposed claim has not crossed the certification boundary. Likewise, certificates for variational constructions establish the performance of explicit admissible trial functions, not their global optimality. A further limitation concerns the standard of verification itself. The certificates here are rigorous numerical objects; interval enclosures and exact modular arithmetic, not machine-checked proofs, and a Lean or Coq development of the reduction and error bounds would be a strictly stronger guarantee. That the same problem has recently attracted both approaches suggests the combination is within reach. More generally, flexible neural (over)parameterizations are particularly useful when their additional representational freedom assists discovery. Once lower-dimensional structure, convexity or an explicit analytic family has been identified, retaining a neural representation offers limited advantage; the mathematically simpler representation should replace it.

%% file: Section1/03methods.tex
\section{Discovery and certification of sieve constants improve prime gap bounds}

\subsection{Maynard--Tao variational formulation}
\label{sec:maynard_variational}

The Maynard--Tao sieve reduces bounded gaps between primes to the construction of admissible trial functions with large variational quotient \cite{maynard2015small, polymath2014variantsselbergsievebounded}. For the separable constructions considered here, the relevant denominator and numerator functionals can be expressed through pairwise one-dimensional convolution integrals. We write the trial function as
\begin{equation}
F_{\theta,c}(t_1,\ldots,t_k)
=
\sum_{j=1}^{m} c_j\prod_{i=1}^{k}g_{\theta,j}(t_i),
\label{eq:maynard_separable}
\end{equation}
where $g_{\theta,j}$ are one-dimensional channels and $c\in\mathbb{R}^m$ is a linear mixing
vector. For each pair,
\begin{equation}
h_{j\ell}(x)=g_{\theta,j}(x)g_{\theta,\ell}(x),
\end{equation}
and the $k$-dimensional functionals reduce to quadratic forms
\begin{equation}
I_{k,\varepsilon}(F_{\theta,c})=c^\top A(\theta)c,
\qquad
J^{(1)}_{k,\varepsilon}(F_{\theta,c})=c^\top B(\theta)c,
\label{eq:maynard_quadratic_forms}
\end{equation}
whose entries depend on one-dimensional integrals involving
$h_{j\ell}^{*(k-1)}$. At fixed channels the objective is therefore the generalized Rayleigh
quotient
\begin{equation}
R(\theta)
=
k\,\lambda_{\max}\!\left(B(\theta),A(\theta)\right)
=
k\max_{c\neq0}\frac{c^\top B(\theta)c}{c^\top A(\theta)c}.
\label{eq:maynard_profiled_rayleigh}
\end{equation}
The nonlinear optimizer consequently searches only over the channel manifold; the optimal linear combination is re-solved as a generalized Ritz vector at each objective evaluation \cite{Parlett1998, GolubVanLoan2013}. The precise normalization of $M_k$, its $\varepsilon$-enlarged analogue, and the conversion of thresholds in $M_k$ to bounds on $H_m$ are given in Supplementary Methods, Section~\ref{sup:maynard_variational_details}.

\subsection{Neural discovery of separable trial functions}
\label{sec:maynard_neural_discovery}

The discovery stage is used only to identify strong explicit trial functions and does not establish a rigorous inequality. To expose the increasingly sharp boundary structure at large $k$, each channel is parameterized as
\begin{equation}
g_{\theta,j}(x)=r_{\theta,j}(x)e^{-\rho_jx},
\qquad \rho_j>0,
\label{eq:maynard_channel_envelope}
\end{equation}
where the residual functions are produced by a shared multilayer perceptron with channel-specific outputs. The network receives $(x,e^{-kx})$ as input, explicitly exposing the natural $O(k^{-1})$ boundary-layer scale. In the positive-channel solver,
$r_{\theta,j}(x)=\exp(f_{\theta,j}(x,e^{-kx}))$, allowing the convolution recursion to be performed without cancellation in logarithmic coordinates.

The leading generalized Ritz vector is profiled out rather than learned jointly with the channel parameters. If
\begin{equation}
B(\theta)v_\theta=\lambda_\theta A(\theta)v_\theta,
\qquad
v_\theta^\top A(\theta)v_\theta=1,
\end{equation}
we freeze $v_\theta$ and differentiate
\begin{equation}
q_{v_\theta}(\theta)
=
k\frac{v_\theta^\top B(\theta)v_\theta}
       {v_\theta^\top A(\theta)v_\theta}.
\label{eq:maynard_frozen_ritz}
\end{equation}
Under the usual nondegeneracy condition, the Hellmann--Feynman/envelope identity \cite{Feynman1939, semay2015hellmann} gives the derivative of the profiled leading eigenvalue without differentiating through the eigendecomposition. Adam \cite{KingmaBa2015} is used for exploration and a block L-BFGS minorize--maximize procedure for polishing \cite{LiuNocedal1989}.
Channel-rank and grid-resolution continuation are used during search. Network architecture, optimization schedules, finite-difference gradient controls and continuation rules are specified in Supplementary Methods, Sections~\ref{sup:channel_model}--\ref{sup:optimization_details}.

\subsection{Stable evaluation of high-order convolution powers}
\label{sec:maynard_stable_convolution}

Direct recursion on $h^{*p}$ becomes unstable as $p$ grows because algebraic concentration, exponential decay and global amplitude must otherwise be represented simultaneously. For each channel pair we write
\begin{equation}
h_{j\ell}(x)=r_{j\ell}(x)e^{-a_{j\ell}x},
\qquad a_{j\ell}=\rho_j+\rho_\ell,
\end{equation}
and factor its convolution powers as
\begin{equation}
\nu^{(j\ell)}_p(s)
=
s^{p-1}e^{-a_{j\ell}s}e^{\sigma^{(j\ell)}_p}
\psi^{(j\ell)}_p(s).
\label{eq:maynard_factored_convolution}
\end{equation}
The algebraic and exponential terms are handled analytically, the scalar $\sigma_p$ carries the global amplitude, and only the residual shape $\psi_p$ is represented numerically. Substitution into the convolution identity and the change of variables $x=su$ yield a beta-weighted recursion,
\begin{equation}
\psi_{p+q}(s)
=
B(p,q)\int_0^1
\psi_p(su)\psi_q(s(1-u))
\frac{u^{p-1}(1-u)^{q-1}}{B(p,q)}\,du,
\label{eq:maynard_beta_recursion}
\end{equation}
which is evaluated with Gauss--Jacobi quadrature \cite{GolubWelsch1969}. The concentration associated with increasing $p$ and $q$ is thereby moved into a known quadrature weight rather than left as unresolved structure in the represented residual.

For positive channels we store $\xi_p=\log\psi_p$ and evaluate the positive quadrature sum by log-sum-exp, preventing underflow across repeated convolutions. The required power $h^{*(k-1)}$ is assembled through a binary addition chain, reducing convolution depth from $O(k)$ to $O(\log k)$. Representation, convolution, primitive-integration and outer-integration grids are kept distinct. Full recursion formulas, grid definitions and quadrature construction are given in Supplementary Methods, Sections~\ref{sup:factored_convolution}--\ref{sup:discovery_grids}.

\subsection{Structure-preserving Rayleigh--Ritz optimization}
\label{sec:maynard_structure_preserving}

The denominator matrix $A$ is mathematically a Gram matrix and must satisfy
$A\succeq0$. Entrywise accurate discretization does not automatically
preserve this property. We therefore use positive two-point interpolation in
the convolution recursion and a shared positive-weight outer integration mesh
across all channel pairs. Whenever the inner quadrature nodes are common to
all pairs, the discrete denominator has the form
\begin{equation}
A^{(N)}_{j\ell}
=
\sum_\alpha w_\alpha\Phi_{\alpha j}\Phi_{\alpha\ell},
\qquad w_\alpha>0,
\end{equation}
and therefore retains the positive-semidefinite Gram structure up to
floating-point round-off. In the saturated regime, where the inner rule for
the channel primitives becomes pair-dependent, this common representation is
preserved only up to inner-quadrature error; numerical adequacy is then
assessed by the positive-semidefiniteness, two-grid and refinement
diagnostics of Section~\ref{sec:maynard_validation}, which are safeguards
rather than rigorous bounds on that error. No bound reported here depends on
this discrete structure: the certified values are produced by the independent
exact certifier.

Because channel norms may span hundreds of orders of magnitude, the Gram
pencil is diagonally equilibrated by the exact congruence transformation
\begin{equation}
A\mapsto D^{-1/2}AD^{-1/2},
\qquad
B\mapsto D^{-1/2}BD^{-1/2},
\qquad
D=\operatorname{diag}(A_{11},\ldots,A_{mm}).
\end{equation}
Numerically unresolved directions are removed by rank-revealing spectral
whitening; no ridge term is added to the variational denominator. The
interpolation, shared-mesh construction, equilibration and rank-revealing
Ritz solve are detailed in Supplementary Methods,
Sections~\ref{sup:gram_preservation}--\ref{sup:ritz_details}.

\subsection{Validation of neural candidates}
\label{sec:maynard_validation}

Validation is embedded into model selection because a flexible optimizer can exploit structured
discretization error. Candidate states are therefore required to reproduce the independently recomputed
Rayleigh quotient, preserve numerical positive semidefiniteness of the denominator, retain sufficient
effective rank, remain stable under a more conservative rank threshold, respect applicable rigorous
upper bounds for both individual channels and the full quotient (rejecting numerical states that spuriously exceed the Cauchy–Schwarz ceiling \cite{polymath2014variantsselbergsievebounded}), and reproduce on a finer representation grid. Closed-form controls based on $g(x)=1$ and $g(x)=x$ exercise the convolution and integration paths before optimization. Final candidates are re-evaluated on multiple grids;
extrapolation is used only diagnostically and never as a certified lower bound. Exact validation criteria, controls and refinement procedures are provided in Supplementary Methods, Sections~\ref{sup:validation_details}--\ref{sup:refinement_details}.

\subsection{Exact certification of neural candidates}
\label{sec:maynard_exact_certification}

A discovered candidate is projected onto rational polynomial channels,
\begin{equation}
F(t_1,\ldots,t_k)
=
\sum_{j=1}^{m}c_j
\prod_{i=1}^{k}
q_j\!\left(\frac{t_i}{1+\varepsilon}\right),
\label{eq:maynard_rational_trial}
\end{equation}
and the mixing vector is transferred from the preconditioned discovery frame to the corresponding
exact channel frame. For fixed rational channels and rational coefficients the associated Maynard
quotient is rational; certification therefore reduces to exact evaluation of the two quadratic forms.

Rather than constructing and reconstructing every exact Gram-matrix entry, the certifier accumulates
the pair contributions to $c^\top Ac$ and $c^\top Bc$ modulo machine-word primes. For each prime
$p$ it computes only two aggregate residues,
\begin{equation}
S_A(p)=
\sum_{j\leq\ell}\kappa_{j\ell}c_jc_\ell A^{\mathrm{int}}_{j\ell}\pmod p,
\qquad
S_B(p)=
\sum_{j\leq\ell}\kappa_{j\ell}c_jc_\ell B^{\mathrm{int}}_{j\ell}\pmod p.
\end{equation}
Only the aggregate integers $S_A$ and $S_B$ are reconstructed, with analytically known positive
denominators $D_A,D_B$:
\begin{equation}
c^\top Ac=\frac{S_A}{D_A},
\qquad
c^\top Bc=\frac{S_B}{D_B}.
\end{equation}
Rigorous a priori integer bounds $\lvert S_A\rvert\leq\mathcal B_A$ and
$\lvert S_B\rvert\leq\mathcal B_B$ determine the reconstruction budget. If
$M_P=\prod_{p\in P}p$ and
\begin{equation}
M_P>2\max(\mathcal B_A,\mathcal B_B),
\end{equation}
the centered Chinese-remainder representative is unique. The final quotient is then formed exactly
from the reconstructed integers and known denominators. Thus the neural optimizer, numerical
quadrature and grid extrapolation no longer enter the proof once the rational trial function has been
fixed.

Before exact evaluation, the candidate is compressed by loss-controlled backward elimination:
channels are removed only while the re-optimized Rayleigh quotient remains within a prescribed
relative loss budget. The exact backend certifies the explicit reduced function, so pruning affects
certificate strength and computational cost but not validity. Detailed denominator formulas,
coefficient-extraction identities, rational representation transfer, pruning rules, integer bounds and
deterministic CRT reconstruction are given in Supplementary Methods,
Sections~\ref{sup:aggregate_modular}--\ref{sup:pruning_details}.


\subsection{Large-$k$ rational construction}
\label{sec:maynard_rational_largek}

At large $k$, neural optimization is used as a structural discovery instrument rather than as the
source of a certified numerical value. The neural candidate is compressed into clustered rational
families
\begin{equation}
g_{\mathrm{rat}}(t)
=
\sum_{a=1}^{r}\sum_{j=1}^{p_{\max}}
u_{a,j}(c_a+nt)^{-j},
\qquad n=k-1,
\label{eq:clustered-rational}
\end{equation}
where $c_a>0$ are pole locations and the higher powers represent confluent directions. For fixed
$c_a$, the coefficients $u_{a,j}$ are eliminated by variable projection, so only the nonlinear pole
locations are optimized. The confluent basis has a direct geometric interpretation because
\begin{equation}
\frac{\partial^q}{\partial c^q}\frac{1}{c+nt}
=
(-1)^q q!(c+nt)^{-(q+1)}.
\end{equation}

The quality of a distilled representation is not determined by pointwise fitting error alone. Each
candidate is rebuilt in an independent deterministic evaluator, the corresponding Gram pencil is
recomputed, and the generalized Rayleigh problem is solved afresh. Candidates are rejected if they
fail probability-mass conservation, analytic moment checks, Fourier-inversion stability, Gram-rank
diagnostics, cross-cluster conditioning, or the rigorous ceiling
\begin{equation}
R_k\leq \frac{k}{k-1}\log k.
\end{equation}
Surviving pole locations are then locally refined under the deterministic Rayleigh objective.

In the large-$k$ regime, additional rational clusters can reduce the pointwise approximation error
while failing these distributional checks. The stable deterministic representation collapses to the
first-order family
\begin{equation}
g(t)=\frac{1}{c+nt},
\qquad n=k-1,
\label{eq:maynard_largek_rational}
\end{equation}
which is subsequently optimized independently of the neural model. This distillation step therefore
uses neural search to identify structure, while the numerical value submitted for certification is
obtained from a separate rational representation. The clustered fit, deterministic Fourier evaluator,
failure diagnostics and local refinement are detailed in Supplementary Methods,
Section~\ref{sup:large-k-failure}.

\subsection{Certified large-$k$ evaluation in ball arithmetic}
\label{sec:maynard_ball_arithmetic}

For the first-order family in Eq.~\eqref{eq:maynard_largek_rational}, certification can be reduced
to a one-dimensional probabilistic Fourier calculation. Let $w=g^2$,
$m_0=\int_0^1w(t)\,dt=[c(c+n)]^{-1}$, and let $S$ be the sum of
$n=k-1$ independent draws from the probability density $w/m_0$ on $[0,1]$.
With
\begin{equation}
G(\rho)=\int_0^\rho g(t)\,dt,
\qquad
H(\rho)=\int_0^\rho w(t)\,dt,
\end{equation}
the separable Maynard quotient reduces exactly to
\begin{equation}
R=\frac{kN}{D},
\qquad
N=\mathbb E\!\left[G(1-S)^2\mathbf 1_{S<1}\right],
\qquad
D=\mathbb E\!\left[H(1-S)\mathbf 1_{S<1}\right].
\label{eq:maynard_probabilistic_reduction}
\end{equation}
Writing
\begin{equation}
\varphi(\theta)
=
\left(\frac{\widehat w(\theta)}{m_0}\right)^n,
\qquad
\Psi_X(\theta)
=
\int_0^1X(\rho)e^{i\theta\rho}\,d\rho ,
\end{equation}
for $X\in\{G^2,H\}$, Fourier inversion gives
\begin{equation}
I_X
=
\frac{1}{2\pi}
\int_{\mathbb R}
\varphi(\theta)e^{-i\theta}\Psi_X(\theta)\,d\theta .
\label{eq:maynard_fourier_certificate}
\end{equation}

The certifier evaluates a truncated trapezoidal form of
Eq.~\eqref{eq:maynard_fourier_certificate} using Arb ball arithmetic \cite{johansson2017arb}. The difference from the
exact integral is bounded by three separately controlled contributions: periodic aliasing, omitted
Fourier mass, and rigorous numerical enclosure of the transforms and auxiliary integrals. Aliasing
is one-sided because the corresponding convolution density is nonnegative and compactly supported;
its residual mass is bounded by certified moment inequalities. The omitted Fourier range is divided
into a finite band controlled by Arb wide-ball suprema and an analytic tail obtained from
integration-by-parts decay.

All error budgets remain in ball arithmetic until the final quotient is assembled. If
$N_{\rm ball}$ and $D_{\rm ball}$ are the certified truncated-sum enclosures and $E_N,E_D$
the rigorous aliasing and truncation budgets, the exported lower bound is
\begin{equation}
R_{\rm low}
=
\operatorname{lb}\!\left(
\frac{k\,\operatorname{lb}(N_{\rm ball}-E_N)}
     {\operatorname{ub}(D_{\rm ball}+E_D)}
\right),
\label{eq:maynard_ball_lower_bound}
\end{equation}
with all conversions explicitly rounded outward. The stored claim is rounded downward, so changes
in valid Arb subdivision or enclosure radii cannot strengthen the reported inequality. The complete
reduction, Fourier lemmas, moment bounds, tail estimates, special-function evaluation and directed
rounding procedure are given in Supplementary Methods,
Sections~\ref{sup:probabilistic_reduction}--\ref{sup:arb_certification}.

\subsection{Finite-range certification and stable-law asymptotics}
\label{sec:maynard_asymptotics}

The first-order rational family also permits a certified bound over a continuous range of integer
dimensions. Let $R_k^{(1)}$ denote the Rayleigh quotient of
$g_c(t)=(c+(k-1)t)^{-1}$ after selection of $c$. To avoid optimizing $c$ independently at every
integer $k$, we use the standardized simplex headroom
\begin{equation}
\eta(c,k)
=
\frac{1-\mathbb E[S]}{\operatorname{sd}(S)}
\end{equation}
as a scale coordinate. Its moments are available in closed form, and the large-$k$ optimizers
empirically concentrate near a fixed value of $\eta$. A log-uniform ladder beginning
at $k=100$ was therefore generated by solving $\eta(c,k)=-0.4746$ and certifying the resulting
fixed rank-one trial at each rung. For consecutive certified values $k_i<k_{i+1}$, monotonicity of
$M_k$ gives, for every integer $k\in[k_i,k_{i+1}]$,
\begin{equation}
M_k
\ge
R_{\rm cert}^{(1)}(k_i)
\ge
\log k-
\left[
\log k_{i+1}-R_{\rm cert}^{(1)}(k_i)
\right].
\label{eq:maynard_interval_cover}
\end{equation}
The maximum interval deficit over the complete ladder is $0.306754936$, yielding the uniform
certified statement
\begin{equation}
M_k\ge\log k-0.307,
\qquad
100\le k\le1.88\times10^9.
\label{eq:maynard_uniform_0307}
\end{equation}
The ladder construction, scale surrogate and interval-by-interval calculation are given in
Supplementary Methods, Sections~\ref{sup:rank1_ladder}--\ref{sup:rank1_scale}.

The asymptotic behavior of the same rational family can be analysed independently of the numerical
discovery pipeline. Put $n=k-1$ and choose the fixed-shift scaling
\begin{equation}
c_k^{-1}=\log n+a,
\qquad a\in\mathbb R.
\label{eq:maynard_fixed_shift}
\end{equation}
After rescaling the squared rational channel, the associated triangular array has Lévy density
converging to $x^{-2}\,dx$. The centered boundary variable converges in distribution to
\begin{equation}
X_a=(a+1)-\Lambda,
\end{equation}
where $\Lambda$ is the spectrally positive $1$-stable random variable with characteristic function
\begin{equation}
\mathbb E e^{is\Lambda}
=
\exp\left\{
\int_0^\infty
\left(
e^{isx}-1-isx\mathbf1_{\{x\le1\}}
\right)\frac{dx}{x^2}
\right\}.
\label{eq:maynard_stable_cf}
\end{equation}
Uniform density and logarithmic-integrability estimates permit passage from weak convergence to the
logarithmic observables in the Rayleigh quotient. The resulting asymptotic defect is
\begin{equation}
R(F_{k,c_k})
=
\log k-\mathcal C(a)+o(1),
\qquad
\mathcal C(a)
=
a-2\,\mathbb E[\log X_a\mid X_a>0].
\label{eq:maynard_Ca}
\end{equation}
Thus, with
\begin{equation}
\mathcal C_*=\inf_{a\in\mathbb R}\mathcal C(a),
\end{equation}
the Maynard constant satisfies
\begin{equation}
M_k\ge\log k-\mathcal C_*-o(1).
\label{eq:maynard_asymptotic_bound}
\end{equation}
Numerical evaluation of the limiting expression suggests \(C^\ast\approx0.3343\). This is a lower-bound construction obtained from fixed-shift scalings and is not asserted here to be the asymptotic optimum over all possible sequences \(c_k\).
The stable-limit proof, logarithmic uniform-integrability argument and numerical evaluation of
$\mathcal C_*$ are detailed in Supplementary Methods, Section~\ref{sup:stable_law}.

\subsection{$\varepsilon$-enlarged sieve calculations}
\label{sec:maynard_epsilon}

We additionally evaluate the enlarged-domain variant of the Maynard variational problem. For each
computed pair $(k,\varepsilon)$, the same discovery-to-certification separation is retained: a
candidate is optimized numerically, transferred to the explicit representation used by the
$\varepsilon$-certifier, and the resulting lower bound is recomputed independently. To compare the
enlarged and vanilla searches we define
\begin{equation}
\Delta_{\rm cert}(k,\varepsilon)
=
L_{k,\varepsilon}-L_{k,0},
\label{eq:eps_delta}
\end{equation}
where $L_{k,\varepsilon}$ and $L_{k,0}$ are the certified lower bounds obtained by the two
pipelines. Across the reliable sweep, $\Delta_{\rm cert}$ decreases with $k$ and changes sign for
the tested values of $\varepsilon$. Linear interpolation of the reliable sign-bracketing pairs gives
the crossover estimates reported in the Results (Supplementary Table \ref{tab:epsilon_sweep}).

Because Eq.~\eqref{eq:eps_delta} compares two lower bounds rather than the unknown exact variational
constants, a negative value does not prove
$M_{k,\varepsilon}<M_k$. It establishes only that the certified vanilla construction found by our
pipeline outperforms the certified enlarged-domain construction at that $k$. We therefore treat the
finite crossover as an empirical property of the certified construction ladder and state its
extension to the exact variational constants as a conjecture. In particular, the observed crossover
locations motivate
\begin{equation}
\varepsilon k^*(\varepsilon)\longrightarrow\infty
\qquad (\varepsilon\to0),
\end{equation}
but the available data do not determine a unique finer scaling law. The sign brackets, excluded
optimization failures, interpolation procedure and scope of the conjecture are given in
Supplementary Methods, Section~\ref{sup:epsilon_crossover}.

\subsection{Independent certificate verification}
\label{sec:maynard_verification}

Large-$k$ certificates are self-contained records containing the exact rational trial parameter,
Fourier discretization parameters, error summaries and the claimed lower bound. The standalone
verifier does not trust stored intermediate numerical bounds. It reconstructs the analytic constants,
regenerates the far-field partition, recomputes wide-ball suprema, certified moments, Fourier nodes,
trapezoidal sums, aliasing and truncation errors, and assembles a new $R_{\rm low}$ in ball arithmetic.
The claim is accepted only when this independently recomputed lower bound supports the stored
inequality.

The certificate proves the Rayleigh quotient of one fixed admissible trial function; it does not
establish optimality of its parameter, re-prove the external Maynard--Tao sieve theorem, or certify
the diameter of the admissible $k$-tuple used to convert a certified $M_k$ value to a prime-gap bound.
Monte--Carlo and small-$k$ comparisons are falsification tests rather than proof components.
Certificate schema and trust boundary are specified in Supplementary Methods,
Section~\ref{sup:standalone_verifier}.

%% file: Section2/03methods_Delsarte_v5.tex
\section{Higher-order Delsarte discovery and certification}

\paragraph{Symmetry-reduced hierarchy.}
For hierarchy level $r$, blocklength $n$ and minimum distance $d$ we
construct the symmetry-reduced higher-order Delsarte LP over type vectors
\[
I_{r,n}=
\Bigl\{\alpha\in\mathbb N^{2^r}:
\textstyle\sum_{u\in\mathbb F_2^r}\alpha_u=n\Bigr\},
\]
where a type records the multiplicity of each column $u\in\mathbb F_2^r$ in
an $r\times n$ binary matrix. For linear codes the program is invariant
under $\mathrm{GL}(r,2)$, so variables are indexed by group orbits. Full
Fourier constraints are represented by multivariate Krawtchouk coefficients
$K_\alpha(\beta)$, and the Loyfer--Linial strengthening adds the
corresponding partial transforms \cite{loyfer2023new}. We implemented these
transforms recursively and extended the symmetry reduction to $r=3$, where
sorting-based orbit representatives are no longer valid; construction,
orbit enumeration and validation anchors are given in Supplementary
Methods~\ref{sup:scaleup}.

\paragraph{Exact dual certification.}
Floating-point optimization is used only to locate competitive dual
solutions. Candidates are reconstructed at high precision, rounded to
dyadic rationals and checked against every dual inequality in exact
arithmetic. Residual rounding violations are repaired using
\[
\sum_{\beta\in I_{r,n}}
\binom{n}{\beta}K_\alpha(\beta)
=
2^{rn}\mathbf 1_{\{\alpha=n\varepsilon_0\}},
\]
which supplies an exactly feasible repair direction with a rational cost.
Finite-$n$ bounds therefore rest on exact dual feasibility and weak duality
alone, not on the numerical LP tolerance, and are upper bounds on the
level-$r$ optimum (Supplementary Methods~\ref{sup:exactcert}).

\subsection*{Reduction to an asymptotic configuration search}

We next considered the spectral dual construction of Coregliano et al., in
which the complementary spectral function is supported on a single
configuration orbit \cite{coregliano2025higher}. A normalized configuration
is a probability distribution $G$ on $\mathbb F_2^\ell$. Because a
level-$\ell$ dual bounds $|C|^\ell$, the asymptotically relevant rate
objective is
\[
J_\ell(G)=\frac{\mathsf H(G)}{\ell},
\]
with $\mathsf H(\cdot)$ the Shannon entropy in bits; this normalization,
and two discrepancies in the cited preprint that it exposes, are recorded in
Supplementary Methods~\ref{sup:normalization}. For a non-zero shift $v$ we
use the Bhattacharyya translation affinity
\[
\Phi(G,v)=\sum_u\sqrt{G(u)G(u+v)} .
\]
Writing $G=p^2$ with $p\ge0$ and $\|p\|_2=1$ gives
$\Phi(G,v)=p^{\mathsf T}S_vp$, where $S_v$ is the permutation induced by
translation by $v$, so discovery reduces to entropy minimization on the
positive unit sphere under quadratic affinity constraints. The analytic and
symmetry reductions leave $2^\ell-1$ free probability coordinates, at most
$15$ for $\ell\le4$, with no remaining dependence on $n$. We had intended a
neural parameterization of $G$, but at this point a generic network was
unnecessary, so we searched the reduced space directly by multistart
sequential quadratic programming: $28$ instance solves over $\ell\le4$ and
$\varepsilon\in\{0.30,0.20,0.10,0.05\}$, $13{,}600$ local optimizations in
total. No run produced a configuration below the first MRRW benchmark, and
the recovered optima matched the quasirandom tensor-product family to
within $10^{-11}$ coordinatewise (Supplementary
Methods~\ref{sup:search}).

\subsection*{The tensorization obstruction}

The repeated null result is explained by an inequality that holds for every
admissible configuration. Let $\varepsilon\in(0,1)$, let $V$ be the set of
shifts at which the spectral walk condition of
\cite{coregliano2025higher} is verified, and let
\[
h(\varepsilon)=H_2\!\left(\tfrac12\bigl(1-\sqrt{1-\varepsilon^2}\bigr)\right)
=H_2\!\left(\tfrac12-\sqrt{\delta(1-\delta)}\right),
\qquad
\delta=\tfrac{1-\varepsilon}{2},
\]
where $H_2$ is the binary entropy function; the second form is the first
MRRW expression. Two properties of the construction enter: the admissible
shifts span $\mathbb F_2^\ell$, and at each such shift a finite even
spectral-walk order $m$ satisfies a lower bound of the form
$W_m(G,v)\ge c\,\varepsilon^{m}$ with $c\ge1$, where $W_m$ is the paired
walk weight appearing in the configuration asymptotics of
\cite{coregliano2025higher}. Comparing $W_m$ termwise with the multinomial
expansion of $\Phi(G,v)^m$ gives $W_m(G,v)\le\Phi(G,v)^m$, so the walk
condition forces the affinity floor $\Phi(G,v)\ge\varepsilon$ at a fixed
$m$, with no exchange of the $n\to\infty$ and $m\to\infty$ limits.
Conditioning each coordinate on the others turns that floor into an entropy
floor: the conditional entropy of $X_i$ given $X_{-i}$ equals
$\mathbb E[h(2\sqrt{p_Y(1-p_Y)})]$ while $\Phi(G,e_i)$ equals
$\mathbb E[2\sqrt{p_Y(1-p_Y)}]$ for the same conditional marginals, so
convexity of $h$ and Jensen's inequality give
$\mathsf H(X_i\mid X_{-i})\ge h(\Phi(G,e_i))$. The entropy chain rule then
yields
\[
\mathsf H(G)\;\ge\;\sum_{i=1}^{\ell}h\bigl(\Phi(G,e_i)\bigr)\;\ge\;\ell\,h(\varepsilon),
\qquad\text{hence}\qquad
J_\ell(G)\;\ge\;h(\varepsilon).
\]
The tensor-product configuration
$G_\varepsilon=\operatorname{Bernoulli}(p_\varepsilon)^{\otimes\ell}$
attains equality in the relaxed affinity-constrained problem. Under the stronger finite-walk hypotheses (A1)--(A2), however, the entropy-based rate estimate is strictly larger than $h(\varepsilon)$. Thus the computationally recovered tensor-product structure is explained by the tensorization inequality within the analyzed relaxation. This does not establish feasibility of $G_\varepsilon$ for the original finite-walk construction, nor does it rule out improvements through configurations or dual constructions that lie outside these hypotheses. Full hypotheses, proofs and scope are given in
Supplementary Methods~\ref{sup:obstruction}
(Theorem~\ref{thm:obstruction} and Corollary~\ref{cor:no-improvement}); the
algebraic reductions underlying the theorem were verified against
independent brute-force implementations
(Supplementary Methods~\ref{sup:verification}).

The obstruction constrains only the configuration. It does not constrain
alternative sign functions $\phi$, spectral functions supported on several
configuration orbits, other uses of the partial-Fourier dual families, or
general dual feasible solutions, and the higher-order hierarchy itself
remains complete. The computational search therefore localized the missing
asymptotic freedom in those components rather than in $G$.

%% file: Section3/03methods.tex
\section{Sign-uncertainty discovery and certification}

\subsection*{Fourier-eigenfunction formulation}

We use
\[
\widehat f(\xi)=\int_{\mathbb R^d}
f(x)e^{-2\pi i\langle x,\xi\rangle}\,dx
\]
and consider the $+1$ sign-uncertainty constant $A_{+}(d)$ over non-zero even integrable functions satisfying $\widehat f=f$, $f(0)=0$, and $f(x)\ge0$ for all sufficiently large $|x|$ \cite{bourgain2010principe, cohn2019optimal}. To construct upper bounds we restrict the search to finite-dimensional radial Laguerre
polynomial--Gaussian eigenspaces. With
\[
u=\pi |x|^2,\qquad \alpha=\frac d2-1,
\]
the functions
\[
\psi_n(u)=L_n^{(\alpha)}(2u)e^{-u}
\]
satisfy $\widehat{\psi_n}=(-1)^n\psi_n$. Hence the truncated $+1$ eigenspace is
\[
V_N=\operatorname{span}\{\psi_0,\psi_2,\ldots,\psi_{2N}\}.
\]
Writing $f=P(u)e^{-u}$ converts the Fourier constraint into the linear requirement that $P$ lie in the corresponding even Laguerre span. 
For integer $d$ these generalized Laguerre polynomials have rational coefficients, which enables exact reconstruction.

\subsection*{From neural proposals to a convex search}

The initial discovery stage searched a broader Fourier-eigenfunction family with trainable Gaussian scales and contact locations. Local and neural optimization rapidly improved the objective but became unreliable when the coefficient systems approached condition numbers of $10^{17}$--$10^{18}$.
Independent multiprecision reevaluation showed that apparently improving trajectories were not reproducible. We therefore treated learned widths and contacts only as proposals and used this failure to reconsider the search coordinates rather than to infer that the function class had been exhausted.

For the pure Laguerre family the rescaled function is simply the polynomial
$P(u)=e^u f(u)$. Since $e^{-u}>0$, fixing a candidate last-sign point $u_0$
reduces admissibility to
\[
P(0)=0,\qquad P(u)\ge0\quad (u\ge u_0).
\]
The feasible set in the Laguerre coefficients is convex and monotone in $u_0$. 
Thus, the non-convex optimization over roots and contact locations can be replaced by a one-dimensional bisection in $u_0$ with a convex semi-infinite feasibility problem at each step.

\subsection*{Adaptive semi-infinite feasibility}

At each bisection point we solve a margin-maximizing LP on an adaptive grid. A positive linear normalization,
\[
\int_{u_0}^{\infty}P(u)e^{-u}\,du=1,
\]
removes the conic scale, and the leading Laguerre coefficient is constrained to have the sign required for positivity at infinity. The grid is only a relaxation. After each LP solve, roots of $P'$ are located at high precision; any missed negative minimum is inserted as a cutting plane and the LP is resolved. Near the fixed-degree feasibility boundary, the active constraints identify candidate double contacts.

The exchange loop is used to locate the boundary and its contact structure, not to certify the quoted upper bound. The underlying space is not a Haar system, so uniqueness of the active contact set is not assumed. Instead, the convex semi-infinite formulation supplies a numerical bracket for sharpness within the chosen finite-dimensional space, while exact admissibility is established independently downstream. Failure of the numerical oracle can therefore reduce sharpness but cannot make an invalid certificate valid.

\subsection*{Exact rational reconstruction and verification}

Let $\tau_1<\cdots<\tau_m$ denote the active tangencies returned by the convex search, rounded to nearby rational values. We reconstruct $P$ exactly in the Laguerre basis by imposing
\[
P(0)=0,\qquad
P(\tau_i)=P'(\tau_i)=0\quad(1\le i\le m),
\]
together with a fixed positive top coefficient. In the square certificates, $n_{\mathrm b}=2m+2$ Laguerre coordinates are retained. Solving over $\mathbb Q$ gives
\[
P(u)=\prod_{i=1}^{m}(u-\tau_i)^2R(u),
\qquad
R\in\mathbb Q[u],
\]
with every polynomial division checked to have zero remainder.

Positivity on the complete ray is then decided independently of the LP grid. We verify that the leading coefficient of $R$ is positive, that $R(\bar u)>0$, and that a Sturm chain has the same sign-variation count at $\bar u$ and at $+\infty$. Hence $R$ has no zero on $[\bar u,\infty)$ and therefore
\[
P(u)\ge0\qquad(u\ge\bar u).
\]
Since $\widehat f=f$ and $f(0)=0$ hold exactly by construction,
$f=P(u)e^{-u}$ is admissible and
\[
A_{+}(d)\le\sqrt{\bar u/\pi}.
\]

For the $d=1$ degree-$38$ certificate, $m=9$ and
\[
\bar u=
\frac{514\,997\,856\,761}{5\cdot10^{11}},
\]
which yields
\[
A_{+}(1)\le0.572588699\ldots.
\]
The $d=2$ degree-$22$ certificate is constructed analogously and gives 
$A_{+}(2)\le0.756206236\ldots$. 
The final certificate contains only rational Laguerre data and integer/rational Sturm arithmetic; floating-point computation is confined to discovery of a useful contact structure.

\subsection*{Scope}

The finite-dimensional convex reformulation does not contradict the
Cohn--Dong--Gon\c{c}alves asymptotic ceiling for sublinear-degree
polynomial--Gaussian families \cite{cohn2022sign}. It shows instead that, at small dimension, the classical family had not been numerically exhausted by the previous root-parameterized searches. Detailed tripwires, discovery experiments, convergence diagnostics, contact degeneracies and higher-dimensional numerical limits are given in the Supplementary Methods.

%% file: 05methods-llm.tex
\section{Use of large language models}

A large language model (Claude, Anthropic; Opus 4.5, 4.8 and Fable5) was used throughout the exploratory development of the computational pipeline as an interactive programming and research-assistance tool. Its contributions were concentrated in software engineering (making the scripts into a package) and numerical debugging: proposing and revising implementation strategies, generating and refactoring code, and diagnosing numerical failure modes in the certification pipeline, including floating-point underflow (e.g. float-64 versus float-32 in GPU representations) and dynamic-range collapse, basis-conditioning and rank-instability effects, quadrature overflow, and (formulation) inconsistencies between the discovery and certification stages (see Supplementary Tables \ref{tab:v9changelog} and \ref{tab:crt-history}). The model also served as a discussant for (formal) mathematical framing.

The mathematical content of this work is the author's. The variational formulation, the reduction underlying the pipeline, the certification scheme, and all theorem statements were conceived, derived, and verified by the author; the model's role with respect to the mathematics was limited to exposition, sanity-checking, and the surfacing of relevant literature, all of which the author confirmed. Where the model proposed an algorithmic optimisation, the proposal was accepted only after the author verified it. For reference; PyTorch and Flint are not installed on Anthropic servers so no computations were done by a LLM.

All model outputs were critically evaluated by the author, who bears sole responsibility for the correctness of every claim. All reported numerical bounds are the output of deterministic software. Crucially, the language model is not invoked by the final certification pipeline: the validity of certified bound rests entirely on the exact-arithmetic certificate and its independent verification, which can be checked without reference to any part of the development process. A structured contribution-provenance record, and the complete certification and verification software are provided in the Supplementary Methods.

%% file: Section1/Maynard_SuppMethods_v8.tex

\section{Supplementary Methods: Maynard--Tao sieve}
\label{sup:maynard}

\subsection{Variational definitions and normalization}
\label{sup:maynard_variational_details}

Let
\[
\mathcal R_k
=
\left\{
\mathbf t\in\mathbb R_{\ge0}^{k}:
\sum_{i=1}^{k}t_i\le1
\right\}.
\]
For an admissible square-integrable trial function \(F\) supported on
\(\mathcal R_k\), define
\[
I_k(F)
=
\int_{\mathcal R_k}F(\mathbf t)^2\,d\mathbf t
\]
and
\[
J_k^{(m)}(F)
=
\int_{\mathcal R_{k-1}}
\left(
\int_0^{1-\sum_{i\neq m}t_i}
F(\mathbf t)\,dt_m
\right)^2
d\mathbf t_{-m}.
\]
The Maynard variational constant is
\[
M_k
=
\sup_{F\neq0}
\frac{\sum_{m=1}^{k}J_k^{(m)}(F)}{I_k(F)}.
\]
For symmetric \(F\), all marginal functionals are equal and
\[
M_k
=
k\sup_{F\neq0}\frac{J_k^{(1)}(F)}{I_k(F)}.
\]
Throughout the computational pipeline, the reported Rayleigh quotient is normalized according to this convention. We also use the classical Cauchy--Schwarz \cite{polymath2014variantsselbergsievebounded} ceiling
\[
M_k\le \frac{k}{k-1}\log k,
\]
as a validation constraint on both single-channel and full numerical quotients.
The enlarged-support calculations use the corresponding Polymath
\(\varepsilon\)-variant, with the same discovery--certification separation.
Because the exact normalization of the enlarged polytope and the external
conversion from a certified variational threshold to a numerical \(H_m\)
bound involve the sieve theorem and an admissible \(k\)-tuple rather than
the NeuralCert certificate itself, those ingredients are kept distinct from
the variational certificate. In particular, the certificate proves the
stated lower bound for the explicit trial function; the final \(H_m\)
conversion additionally uses the external Maynard--Tao/Polymath theorem and
a separately supplied admissible-tuple diameter.

\subsection{From exact replication to neural discovery}
\label{sup:maynard_history}

The Maynard experiments began from an exact finite-dimensional control rather than from a neural model. Following Lemmas~8.1 and~8.2 of \cite{maynard2015small}, we used the symmetric polynomial family
\[
F(\mathbf t)
=
\sum_{\ell=1}^{d}
a_\ell
(1-P_1(\mathbf t))^{b_\ell}
P_2(\mathbf t)^{c_\ell},
\qquad
P_j(\mathbf t)=\sum_{i=1}^{k}t_i^j,
\]
for which both \(I_k\) and \(\sum_mJ_k^{(m)}\) reduce to rational quadratic forms in the coefficient vector \(a\). The finite-dimensional optimization is therefore a generalized Rayleigh--Ritz problem. This approach is consistent with Maynard's observation that an extremizer should satisfy an eigenfunction equation for the underlying integral operator. The neural stage can therefore be interpreted as a nonlinear numerical search for a dominant variational eigenfunction, rather than as unconstrained function fitting.
As a first control we reproduced Maynard's explicit \(k=5\) trial
\[
F(\mathbf t)
=
(1-P_1)P_2
+\frac{7}{10}(1-P_1)^2
+\frac{1}{14}P_2
-\frac{3}{14}(1-P_1),
\]
for which
\[
Q_5(F)=\frac{1417255}{708216}>2.
\]
We subsequently reproduced the degree-constrained generalized eigenvalue calculation at \(k=105\). Exact rational assembly was retained throughout; the floating-point generalized eigensolve was used only to locate a Ritz vector, which was then rationalized and reevaluated in the exact quadratic forms.

The exact polynomial controls were next used to calibrate Monte Carlo estimators before any black-box optimization. Uniform simplex samples were generated from normalized independent exponential variables. For the marginal functional, two conditionally independent samples \(U,U'\) were drawn on the same fibre. This gives the unbiased product estimator
\[
\mathbb E[
F(U,\mathbf W)F(U',\mathbf W)\mid\mathbf W]
=
\left(
\frac1C\int_0^CF(u,\mathbf W)\,du
\right)^2,
\]
where \(C=1-\sum W_i\). Squaring a single inner sample would add its conditional variance and bias the estimate upward. Agreement of the Monte Carlo estimates with the exact polynomial values, expressed through standardized residuals, was used as the first stochastic validation gate.

Only after these controls were in place was the fixed polynomial ansatz replaced by a trainable symmetric function. Early neural models used scaled power-sum coordinates
\[
\phi_j(\mathbf t)=k^{j-1}\sum_i t_i^j
\]
and optimized a stochastic Rayleigh quotient. These experiments established
that the optimizer could recover useful symmetric structure without being
given the final polynomial representation, but also exposed the variance,
conditioning and representation-transfer problems that motivated the
deterministic separable pipeline used in the main analysis.

Finally, further control showed that enlargement of the admissible domain does not by itself improve a fixed trial function. For example, Maynard's Eq.~(8.16), optimized for the simplex geometry, decreases from approximately $2.00$ to $1.66$ when evaluated at $\varepsilon=1/4$. The gain from enlarged support therefore requires re-optimization to the modified geometry rather than simple reuse of a simplex-optimized extremizer. This provided an additional motivation for learning the trial function directly in each $\varepsilon$-geometry.

Recent computational work has also explored heuristic modifications of
Maynard-type multidimensional sieves, including dynamically enlarged support
regions and random-matrix-inspired perturbations \cite{Ghadimi2025}, but
their predicted prime-gap improvements remain heuristic. In contrast, the
present framework separates exploratory optimization from certification of
an explicit variational trial function and invokes only independently stated
sieve-theoretic inputs for the final number-theoretic consequence.

\subsection{Exploratory representation experiments}
\label{sup:maynard_exploratory_experiments}

Before arriving at the final separable channel representation, five candidate compression mechanisms, most of which were ultimately rejected. These experiments were diagnostic rather than proof components: their role was to identify which representations preserved the variational structure well enough to justify further development. All bounds in the paper are certified independently of
the outcomes described below.

\subsubsection{Experiment 1: partitions}

Exact Rayleigh--Ritz optimization in the full monomial-symmetric basis is effective at small $k$, but encounters a symmetric curse of dimensionality: although the partition basis is $L^2$-dense, the number of basis elements up to degree $D$ grows as the restricted partition count $p_k(D)$ and becomes rapidly prohibitive. This motivated attempts to retain only a small, variationally relevant subset of partitions.

Three selectors were tested: exact one-vector residual greedy expansion, a global floating-point Ritz oracle, and neural coefficient extraction through a power-sum-product representation. The residual score
\[
\eta_\lambda
=
\frac{
|\langle m_\lambda,(L-\lambda_S)F_S\rangle|
}{
\|m_\lambda\|_2
}
\]
improved the recorded exploratory bound by only about
\(2.5\times10^{-3}\) after adding approximately 30 highest-scoring
partitions, much less than adding a complete degree shell. The global floating-point and neural coefficient selectors were additionally destabilized by the ill-conditioned change of coordinates into the monomial-symmetric basis. The experiment therefore rejected sparse single-partition selection as an effective acceleration mechanism: the missing variational gain was collective and strongly correlated rather than localized in a few coordinates.

This failure suggested that the relevant structure was not sparse in the partition basis but instead closer to a low-complexity product or tensor representation, motivating the tensor-decomposition experiments below.

\subsubsection{Experiment 2: symmetric tensor decompositions}

The failure of sparse partition selection suggested that the optimizer might
be compact in a tensor representation. We therefore fitted the symmetric CP
family
\[
F_{\rm CP}(\mathbf t)
=
\sum_{q=1}^{r}c_q\prod_{i=1}^{k}\phi_q(t_i)
\]
to the certified \(k=5\), degree-\(10\) Ritz optimizer. The relative empirical
\(L^2\) reconstruction errors were
\[
\begin{array}{c|rrrrrr}
r & 1 & 2 & 3 & 4 & 5 & 6\\
\hline
\mathcal E_r
&0.1887&0.1536&0.1032&0.1028&0.1076&0.1089
\end{array}
\]
and plateaued near \(10\%\) beyond rank three. The corresponding Monte Carlo
Rayleigh estimates were statistically compatible with the exact target but
noisy and non-monotone; values above the target were within the quoted Monte
Carlo uncertainty and were never treated as improved lower bounds. Thus,
increasing the symmetric CP rank beyond \(r=3\) produced essentially no
further reduction in reconstruction error over the tested range. This
rejected the simplest low-rank CP compression as a complete representation,
but not the broader tensor-network hypothesis.

\subsubsection{Experiment 3: structural signatures of the CP residual}

For the rank-six symmetric CP fit, the relative empirical reconstruction
error was
\[
\frac{\|F^\star-F_{\mathrm{CP},6}\|_2}{\|F^\star\|_2}
=0.10903.
\]
We projected the residual onto the full degree-10 monomial-symmetric basis.
The largest fitted coordinate signatures occurred in the one-part directions
\[
(4),(3),(2),(5),(1)
\]
and the two-part directions
\[
(2,1),(3,1),(1,1).
\]
Since one-part monomial-symmetric functions are power sums, this pointed to
low-order power-sum structure in the component not captured by the CP fit.

As a diagnostic, we formed the diagonal coordinate proxy
\[
\zeta_\lambda=b_\lambda^2\|m_\lambda\|_{2,\mathrm{emp}}^2.
\]
Approximately \(83\%\) of this proxy was associated with one-body partitions
and \(17\%\) with two-body partitions, with negligible proxy weight at higher
body counts. This must not be interpreted as an orthogonal \(L^2\)-energy
decomposition:
\[
\sum_\lambda\zeta_\lambda=6.71\times10^6,
\]
whereas the empirical squared residual norm was only \(1.43\), revealing
strong cancellation in the non-orthogonal partition basis.

The diagnostic nevertheless suggested that the residual was concentrated in
structurally simple low-order directions rather than in a diffuse high-body
tail. Since such components are in principle representable within a
sufficiently flexible CP family, the observed rank saturation pointed more
naturally to an optimization or conditioning limitation than to an immediate
expressivity obstruction. One plausible explanation, explored in subsequent
experiments, was that these directions were difficult to resolve by
gradient-based fitting against Monte Carlo noise.

\subsubsection{Experiment 4: explicit power-sum corrections}

The residual signature motivated the augmented model
\[
F_{\rm CP+P}
=
F_{\rm CP}
+\beta_0+\sum_{j=1}^{J}\beta_j\widetilde P_j,
\qquad
\widetilde P_j
=
P_j/\mathbb E[P_j].
\]
This experiment did not yield a reliable acceleration. The revised
implementation failed first to reproduce its own CP-only baseline, the Monte Carlo estimator for \(J\) became high variance for some augmented functions, and the explicit power-sum branch overlapped strongly with directions already represented by the CP factors. Consequently, the small observed changes with increasing \(J\) could not be separated cleanly from optimization and implementation signal. 
\subsubsection{Experiment 5: body-restricted partition spaces}

The final pre-separable experiment restricted the exact polynomial trial
space simultaneously by total degree and partition length, or ``body count''.
For
\[
\mathcal V(\{D_b\})
=
\operatorname{span}\{
m_\lambda:\ell(\lambda)=b,\ |\lambda|\le D_b
\},
\]
matrix entries were assembled by enumerating coalescences of the positive
parts of two partitions rather than their full permutation orbits. The
restriction is therefore specified by a degree cap \(D_b\) for each body
count \(b\); at \(k=50\) no run included \(b>5\), so every space reported
below is a genuine body-restricted subspace of the full
monomial-symmetric space of the same total degree. At \(k=5\), where
\(\ell(\lambda)\le k=5\) holds automatically, the unrestricted
degree-\(10\) space contained \(113\) partitions and reproduced
\[
M_5\ge 2.0071443298,
\]
approaching Bogaert's Krylov value \(2.0071451444\). Restricting to
partition length at most two reduced the space to \(36\) basis elements and
gave a strictly weaker value, confirming the coalescence construction while
showing that low-body truncation alone does not recover the full small-\(k\)
optimizer.

At \(k=50\), exact matrix assembly remained feasible, but the generalized
eigensolve became the dominant limitation. Rank-revealing whitening retained
only a fraction of the nominal Gram modes; for a common degree cap of
\(15\), \(18\), \(20\) and \(22\) on every body count the reductions were
\[
408\to191,\qquad
769\to278,\qquad
1125\to344,\qquad
1601\to421.
\]
The retained block consistently approached the imposed conditioning
threshold, \(\kappa\sim10^{12}\), indicating that the effective ceiling was
set by float64 numerical resolution rather than by the nominal size of the
partition space. This effect is visible in the non-monotone certified ladder
obtained from these increasing nominal spaces:
\[
3.36957,\qquad
3.36618,\qquad
3.36800,\qquad
3.36435 .
\]
This does not contradict Rayleigh--Ritz monotonicity: the nominal trial
spaces are nested, whereas the independently whitened effective subspaces
need not be. Exact reevaluation still certifies the returned trial vector,
but cannot restore directions discarded before the eigensolve.

The precision dependence is particularly clear in a matched
\(408\)-dimensional calculation at degree cap \(15\). Retaining the
complete pencil at high precision gave the certified value
\[
M_{50}\ge 3.6347,
\]
whereas float64 spectral whitening retained only \(191\) modes and yielded
\(3.3696\). Clearly, the lost variational gain was already present in the
polynomial trial space but fell below the numerical resolution of the
whitened Gram pencil. The loss is strongly space-dependent rather than
systematic: on the \(441\)-dimensional one--two-body space the
high-precision solve gave \(3.1862092\) against \(3.1862074\) from \(134\)
whitened modes, so there the discarded directions carried almost no
variational weight. A larger space admitting three-body partitions up to
degree \(20\) similarly reached \(3.3904\) with all \(678\) modes retained
at high precision, while larger float64-whitened spaces did not
systematically improve the bound. Supplementary
Table~\ref{tab:precision-vs-dimension-condensed} summarizes representative
calculations.

\paragraph{Why float64 whitening was insufficient in the monomial-symmetric basis.}
One numerical difficulty in the early polynomial calculations was severe ill-conditioning of the denominator Gram matrix in the monomial-symmetric basis. The matrix is positive definite on a linearly independent trial space, but after conversion to double precision its smallest eigenvalues can fall below the numerical resolution of the eigensolver. Whitening then became problematic because the transformation
\[
W = U_{\mathrm{keep}}\operatorname{diag}(\sigma_i^{-1/2})
\]
amplified perturbations associated with small Gram eigenvalues, while directions below the chosen threshold were removed entirely. So, the optimization is no longer performed over the nominal polynomial space, but only over its numerically resolved spectral image. Exact rational re-evaluation of the returned Ritz vector still yields a valid lower bound for that particular vector, but it cannot recover variationally relevant directions that were discarded during whitening. With increasing $k$, larger float64-whitened spaces could produce weaker certified bounds than smaller spaces solved at high precision, demonstrating the numerical precision bottleneck. This motivated representations with Gram structure without spectral truncation.

\renewcommand{\tablename}{Supplementary Table}
\setcounter{table}{0}

\begin{table}[H]
\centering
\caption{Representative \(k=50\) partition-space calculations. The trial
space is specified by a degree cap per body count \(b\) (partition length);
no run included \(b>5\). Certified values refer to the returned explicit
trial function. High-precision solves retain the full trial space, whereas
float64 whitening removes Gram directions below the numerical resolution
threshold.}
\label{tab:precision-vs-dimension-condensed}
\begin{tabular}{llllr}
\toprule
Trial space (degree cap by body count) & Eigensolver & Dimension & Retained & Certified \\
\midrule
\(b\le5\), all degrees \(\le15\)
    & high precision & 408 & 408 & \(3.6347\) \\
\(b\le5\), all degrees \(\le15\)
    & float64 whitening & 408 & 191 & \(3.3696\) \\
\(b\le2\), degrees \(\le40\)
    & float64 whitening & 441 & 134 & \(3.1862\) \\
\(b\le3\), degrees \((40,40,20)\)
    & high precision & 678 & 678 & \(3.3904\) \\
\(b\le5\), all degrees \(\le22\)
    & float64 whitening & 1601 & 421 & \(3.3644\) \\
\bottomrule
\end{tabular}

\par\smallskip
\parbox{0.95\linewidth}{%
\footnotesize
\textit{Note.} Bogaert reported \(M_{50}\ge 3.9358660346\).
}
\end{table}

The experiment therefore identified two distinct obstacles in the
monomial-symmetric representation: rapid growth of the partition basis and
loss of numerically resolvable Gram directions. The limiting issue at
\(k=50\) was not expressive power of the polynomial space itself, but the
combination of basis growth and near-degeneracy. This motivated the move to
representations whose complexity is controlled before the generalized
eigensolve and whose structure is more directly compatible with exact
certification.

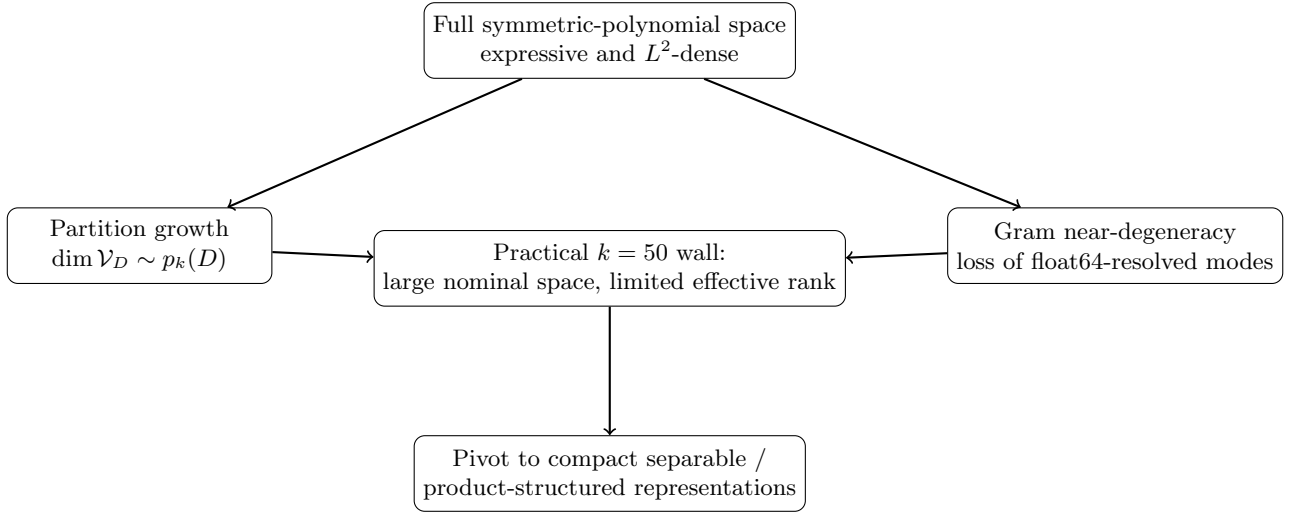
\begin{figure}[H]
\centering
\begin{tikzpicture}[
    node distance=1.7cm and 2.0cm,
    every node/.style={font=\small},
    box/.style={
        draw,
        rounded corners,
        align=center,
        minimum width=3.5cm,
        minimum height=1.0cm
    },
    arrow/.style={->, thick}
]

\node[box] (poly) {
Full symmetric-polynomial space\\
expressive and \(L^2\)-dense
};

\node[box, below left=of poly] (growth) {
Partition growth\\
\(\dim \mathcal V_D \sim p_k(D)\)
};

\node[box, below right=of poly] (gram) {
Gram near-degeneracy\\
loss of float64-resolved modes
};

\node[box, below=2.0cm of poly] (wall) {
Practical \(k=50\) wall:\\
large nominal space, limited effective rank
};

\node[box, below=of wall] (pivot) {
Pivot to compact separable /\\
product-structured representations
};

\draw[arrow] (poly) -- (growth);
\draw[arrow] (poly) -- (gram);
\draw[arrow] (growth) -- (wall);
\draw[arrow] (gram) -- (wall);
\draw[arrow] (wall) -- (pivot);

\end{tikzpicture}
\caption{
Scaling obstruction of the monomial-symmetric representation. Increasing
polynomial degree expands an expressive symmetric trial space, but at
\(k=50\) this produces both rapid partition growth and increasingly
ill-conditioned Gram pencils. The resulting loss of numerically resolved
directions motivates the later separable representation.
}
\label{fig:partition-wall}
\end{figure}

\subsubsection{Consequence for the final representation}

The abovementioned five experiments ruled out several superficially natural routes: sparse partition support, very-low-rank symmetric CP compression, unqualified interpretation of residual coefficients, additive power-sum patches, and brute-force growth of ill-conditioned polynomial spaces. The common failure mode was not lack of a strong candidate function but loss of structure during representation or optimization. This motivated the separable-channel formulation used in the main Methods, in which a small dictionary of one-dimensional functions is optimized directly, the generalized Ritz coefficients are profiled out, and independent certification operates only after the discovered function has been reduced to an explicit mathematical representation.

Scripts for the exploratory analyses described above are not included in NeuralCert. Together with additional exploratory outputs including full residual rankings, body-count and parity summaries, and repeated fit diagnostics, they are available upon request.

\subsection{Neural channel model}
\label{sup:channel_model}

The discovery stage uses the symmetric separable representation
\[
F_{\theta,c}(t_1,\ldots,t_k)
=
\sum_{j=1}^{m}c_j\prod_{i=1}^{k}g_{\theta,j}(t_i).
\]
The one-dimensional channels are written
\[
g_{\theta,j}(x)
=
r_{\theta,j}(x)e^{-\rho_jx},
\qquad \rho_j>0,
\]
with trainable rates initialized across a broad geometric range. The residual factors are produced
jointly by a shared multilayer perceptron with channel-specific outputs and input
\[
(x,e^{-kx}).
\]
For the positive-channel model,
\[
r_{\theta,j}(x)=
\exp\!\left(f_{\theta,j}(x,e^{-kx})\right),
\qquad
g_{\theta,j}(x)=
\exp\!\left(f_{\theta,j}(x,e^{-kx})-\rho_jx\right).
\]
Positivity permits cancellation-free log-domain convolution. The architecture differs from a conventional DeepSets model: permutation invariance follows from the product structure of the separable expansion, while the neural network supplies a dictionary of one-dimensional factors.

\begin{lstlisting}[style=neuralcert,
  caption={Positive neural channel dictionary. A shared network produces the
  residual factors, while the exponential envelopes are represented
  analytically.},
  label={lst:channel-model}]
PROCEDURE CHANNELS(x, k, theta, eta, rho_max)
    rho     <- MIN(SOFTPLUS(eta), rho_max)
    features <- STACK(x, EXP(-k*x))
    log_raw <- SHARED_MLP_theta(features)
    g       <- EXP(log_raw - OUTER(x, rho))
    RETURN g, log_raw, rho
END

FUNCTION TRIAL(t, c, theta)
    RETURN SUM_j c[j] * PRODUCT_i CHANNEL_j(t[i], theta)
END
\end{lstlisting}

\subsection{Factored convolution powers and Gauss--Jacobi recursion}
\label{sup:factored_convolution}

For a channel pair $(j,\ell)$, write
\[
h_{j\ell}(x)
=
r_{j\ell}(x)e^{-a_{j\ell}x},
\qquad
r_{j\ell}(x)=r_{\theta,j}(x)r_{\theta,\ell}(x),
\qquad
a_{j\ell}=\rho_j+\rho_\ell.
\]
Every convolution power is represented as
\[
\nu^{(j\ell)}_p(s)
:=
h_{j\ell}^{*p}(s)
=
s^{p-1}e^{-a_{j\ell}s}
e^{\sigma^{(j\ell)}_p}\psi^{(j\ell)}_p(s).
\]
The initialization is
\[
\sigma^{(j\ell)}_1=0,
\qquad
\psi^{(j\ell)}_1(s)=r_{j\ell}(s).
\]
We suppress the channel-pair superscript below.

For positive integers $p$ and $q$, substitution into
\[
\nu_{p+q}(s)=\int_0^s\nu_p(x)\nu_q(s-x)\,dx
\]
and the change of variables $x=su$ give
\[
\nu_{p+q}(s)
=
s^{p+q-1}e^{-as}e^{\sigma_p+\sigma_q}
\int_0^1
\psi_p(su)\psi_q(s(1-u))
u^{p-1}(1-u)^{q-1}\,du.
\]
Let
\[
B(p,q)
=
\int_0^1u^{p-1}(1-u)^{q-1}\,du
=
\frac{\Gamma(p)\Gamma(q)}{\Gamma(p+q)}.
\]
We retain the beta factor in the residual recursion and therefore use
\[
\sigma_{p+q}=\sigma_p+\sigma_q,
\]
together with
\[
\psi_{p+q}(s)
=
B(p,q)\int_0^1
\psi_p(su)\psi_q(s(1-u))
\frac{u^{p-1}(1-u)^{q-1}}{B(p,q)}\,du.
\]
Thus the integral is an expectation with respect to the normalized
$\operatorname{Beta}(p,q)$ density. The factor $B(p,q)$ occurs once,
in $\psi_{p+q}$, and is not also included in $\sigma_{p+q}$.

The expectation is evaluated by Gauss--Jacobi quadrature:
\[
\psi_{p+q}(s)
\approx
B(p,q)\sum_{r=1}^{n_q}
\widehat w_r^{(p,q)}
\psi_p(su_r^{(p,q)})
\psi_q(s(1-u_r^{(p,q)})),
\]
where
\[
0<u_r^{(p,q)}<1,
\qquad
\widehat w_r^{(p,q)}>0,
\qquad
\sum_{r=1}^{n_q}\widehat w_r^{(p,q)}=1.
\]
The nodes and normalized weights incorporate the beta-weight
concentration associated with the convolution orders $p$ and $q$.

An optional numerical normalization can transfer residual amplitude
into the scalar scale. If $\widetilde\psi_{p+q}$ denotes the residual
produced by the recursion above, choose a positive scalar
$C_{p+q}$, independent of $s$, and set
\[
\psi_{p+q}(s)
=
\frac{\widetilde\psi_{p+q}(s)}{C_{p+q}},
\qquad
\sigma_{p+q}
=
\sigma_p+\sigma_q+\log C_{p+q}.
\]
This leaves the represented convolution power unchanged. Taking
$C_{p+q}=1$ recovers the unnormalized recursion. For positive residuals,
the factor $B(p,q)$ is handled through $\log B(p,q)$ in the log-domain
recursion; any additional normalization subtracts $\log C_{p+q}$ from
the log residual and adds it to $\sigma_{p+q}$.

The power $h^{*(k-1)}$ is assembled by a binary addition chain,
reducing convolution depth from $k-2$ sequential stages to
$O(\log k)$.

\subsection{Log-domain recursion}
\label{sup:log_convolution}

For positive channels define $\xi_p(s)=\log\psi_p(s)$. The recursion becomes
\[
\xi_{p+q}(s)
=
\log B(p,q)
+
\operatorname{logsumexp}_r
\left[
\log\widehat w_r^{(p,q)}
+\xi_p(su_r^{(p,q)})
+\xi_q(s(1-u_r^{(p,q)}))
\right].
\]
Values that would underflow in linear arithmetic remain representable as large negative logarithms, and the positive quadrature sum is evaluated without cancellation. This is important because loss of a small component during an early convolution can propagate into entire entries of the subsequent Gram matrices.

\begin{lstlisting}[style=neuralcert,
  caption={Binary-chain evaluation of the factored convolution power in
  logarithmic coordinates.},
  label={lst:factored-log-convolution}]
PROCEDURE LOG_CONVOLUTION_POWER(log_raw, j, l, k)
    xi[1](s) <- log_raw_j(s) + log_raw_l(s)

    FOR EACH (p,q) IN BINARY_ADDITION_CHAIN(k-1)
        (u,w) <- NORMALIZED_GAUSS_JACOBI(p,q)
        FOR EACH representation node s
            a <- LOG_POSITIVE_INTERPOLATE(xi[p], s*u)
            IF q = 1
                b <- EVALUATE_BASE_PAIR_EXACTLY(s*(1-u), j, l)
            ELSE
                b <- LOG_POSITIVE_INTERPOLATE(xi[q], s*(1-u))
            END
            xi[p+q](s) <- LOG_BETA(p,q)
                          + LOGSUMEXP_r(LOG(w[r])+a[r]+b[r])
        END
    END
    RETURN xi[k-1]
END
\end{lstlisting}

\subsection{Representation and integration grids}
\label{sup:discovery_grids}

Residual functions are stored on an endpoint-inclusive Chebyshev--Lobatto grid
\[
0=s_0<s_1<\cdots<s_{N-1}=L,
\qquad L=1+\varepsilon.
\]
Every convolution target $s_i u_r$ or $s_i(1-u_r)$ lies in $[0,s_i]$, so no extrapolation is required.
Four numerical point sets are kept distinct: the Chebyshev--Lobatto representation grid;
Gauss--Jacobi rules for the beta-weighted convolution; Gauss--Legendre rules for the primitives
\[
G_j(y)=\int_0^y g_j(x)\,dx,
\qquad
H_{j\ell}(y)=\int_0^y g_j(x)g_\ell(x)\,dx;
\]
and a geometrically graded outer mesh for the final entries of $A$ and $B$. The latter must resolve
the pair-dependent factor $x^{k-2}e^{-a_{j\ell}x}$.

\begin{lstlisting}[style=neuralcert,
  caption={Construction of the four numerical point sets used by the
  discovery evaluator.},
  label={lst:numerical-grids}]
PROCEDURE BUILD_GRIDS(k, epsilon, N, n_jacobi, n_inner, n_outer)
    L <- 1 + epsilon;  U <- 1 - epsilon
    s[i] <- L*(1-COS(pi*i/(N-1)))/2,  i=0,...,N-1

    FOR EACH (p,q) IN BINARY_ADDITION_CHAIN(k-1)
        jacobi[p,q] <- NORMALIZED_GAUSS_JACOBI(n_jacobi,p,q)
    END
    inner <- GAUSS_LEGENDRE(n_inner, interval=[0,1])

    x_star <- MIN(X,(k-2)/a_ref)
    edges  <- GEOMETRIC_MESH(x_star/32,X)
    outer  <- COMPOSITE_GAUSS_LEGENDRE(edges,n_outer)
    // The same positive-weight outer rule is used for every channel pair.
    RETURN s, jacobi, inner, outer
END
\end{lstlisting}

\subsection{Preservation of Gram structure}
\label{sup:gram_preservation}

The exact denominator matrix is a Gram matrix and therefore satisfies
$A\succeq 0$. To preserve this structure numerically, a target
$z\in[s_i,s_{i+1}]$ is interpolated using the positive stencil
\[
\psi(z)
\approx
\omega_i(z)\psi(s_i)+\omega_{i+1}(z)\psi(s_{i+1}),
\qquad
\omega_i,\omega_{i+1}\geq 0,
\qquad
\omega_i+\omega_{i+1}=1.
\]
In log-space, this interpolation is evaluated by a two-term log-sum-exp.
Interpolating $\psi$ itself, rather than $\log\psi$, is what makes the
stencil multilinear and hence compatible with the Gram form. When the inner
quadrature nodes are common to all channel pairs, the shared positive-weight
outer rule yields the discrete Gram representation
\[
A^{(N)}_{j\ell}
=
\sum_\alpha w_\alpha
\Phi_{\alpha j}\Phi_{\alpha\ell},
\qquad
w_\alpha>0,
\]
and hence $A^{(N)}\succeq 0$ up to floating-point round-off. The shared mesh
therefore preserves a common discrete Gram representation rather than merely
improving entrywise quadrature accuracy.

In the saturated regime, where the substitution used for the channel
primitives places the inner nodes at pair-dependent positions, the common
discrete Gram representation is preserved only up to inner-quadrature error.
Numerical adequacy is then assessed using the positive-semidefiniteness,
two-grid and refinement diagnostics described below; these diagnostics are
safeguards rather than rigorous bounds on the quadrature error. Since the
reported bounds come from the exact certifier and not from the discrete
pencil, a failure of this structure degrades the discovery objective rather
than the validity of any certificate.

\begin{lstlisting}[style=neuralcert,
  caption={Log-domain assembly of one pair of Gram entries. With common inner
  quadrature nodes, the shared positive-weight outer rule gives an exact
  discrete Gram representation; in the saturated regime, pair-dependent inner
  quadrature preserves this structure only up to inner-quadrature error.},
  label={lst:gram-pair-assembly}]
FUNCTION LOG_INTERPOLATE(z, xi)
    (i,i+1,omega0,omega1) <- POSITIVE_TWO_POINT_STENCIL(z)
    RETURN LOGADDEXP(LOG(omega0)+xi[i], LOG(omega1)+xi[i+1])
END

PROCEDURE GRAM_PAIR(j, l)
    a  <- rho[j] + rho[l]
    xi <- LOG_CONVOLUTION_POWER(log_raw,j,l,k)
    log_Ghat, log_Hhat <- GAUSS_LEGENDRE_PRIMITIVES(j,l)
    // Saturated inner rules may depend on the channel pair (j,l).

    log_A[j,l] <- LOGSUMEXP_x((k-2)*LOG(x)-a*x
                    + LOG_INTERPOLATE(x,xi)
                    + LOG(L-x)+LOG_INTERPOLATE(L-x,log_Hhat)+LOG(w_x))
    log_B[j,l] <- LOGSUMEXP_x((k-2)*LOG(x)-a*x
                    + LOG_INTERPOLATE(x,xi)+2*LOG(L-x)
                    + LOG_INTERPOLATE(L-x,log_Ghat_j)
                    + LOG_INTERPOLATE(L-x,log_Ghat_l)+LOG(w_x))
    RETURN log_A[j,l], log_B[j,l]
END
\end{lstlisting}

\subsection{Diagonal equilibration and rank-revealing Ritz solution}
\label{sup:ritz_details}

With
\[
D=\operatorname{diag}(A_{11},\ldots,A_{mm}),
\]
the generalized Rayleigh quotient is invariant under
\[
A\mapsto D^{-1/2}AD^{-1/2},
\qquad
B\mapsto D^{-1/2}BD^{-1/2}.
\]
After equilibration, write
\[
\widetilde A=U\operatorname{diag}(\alpha_1,\ldots,\alpha_m)U^\top.
\]
Directions satisfying
\[
\alpha_i\leq\tau_{\rm rank}\max_j\alpha_j
\]
are removed before whitening. On the retained space,
\[
W=
U_{\rm keep}\operatorname{diag}(\alpha_i^{-1/2}),
\]
and the leading eigenpair of $W^\top\widetilde BW$ determines the profiled quotient and mixing
vector. No ridge term is introduced; rank truncation is treated as a numerical-resolution decision
and monitored by rank-stability tests.

\begin{lstlisting}[style=neuralcert,
  caption={Diagonally equilibrated, rank-revealing generalized Ritz solve.},
  label={lst:rank-revealing-ritz}]
PROCEDURE TOP_RITZ(log_A, log_B, tau_rank)
    d[j]   <- log_A[j,j]
    Ahat[j,l] <- EXP(log_A[j,l]-(d[j]+d[l])/2)
    Bhat[j,l] <- EXP(log_B[j,l]-(d[j]+d[l])/2)

    (alpha,U) <- SYMMETRIC_EIGENDECOMPOSITION(Ahat)
    keep <- {i : alpha[i] > tau_rank*MAX(alpha)}
    REQUIRE keep is nonempty
    W <- U[:,keep] * DIAG(alpha[keep]^(-1/2))

    (lambda,z) <- LEADING_EIGENPAIR(TRANSPOSE(W)*Bhat*W)
    v <- W*z
    v <- v/SQRT(TRANSPOSE(v)*Ahat*v)
    RETURN k*lambda, v, CARDINALITY(keep), d
END
\end{lstlisting}

\subsection{Hellmann--Feynman gradients and optimization}
\label{sup:optimization_details}

Numerical solution of the generalized Ritz problem may require removal of
directions associated with numerically negligible eigenvalues of
$A(\theta)$. Let
\[
\mathcal S(\theta)
\]
denote the retained coefficient subspace after this rank-revealing
truncation, and define the corresponding numerical Ritz value by
\[
\widehat R(\theta)
=
\max_{\substack{v\in\mathcal S(\theta)\\v\ne0}}
k\frac{v^\top B(\theta)v}
       {v^\top A(\theta)v}.
\]
Because $\mathcal S(\theta)$ may change with $\theta$, the map
$\widehat R(\theta)$ need not be differentiable when a singular direction
crosses the truncation threshold. Moreover, the standard
Hellmann--Feynman or envelope identity does not automatically apply across
such a parameter-dependent change of retained subspace.

We therefore freeze the retained subspace during each differentiable
optimization block. At outer iteration $s$, let
\[
\mathcal S_s:=\mathcal S(\theta_s)
\]
and define the fixed-subspace profile
\[
R_s(\theta)
=
\max_{\substack{v\in\mathcal S_s\\v\ne0}}
q_v(\theta),
\qquad
q_v(\theta)
=
k\frac{v^\top B(\theta)v}
       {v^\top A(\theta)v}.
\]
Let $v_s\in\mathcal S_s$ be the normalized leading generalized Ritz
vector at the anchor point:
\[
B(\theta_s)v_s
=
\lambda_s A(\theta_s)v_s
\quad\text{on }\mathcal S_s,
\qquad
v_s^\top A(\theta_s)v_s=1.
\]
Provided that $A(\theta)$ remains positive definite on $\mathcal S_s$ in
a neighbourhood of $\theta_s$ and that the leading restricted
generalized eigenvalue is simple, the fixed-subspace envelope identity
gives
\[
\nabla_\theta R_s(\theta_s)
=
\nabla_\theta q_{v_s}(\theta_s).
\]
Thus the Ritz vector may be frozen when differentiating the quotient, but
only while the retained subspace is held fixed.

For every fixed $v\in\mathcal S_s$,
\[
q_v(\theta)\le R_s(\theta),
\]
and at the anchor point
\[
q_{v_s}(\theta_s)=R_s(\theta_s).
\]
Consequently, an inner step satisfying
\[
q_{v_s}(\theta_{s+1})
\ge
q_{v_s}(\theta_s)
\]
obeys the fixed-subspace minorize--maximize relation
\[
R_s(\theta_{s+1})
\ge
q_{v_s}(\theta_{s+1})
\ge
q_{v_s}(\theta_s)
=
R_s(\theta_s).
\]
This monotonicity statement concerns $R_s$, with $\mathcal S_s$ fixed. It
does not by itself imply monotonicity of the adaptively truncated value
$\widehat R(\theta)$ after the numerical rank is recomputed.

Neural channels are first optimized with Adam using learning-rate warm-up,
cosine decay, and gradient clipping. Candidate solutions are subsequently
polished by block minorize--maximize L-BFGS. Within each such block, the
retained subspace, truncation mask, and frozen Ritz vector are held fixed.
The rank-revealing decomposition is recomputed only between blocks.

After this recomputation, a candidate outer step is accepted only if the
adaptively evaluated Ritz value satisfies
\[
\widehat R(\theta_{s+1})
\ge
\widehat R(\theta_s)-\tau_{\mathrm{acc}},
\]
where $\tau_{\mathrm{acc}}$ is a prescribed numerical acceptance
tolerance. If a change in retained rank causes this test to fail, the
candidate is rejected and the step size or optimization block is
restarted. This acceptance test is an explicit numerical safeguard; it is
not presented as a consequence of the Hellmann--Feynman identity.

Finite-difference gradient tests are performed with the retained subspace
and truncation mask frozen. A finite-difference stencil that changes the
retained numerical rank crosses a nonsmooth truncation boundary and is
therefore not used as a test of the Hellmann--Feynman gradient. Such cases
are instead flagged as rank-transition events. Away from these events, the
test checks the complete differentiation path through the neural channels,
factored convolution recursion, integration, and matrix assembly.

Optimization proceeds through increasing channel ranks
\[
m_1<\cdots<m_S
\]
and representation resolutions
\[
N_1<\cdots<N_S.
\]
Existing channels are copied when the channel rank is increased, and new
channels are initialized as perturbed duplicates. Because pairwise
assembly scales as $m(m+1)/2$, more iterations are allocated to the
smaller early models.

\begin{lstlisting}[style=neuralcert,
  caption={Profiled optimization with a frozen Ritz vector.},
  label={lst:profiled-optimization}]
FUNCTION FROZEN_QUOTIENT(theta, v)
    (A_theta,B_theta) <- ASSEMBLE_GRAM(theta)
    RETURN k*(TRANSPOSE(v)*B_theta*v)/(TRANSPOSE(v)*A_theta*v)
END

PROCEDURE OPTIMIZE_CHANNELS(theta)
    FOR EACH increasing pair (channel rank m, grid size N)
        GROW_DICTIONARY_BY_PERTURBED_DUPLICATES(theta,m)
        FOR EACH Adam iteration, with warm-up and cosine decay
            IF this is a reporting point
                VALIDATE_CURRENT_STATE(theta)   // before the update
            END
            (_,v) <- TOP_RITZ(ASSEMBLE_GRAM(theta))
            theta <- CLIPPED_ADAM_ASCENT(
                         FROZEN_QUOTIENT(theta,STOP_GRADIENT(v)))
        END
    END

    REPEAT for each block minorize-maximize step
        (_,v) <- TOP_RITZ(ASSEMBLE_GRAM(theta))
        theta <- LBFGS_MAXIMIZE(FROZEN_QUOTIENT(theta,STOP_GRADIENT(v)))
        VALIDATE_CURRENT_STATE(theta)
        UPDATE incumbent only if its two-grid score improves
    END
    RETURN theta
END
\end{lstlisting}

\subsection{Validation-aware model selection}
\label{sup:validation_details}

At each reporting point the candidate must satisfy:
\begin{enumerate}
\item agreement between the training objective and an independently recomputed Rayleigh quotient;
\item numerical positive semidefiniteness of the denominator matrix;
\item retention of a prescribed minimum fraction of the channel span after rank truncation;
\item stability of the leading quotient under a more conservative rank threshold;
\item compliance of every diagonal ratio $kB_{jj}/A_{jj}$ with the applicable rigorous upper bound;
\item compliance of the full quotient with the corresponding rigorous upper bound; and
\item reproduction on a finer representation grid within the prescribed snapshot tolerance.
\end{enumerate}
When two grids are used during training, candidates are ranked by
\[
R_{\rm score}=\min\{R_N,R_{N_{\rm hi}}\}.
\]
A substantial negative eigenvalue of $A$ is treated as failure of the numerical representation rather
than repaired by discarding the negative branch. Rank-stability tests separately guard against
whitening directions below the accuracy of the assembled matrices.

\begin{lstlisting}[style=neuralcert,
  caption={Validation and two-grid model selection applied at each reporting
  point.},
  label={lst:validation-gates}]
FUNCTION RIGOROUS_CEILING(k, epsilon)
    IF epsilon = 0
        RETURN k*LOG(k)/(k-1)
    ELSE
        RETURN k*LOG(2*k-1)/(k-1)
    END
END

PROCEDURE VALIDATE_CURRENT_STATE(theta)
    R_bound <- RIGOROUS_CEILING(k,epsilon)

    (R_train,A,B,v,rank) <- INDEPENDENT_RECOMPUTATION(theta)
    REQUIRE AGREES(R_train,FROZEN_QUOTIENT(theta,v))
    REQUIRE MIN_EIGENVALUE(A) >= -tau_psd*MAX_EIGENVALUE(A)
    REQUIRE rank >= required_fraction*NUMBER_OF_CHANNELS(theta)
    REQUIRE STABLE_UNDER_STRONGER_RANK_TRUNCATION(A,B)
    REQUIRE EACH k*B[j,j]/A[j,j] <= R_bound
    REQUIRE R_train <= R_bound

    R_hi <- RECOMPUTE_ON_GRID(theta,1.5*N)
    REQUIRE RELATIVE_DIFFERENCE(R_train,R_hi) <= tau_snapshot
    score <- MIN(R_train,R_hi)
    STORE theta only if score improves the validated incumbent
    RETURN score
END
\end{lstlisting}

\subsection{Closed-form controls and representation refinement}
\label{sup:refinement_details}

Every run begins with closed-form controls. For $g(x)=1$ and $\varepsilon=0$,
\[
R_{\rm const}=\frac{2k}{k+1},
\]
and for $g(x)=x$,
\[
R_x=\frac{3k}{3k+1}.
\]
The first exercises convolution, outer integration and scaling; the second also probes interpolation
of a nonconstant residual convolution shape.

After training, candidates are evaluated on grids $N$, $1.25N$ and $1.5N$ and fitted to
\[
R_N=R_\infty-cN^{-p}.
\]
The fitted limit is diagnostic only. the export is marked unconverged when the sequence is noncontracting,
violates a rigorous ceiling or fails the refinement tolerance. No extrapolated discovery value is
treated as a certified lower bound.

The final separable discovery solver emerged through a sequence of numerical and structural corrections. Supplementary Table~\ref{tab:v9changelog} summarizes the principal failure modes identified during development and the corresponding fixes. These historical iterations are not separate proof components; they document how the numerical invariants enforced in the final
solver were established.

Supplementary Table~\ref{tab:v9changelog} lists the different development steps for the discovery script. 

\begingroup
\small
\begin{longtable}{@{}p{0.5cm}p{6.5cm}p{7cm}@{}}
\caption{Reconstructed development history of the \texttt{maynard\_separable\_rayleigh\_v9\_factored}
lineage (neural discovery). Version numbers map to successive saved iterations.}\label{tab:v9changelog}\\
\hline
\textbf{\#} & \textbf{Correction} & \textbf{Cause / symptom} \\
\hline
\endfirsthead
\multicolumn{3}{c}{\tablename\ \thetable\ -- continued}\\
\hline
\textbf{\#} & \textbf{Correction} & \textbf{Cause / symptom} \\
\hline
\endhead
\hline \multicolumn{3}{r}{\emph{continued on next page}}\\
\endfoot
\hline
\endlastfoot
1 & Baseline v9 (factored): $\nu_p=s^{p-1}e^{-as}e^{\sigma}\psi_p$ per pair; stable \texttt{inv\_softplus}; per-pair max-rescale; adaptive log-space GL outer windows; diagonal preconditioning ($\operatorname{diag}A=1$); \texttt{parts()} protocol. & Foundation. \texttt{inv\_softplus} fixes $\mathrm{softplus}^{-1}=\log(\mathrm{expm1})\to+\infty$ for rate $\ge 710$ (all $k\ge355$); per-pair rescale replaces a global max that flushed small pairs to round-off. \\
2 & $\sigma$ double-count fix: divide exact base re-evaluations by $e^{-\sigma_1}$. & Exact base evals carried $e^{\sigma_1}$ while $\sigma_1$ was also added in \texttt{new\_sig}; the spurious factor is not of Gram form $\Rightarrow$ indefinite $A$. \\
3 & \texttt{train()} validation-ordering fix: validate before \texttt{opt.step()}. & $R_{\text{train}}$ read pre-step, \texttt{report()} post-step; \texttt{agree} measured the Adam step, rejecting nearly all snapshots (fallback to iter 700). \\
4 & Log-space channel protocol: \texttt{log\_parts}; \texttt{-{}-channel-sign positive/free}. & Prerequisite for the log-space recursion; needs sign-definite channels. \\
5 & Log-space $\xi$ recursion: $\xi_{p+q}=\log B(p,q)+\operatorname{logsumexp}(\log\hat w+\xi_p+\xi_q)$. & The $(\min\psi_1/\max\psi_1)^{k-1}$ collapse drove $\psi$ past the float64 floor; round-off put random signs on diagonal pairs ($\psi\ge0$), making $A$ indefinite. \\
6 & Shared graded outer mesh (one geometric Gauss rule for all pairs). & Per-pair windows gave each entry its own rule $\Rightarrow A$ the Gram matrix of nothing; $\lambda_{\min}/\lambda_{\max}\approx-1$ and did not improve under refinement. \\
7 & Log-singularity fix: $G(y)=y\,\hat G(y)$; interpolate $\log\hat G$, reattach $\log y$ analytically. & $G(y)$ vanishes linearly $\Rightarrow\log G\to-\infty$ (clamped); interpolating that oscillated $\Rightarrow B\approx10^{45}$. \\
8 & Gram-preserving positive interpolation: two-point stencil, weights $\ge0$, via \texttt{logaddexp}; \texttt{-{}-interp positive/spectral}. & Interpolating the log gives $\prod\psi_i^{M_i}$ (geometric), not $\sum M_i\psi_i$ (linear) $\Rightarrow$ not multilinear $\Rightarrow$ not Gram $\Rightarrow$ indefinite $A$. \\
9 & Hard gates (PSD, $R\le k$, gradient) + removed false ``PSD by construction'' + parameter-independent mesh. & $R=118,\,979,\,5.2\times10^{8}$ were optimised toward; positive-entrywise $\ne$ PSD; moving mesh gave a spurious HF-gradient gap. \\
10 & Resolution diagnostics ($\xi$ span/jump), $g=x$ preflight ($3k/(3k{+}1)$), interp-dependent $n_{\text{rep}}$, sharp bound $\tfrac{k}{k-1}\ln k$, Richardson gate, $N$ provenance. & $g\equiv1$ is exact for any $N$ in positive mode (blind to under-resolution); $R\le k$ too loose to catch a $6\to49$ climb. \\
11 & Refinement-stable snapshots + strikes + $1.5\times$ validation twin. & Optimiser maximises $R_{\text{true}}+\mathrm{error}(\theta)$; refinement-instability is the exploitation signature and must gate snapshots. \\
12 & Two-tier snapshots + $\min(R_N,R_{1.5N})$ ranking + fitted-order Richardson + grid continuation + result reorder. & One tolerance did two jobs and rejected every honest state; trained-channel order $p\approx1.6\ne2$, so fixed-order Richardson mis-extrapolated. \\
13 & $\varepsilon>0$ bound fix: $R_{\text{bound}}=k$ for $\varepsilon>0$; tight bound only for $\varepsilon=0$. & $M_{k,\varepsilon,1/2}$ legitimately exceeds the vanilla bound ($4.02>3.99$ at $k=50$), so the tight bound falsely aborted. Later aligned this with proposition 6.5 from Polymath8b to 
\[
R_{\mathrm{bound}}=
\frac{k}{k-1}\log(2k-1)
\qquad (\varepsilon>0).
\] There is a more stringent bound in Polymath8b Remark 6.6 but this suffices as hard-coded discovery threshold / safeguard. \\
14 & Per-pair diagnostic + \texttt{-{}-dump-violation} + \texttt{-{}-patience}; $k$-aware panels; excursion guard; \texttt{polish()} incumbent seeding. & Fixed 8-panel mesh under-resolved $x^{k-2}$ at $k\ge250$ ($e^{152}$ vs $2n_{\text{out}}{=}128$); \texttt{polish()} started at $-\infty$ so step~1 installed $R=70$ as best. \\
15 & \texttt{-{}-trunc-floor} prevention + rank-stability detection + strike-counter fix. & $R=56$ lived in $A$-directions at $\lambda/\lambda_{\max}\sim10^{-6}$, below assembly accuracy $10^{-3}$; whitening amplified error (again!) $\sim10^{3}$. \texttt{strikes} reset before the check $\Rightarrow$ never accrued. \\
16 & Stability recalibration (25\%, floored) + absolute $R\le k$ backstop + excursion baseline anchor + grow-step guard + Richardson already-converged path. & $5\%$ flagged honest wobble; $\text{best}=-\infty$ let a runaway pass at iter 0; a $0.375$ baseline made $3.66$ look $9.7\times$; a post-grow state crashed the run; a converged $k{=}500$ run got a false FAIL. \\
\end{longtable}
\endgroup

\begin{lstlisting}[style=neuralcert,
  caption={Closed-form preflights and post-training representation refinement.},
  label={lst:controls-refinement}]
PROCEDURE PREFLIGHT_AND_REFINE(theta,k,epsilon,N)
    REQUIRE NUMERICAL_QUOTIENT(g=1) AGREES WITH CLOSED_FORM_CONSTANT(k,epsilon)
    IF epsilon = 0
        REQUIRE NUMERICAL_QUOTIENT(g=x) AGREES WITH 3*k/(3*k+1)
    END

    FOR Nq IN {N, CEIL(1.25*N), CEIL(1.5*N)}
        R[Nq] <- VALIDATED_QUOTIENT(theta,Nq)
    END
    IF successive differences contract with a common sign
        FIT R[Nq] = R_infinity-c*Nq^(-p)
        REQUIRE 0.8 <= p <= 4
        REQUIRE ABS(R[1.5*N]-R_infinity)/ABS(R_infinity) <= tau_refine
        REQUIRE R_infinity <= RIGOROUS_CEILING(k,epsilon)
    ELSE
        MARK candidate as numerically unresolved
    END
    RETURN R, R_infinity, p
END
\end{lstlisting}

\subsection{Aggregate modular evaluation}
\label{sup:aggregate_modular}

After projection, let $q_j,q_\ell$ be rational channel polynomials and
\[
h_{j\ell}(u)=q_j(u)q_\ell(u).
\]
Rather than reconstructing every exact entry of the Gram matrices, the
aggregate certifier accumulates the weighted pair contributions modulo
machine-word primes:
\[
S_A(p)
=
\sum_{j\leq\ell}
\kappa_{j\ell}c_jc_\ell A^{\rm int}_{j\ell}
\pmod p,
\qquad
S_B(p)
=
\sum_{j\leq\ell}
\kappa_{j\ell}c_jc_\ell B^{\rm int}_{j\ell}
\pmod p.
\]
All polynomial powers, coefficient extractions, pair weights and channel
sums are evaluated in $\mathbb F_p[x]$. The machine-word primes have no denominator occurring in the modular formulas with vanishing modulo
$p$, ensuring that every required inverse in $\mathbb F_p$ is defined.

Polynomial arithmetic is performed with exact compiled integer and modular
kernels. Earlier exact-certification implementations used Kronecker
substitution: integer polynomial coefficients were packed into large
integers, multiplied using GMP-backed arbitrary-precision arithmetic, and
unpacked exactly; polynomial powers were then formed by binary
exponentiation. In the later FLINT-backed implementations, polynomial
multiplication, powering and modular reduction are delegated to FLINT
\cite{flint}, accessed through \texttt{python-flint}. FLINT selects its
internal multiplication strategy according to operand size and
representation. These implementation choices affect runtime and memory use
but not the mathematical certificate, which depends only on the exact
resulting elements of $\mathbb F_p[x]$ and the subsequent CRT
reconstruction.

Only the two scalar aggregate integers $S_A,S_B$ are reconstructed. Their
exact quadratic forms are
\[
c^\top Ac=\frac{S_A}{D_A},
\qquad
c^\top Bc=\frac{S_B}{D_B},
\]
with global denominators fixed before the modular sweep.

Let
\[
N_{\mathrm{CRT}}=ks_{\max},\qquad
r_{\max}=(k-1)s_{\max},\qquad
E=k-1+r_{\max}+m_{\max},
\]
and
\[
L_{\mathrm{cm}}
=
\operatorname{lcm}
\{k-1,\ldots,k-1+r_{\max}+m_{\max}\}.
\]
The resulting global denominators are
\[
D_A
=
(k+N_{\mathrm{CRT}})!\,d_h^k d_c^2,
\]
and
\[
D_B
=
(k-2+r_{\max})!\,
L_{\mathrm{cm}}\,
\rho_d^E T_{\rm den}\,
d_h^{k-1}d_c^2.
\]
All factors are therefore known before reconstruction.

The numerator-side coefficient sum can be expressed as a polynomial
correlation. If $T_{j\ell}(x)$ denotes the pair-dependent
integrated-channel polynomial and
\[
V_p(x)
=
\sum_{v=0}^{r_{\max}+m_{\max}}
(k-1+v)^{-1}x^v
\in\mathbb F_p[x],
\]
then the required shifted inner products are obtained from multiplication
with a reversed form of $T_{j\ell}$, implementing the transposition
principle.

Polynomial powers are formed deterministically from exact arithmetic.
Classical sparse-polynomial powering algorithms provide related
powering recurrences \cite{ProbstAlagar1979}, but in the present
application the factorially weighted pair polynomials are typically dense.
The principal computational gain therefore comes from aggregate modular
evaluation, exact coefficient extraction and CRT reconstruction rather than
from sparse polynomial storage itself.

\subsection{Known denominators and deterministic CRT reconstruction}
\label{sup:deterministic_crt}

The modular stage reconstructs signed integer numerators rather than generic rational numbers.
Rigorous bounds
\[
|S_A|\leq\mathcal B_A,\qquad |S_B|\leq\mathcal B_B
\]
are obtained using
\[
\|h^r\|_1\leq\|h\|_1^r
\]
together with explicit bounds for factorial ratios, powers of the rational support ratio, integrated
channel coefficients and rational mixing weights. If
\[
M_P=\prod_{p\in P}p
\]
and
\[
M_P>2\max(\mathcal B_A,\mathcal B_B),
\]
the centered CRT representative in $(-M_P/2,M_P/2]$ is unique. The implementation uses an
additional safety margin, selecting primes until the product exceeds four times the larger bound.
Held-out primes provide an implementation-level check but are not needed for mathematical uniqueness.

The certified quotient is
\[
R_{\rm cert}
=
k(1+\varepsilon)
\frac{S_B/D_B}{S_A/D_A}.
\]
Once the rational channels and mixing vector are fixed, this value is independent of the neural
discovery stage.

\subsection{Development and validation history of the CRT certifier}
\label{sup:crt-history}

The exact certifier was developed through a sequence of increasingly large
stress tests in which numerical agreement between the floating-point proposal
stage and the exact reconstruction stage was treated as a falsification
criterion. Supplementary Table~\ref{tab:crt-history} records the principal
failure modes and the corresponding corrections. These historical iterations
are not independent proof components; they document how the final
implementation acquired the numerical invariants and trust boundaries used
throughout the reported certifications.

%
%

\newcolumntype{V}{>{\raggedright\arraybackslash}p{0.055\textwidth}}
\newcolumntype{C}{>{\raggedright\arraybackslash}p{0.52\textwidth}}
\newcolumntype{M}{>{\raggedright\arraybackslash}p{0.36\textwidth}}

{\small
\begin{longtable}{@{}VCM@{}}
\caption{Development history of the exact multimodular (CRT) certifier for
$M_{k,\varepsilon,1/2}$. Each version records the change and the failure mode
that motivated it. With one exception (the FFT accuracy floor, v4), every
defect was a silent disagreement between the floating-point pass and the exact
pass at a specific weight vector~$c$, invisible at small~$k$ and exposed only
by scale, conditioning, or polynomial degree.}
\label{tab:crt-history}\\
\toprule
\textbf{Ver} & \textbf{Change} & \textbf{Failure mode / motivation}\\
\midrule
\endfirsthead

\multicolumn{3}{@{}l}{\small\itshape Table~\ref{tab:crt-history}, continued}\\
\toprule
\textbf{Ver} & \textbf{Change} & \textbf{Failure mode / motivation}\\
\midrule
\endhead

\midrule
\multicolumn{3}{r@{}}{\small\itshape continued on next page}\\
\endfoot

\bottomrule
\endlastfoot

v1 &
Initial multimodular backend. $\hat{h}^{\,k-1}$ is never formed over
$\mathbb{Z}$; per $62$-bit prime, one \texttt{nmod} powering plus two
weight-table multiplies execute both linear functionals as single coefficient
extractions: $[x^{n_{\max}}](bh^{k}\cdot W_{\mathrm{rev}})$ for~$A$, and for~$B$
the pair-dependent $t$-weights combined as ${\sim}m_{\max}$ scalar--poly ops
against per-prime tables $G^{(m)}$ folding in
$\Lambda=\operatorname{lcm}(k{-}1,\dots,k{-}1{+}r_{\max}{+}m_{\max})$, the
$\rho$-power clearing, and batched harmonic inverses. Reconstruction is
aggregate-only---two integers $S_A,S_B$ with known-denominator scalings---via
deterministic centered CRT against exact a~priori bounds, plus held-out
verification primes. Float pass decoupled from the exact engine
($A=\int_0^1 h^{*k}$, $B=\int_0^{\rho} h^{*(k-1)}Q_jQ_l(1{-}s)$ by renormalized
binary powering). Greedy $c$-weighted pruning on the float pencil
(\texttt{-{}-prune-tol}); only surviving pairs reach the CRT engine. &
Establishes the architecture: avoid the one unavoidable giant integer,
reconstruct only two aggregate scalars, keep the float discovery pass cheap and
separate from exact arithmetic.\\
\addlinespace

v2 &
\texttt{\_ritz} no longer calls the Cholesky-based generalized solver: it
equilibrates by the diagonal, eigendecomposes the metric~$A$, truncates to
eigenvalues above $10^{-12}$ of the top one, and solves the standard symmetric
problem in whitened coordinates, embedding zeros for dead channels. Self-test
adds exact agreement with \texttt{scipy.eigh(B,A)} on a PD case and a finite
$\lambda$ on a singular metric. &
\texttt{scipy.eigh(B,A)} requires $A$ positive definite; at $k=500$ the float
Gram is numerically semidefinite and Cholesky fails
(\texttt{LinAlgError: leading minor not positive definite}). Restricting to
$\operatorname{range}(A)$ is the correct regularization: null-space components
contribute ${\approx}0$ to $c^{\mathsf T}Ac$.\\
\addlinespace

v3 &
B-side of the float DP rescaled by $u=s/\rho$: run the $(k{-}1)$-fold
convolution power on $h_\rho(t)=h(\rho t)$, so
$B=\rho^{\,k-1}\!\int_0^1 h_\rho^{*(k-1)}(u)\,T(\rho u)\,du$, with $\rho^{k-1}$
carried in the log-scale ($\varepsilon=0$ degenerates to the old path). Driver
interlock: float-predicted $R$ checked against discovery $R_{\mathrm{NN}}$
before any exact work. &
At $k=500$ the certified value collapsed to ${\sim}10^{-12}$ (valid but
worthless). $h^{*(k-1)}$ concentrates like $s^{k-2}$ near $s=1$, so the whole
B-region $s\le\rho=\tfrac{12}{13}$ sat $\rho^k\sim e^{-40}$ below the
renormalized peak---beneath float64 accuracy. $B$ was noise; the pruner dropped
$48$ channels ``for free'' and the exact engine certified a garbage frozen~$c$.\\
\addlinespace

v4 &
Truncated convolution \texttt{\_tconv} forced to \emph{direct}
\texttt{np.convolve}, removing the FFT path. Large-$k$ float-DP regression test
against exact $g\equiv1$ closed forms at $k=200,500$. &
FFT convolution carries an \emph{absolute} error ${\sim}10^{-16}$ relative to
the vector peak, while $h^{*j}$ spans $s^{j-1}$ internally; entries sit up to
$2^{-j}$ below peak and are buried by noise by $j\sim55$, then amplified through
each squaring (closed-form check failed by $e^{+1000}$). Direct summation keeps
per-entry \emph{relative} accuracy; the self-convolution integrand peaks at the
interior $t=s/2$, representable to $j\sim511$ (float64 ceiling near $k\sim1000$).\\
\addlinespace

v5 &
Product-exact prime budgeting (\texttt{gen\_primes\_for\_bound}): accumulate the
actual prime product until it exceeds $4\cdot\text{bound}$ rather than assuming
$61$ bits/prime. Miller--Rabin citation corrected (least strong pseudoprime to
bases $2\ldots37$ is $3.19\times10^{23}$; Sorenson--Webster 2017, OEIS
A014233). Self-test cross-checks generated primes against FLINT's proven
\texttt{fmpz.is\_prime}; new strong-cancellation test ($c=(1,-1)$, near-identical
channels). &
Collaborator review round: tighten certification claims (exact prime count,
correct primality bound, FLINT belt-and-braces), cover the small-signed-aggregate
CRT branch, and \emph{measure} which proposed speedups were real
(\texttt{mul\_low} benched slower than the full FFT multiply; native modular dot
kernel deferred).\\
\addlinespace

v6 &
Added \texttt{-{}-float-only} survey mode: projection $+$ float DP $+$ pruning per
degree, stopping before any exact/CRT work. &
Cheap degree scouting at large~$k$: projection $L_\infty$ and float-predicted
$R$ answer ``is this $k$ reachable at this basis degree'' in minutes per degree,
without spending the prime budget.\\
\addlinespace

v7 &
Stable Legendre-basis float evaluation. \texttt{project\_channels\_v3} returns
both monomial integers (exact engine) and dyadic Legendre numerators (float
engine); the float engine evaluates channels and the B-weight antiderivative via
the shifted-Legendre recurrence, with a non-finite guard that raises with a
diagnosis. &
At $d=160$ the float pass produced NaN: channels were evaluated by
\texttt{np.polyval} on monomial coefficients of size ${\sim}8^{d}\sim10^{150}$,
so Horner formed an $O(1)$ value as an alternating sum of $10^{150}$-scale
terms (${\sim}150$ orders of cancellation), poisoning the convolutions to
$\pm\infty$ and hence $\text{off}_A=\max(-\infty)=-\infty$,
$(-\infty)-(-\infty)=\text{NaN}$. The Legendre recurrence
($|\tilde P_n|\le1$ on $[0,1]$) is well-conditioned at any degree.\\
\addlinespace

v8 &
Auto-escalating rank truncation in the whitened Ritz, anchored to
$R_{\mathrm{NN}}$: \texttt{\_ritz} gains a \texttt{rank\_tol} threaded through
\texttt{greedy\_prune}; the driver escalates
$10^{-10}\!\to\!10^{-8}\!\to\!10^{-6}\!\to\!10^{-4}$ until the predicted $R$ is
physical, or uses a fixed \texttt{-{}-rank-tol}. &
At $d=40$, $k=500$ the float $R$ came out $7.4\times10^{21}$. Gram entries were
cancellation-limited to ${\sim}10^{-8}$ but the rank cut was
$10^{-12}\lambda_{\max}$; directions above the cut yet below the accuracy floor
are noise, and whitening by $1/\sqrt{\text{noise}}$ makes
$\lambda=B_{\text{noise}}/A_{\text{noise}}$ explode. The true pencil is bounded
by $M_{k,\varepsilon}$ (single digits).\\
\addlinespace

v9 &
Worker refactor: the B-functional inner sum
$S_R=\sum_m tw_m/(k{-}1{+}R{+}m)$ computed for all $R$ by \emph{one} correlation
multiply $IT=\mathrm{INV}\cdot TW_{\mathrm{rev}}$ against a single global per-prime
inverse polynomial, then combined in a short loop against the base ladder. The
$(m_{\max}{+}1)$-polynomial $G_{\mathrm{rev}}$ tables are eliminated; per-prime
tables become three ladders regardless of degree. Added \texttt{-{}-resume}:
append-only $(p,S_A,S_B)$ checkpoint. &
At $d=160$, $m_{\max}=2(d{+}1)=322$, so the v1 $G_{\mathrm{rev}}$ tables were
$323\times160\text{k}\approx413$ MB per worker and ${\sim}5\times10^{7}$ Python
ops per prime for table construction alone---dominating the sweep across
${\sim}28$k primes. The correlation form is degree-oblivious; resume makes the
multi-hour sweep preemption-tolerant (spot instances).\\
\addlinespace

v10 &
Exact mantissa-preserving $c$ rationalization plus a frozen-$c$ gate.
$c_q=[\,\mathrm{Fraction}(\mathrm{float}(v))\,]$ (full $53$-bit mantissa per
component) replaces uniform-absolute $\mathrm{dyadic}(v,\text{bits})$ rounding;
the gate re-evaluates the float quotient at the rationalized $c$ and aborts if
it moved by $>1\%$. Resume caches gain a SHA-256 payload fingerprint and refuse
stale residues. &
After an $11$-hour $k=500$ sweep the certified value was $8.9\times10^{-5}$.
Uniform absolute $2^{-48}$ rounding annihilated small-but-essential components
of the balanced Ritz $c$, whose entries span dozens of orders
($c_j\sim A_{jj}^{-1/2}$, $A_{jj}\sim(\sup)^{2k}$, so channels differing by
$1.1$ in sup-norm have diagonals differing by $1.1^{1000}\sim10^{41}$). The
float pass evaluated the true $c$ (${\approx}4.19$); the exact engine certified
the mutilated one.\\
\addlinespace

v11 &
Fixed the Gauss--Jacobi overflow in the shared quadrature builder: replace
\texttt{scipy.roots\_jacobi} with an overflow-free Golub--Welsch construction
(Jacobi tridiagonal eigenproblem; weights $=$ squared first-eigenvector
components normalized to sum~$1$, so the diverging $\mu_0$ cancels).
\emph{(Shared-infrastructure fix; gates whether large-$k$ discovery npz files
can be produced at all.)} &
\texttt{roots\_jacobi} computes $\mu_0=2^{a+b+1}B(a{+}1,b{+}1)$ in linear space;
with $p-1\sim2499$ at $k=2500$ the Beta/Gamma factors overflow float64
($\Gamma(171)$ already overflows), returning NaN weights and crashing the
$g\equiv1$ preflight. Golub--Welsch needs only recurrence coefficients (no
$\Gamma$); validated to ${\sim}10^{-14}$ vs \texttt{roots\_jacobi} for
$p\le60$ and finite/correct to $p=6000$.\\
\addlinespace

v12 &
Diagonal preconditioning of the float Gram: \texttt{matrices()} subtracts
$(\log A_{jj}+\log A_{ll})/2$ from each log-entry (so $\hat A$'s diagonal is
exactly~$1$, off-diagonals $O(1)$ by Cauchy--Schwarz), applies the same
congruence to~$B$, and returns $\log A_{\mathrm{diag}}$. Regression: a $195$-nat
diagonal spread must keep all channels alive with float-vs-CRT agreement
$10^{-5}$. &
A single global $\exp(\log A-\text{off})$ offset underflowed entries
$<\text{off}-745$ to exactly~$0$; at $k=1000$ the diagonal spans ${\sim}6900$
nats, so the Ritz solve collapsed to rank~$2$ (observed: $30$ of $32$ channels
dropped ``for free'' at $0.0$ loss). Per-channel log-diagonal preconditioning
keeps every channel alive regardless of~$k$.\\
\addlinespace


v13 &
Frame conversion of the imported discovery~$c$ on load. When the npz carries
$\log A_{\mathrm{diag}}$, the loader maps discovery's exported
(preconditioned-frame) $c$ to the true channel frame via
$c_{\mathrm{true}}=c\cdot\exp(-\tfrac12\log A_{\mathrm{diag}})\cdot
\text{channel\_norms}$ (both factors were then thought necessary: the
$k$-th-power diagonal scale is independent of the $L^2$ scale).
Later corrected: discovery assembles its Gram pencil on the
$L^2$-normalized channels $g_j/\|g_j\|_2$, and the certifier rebuilds exactly
these channels by projecting $\texttt{g\_fine}/\text{channel\_norms}$. The
coefficient on the normalized channels is therefore
$c_{\mathrm{true}}=\exp(-\tfrac12\log A_{\mathrm{diag}})\,\hat c$, with no
channel-norm factor. A conversion to the raw channels would require
$\|g_j\|_2^{-k}$, not $\|g_j\|_2$. The conversion is now shared by all exact
backends (\texttt{certification/frames.py}). Validity was never affected,
since any explicit rational~$c$ yields a valid lower bound, and the re-Ritz
path (v14) never used the channel-norm factor.
&
At $k=53$ discovery reported $3.994$ but the certifier gave $3.594$. The
pencil's frame-invariant top eigenvalue was $3.9948$; discovery exports $c$ in
the frame of its \emph{preconditioned} Gram ($\hat c=D^{1/2}c_{\mathrm{true}}$),
while the certifier evaluates $c$ on the \emph{true} channels. The deflation
grows with the $\log A_{\mathrm{diag}}$ spread ($308$ nats at $k=53$, ${\sim}6900$
at $k=1000$).\\
\addlinespace

v14 &
Frame conversion of the re-Ritz~$c$ (second instance of the same bug). The
driver captures $\log A_{\mathrm{diag}}$ from the float engine and maps
\texttt{\_ritz}'s output back to the true frame
($c_{\mathrm{true}}=\exp(-\tfrac12\log A_{\mathrm{diag}})\,\hat c$) on the active
channels before rationalizing; the frozen-$c$ gate is made frame-consistent. &
After v13, \texttt{-{}-use-nn-c} still gave $3.594$. \texttt{\_ritz} whitens the
preconditioned Gram internally, so its eigenvector lives in the preconditioned
frame---and the driver certified it as if true-frame. Certifying $\hat c$ as
true-frame gives $3.5938$ (the observed value); the mapped $c$ gives $3.9939$,
and the end-to-end run then certified $3.9944$ (gap $+4\times10^{-4}$ to
discovery). Affected \emph{every} re-Ritz-path certification against a
preconditioned-export npz.\\

\end{longtable}
}

\subsection{Representation-preserving transfer}
\label{sup:representation_transfer}

The discovery solver uses diagonal preconditioning. If
\[
A_{\rm disc}=D^{-1/2}A_{\rm true}D^{-1/2}
\]
and $\widehat c$ is the mixing vector in discovery coordinates, then
\[
c_{\rm true}=D^{-1/2}\widehat c.
\]
Any additional channel normalization is likewise restored before certification.

During projection and floating-point evaluation the channels are represented in a shifted-Legendre
basis,
\[
q_j(u)=\sum_{r=0}^d a_{jr}\widetilde P_r(u),
\]
and evaluated through the three-term recurrence. For exact certification the same rational expansion
is transformed symbolically into integer monomial coefficients with a common denominator. If
\[
h(u)=\sum_{r=0}^s h_ru^r,
\]
define its factorially weighted coefficient polynomial by
\[
\widehat h(x)=\sum_{r=0}^s r!h_rx^r.
\]

For enlarged support,
\[
u=\frac{t}{1+\varepsilon},
\qquad
r_{\varepsilon}=\frac{1-\varepsilon}{1+\varepsilon},
\]
and the associated Jacobian and powers of $1+\varepsilon$ and $r_{\varepsilon}$ are retained explicitly. The
mixing vector is embedded into the rationals by preserving each binary floating-point component as
an exact dyadic rational, avoiding loss of small but cancellation-relevant entries. Before modular
evaluation, the frozen rational vector is re-evaluated in the floating-point quadratic forms and required
to reproduce the Ritz candidate within tolerance. This is a transfer check, not part of the proof.

\subsection{Loss-controlled channel reduction}
\label{sup:pruning_details}

Exact certification cost is quadratic in the number of surviving channels. Starting from the full
floating-point Gram pencil, channels are removed by backward elimination. For each candidate removal,
the generalized Rayleigh--Ritz problem is re-solved. If $\lambda_0$ is the original value and
$\lambda_S$ the value on the active subset,
\[
\frac{\lambda_0-\lambda_S}{\lambda_0}\leq\tau_{\rm prune}
\]
is required. At each step the channel whose removal leaves the largest quotient is discarded.
This differs from coefficient thresholding because importance is measured directly in the reoptimized
variational objective. The reduced rational trial function is then certified exactly; pruning changes
only cost and attainable bound, not mathematical validity.

\renewcommand{\figurename}{Supplementary Figure}
\setcounter{figure}{0}
\begin{figure}[t]
    \centering
    \includegraphics[width=0.85\linewidth]{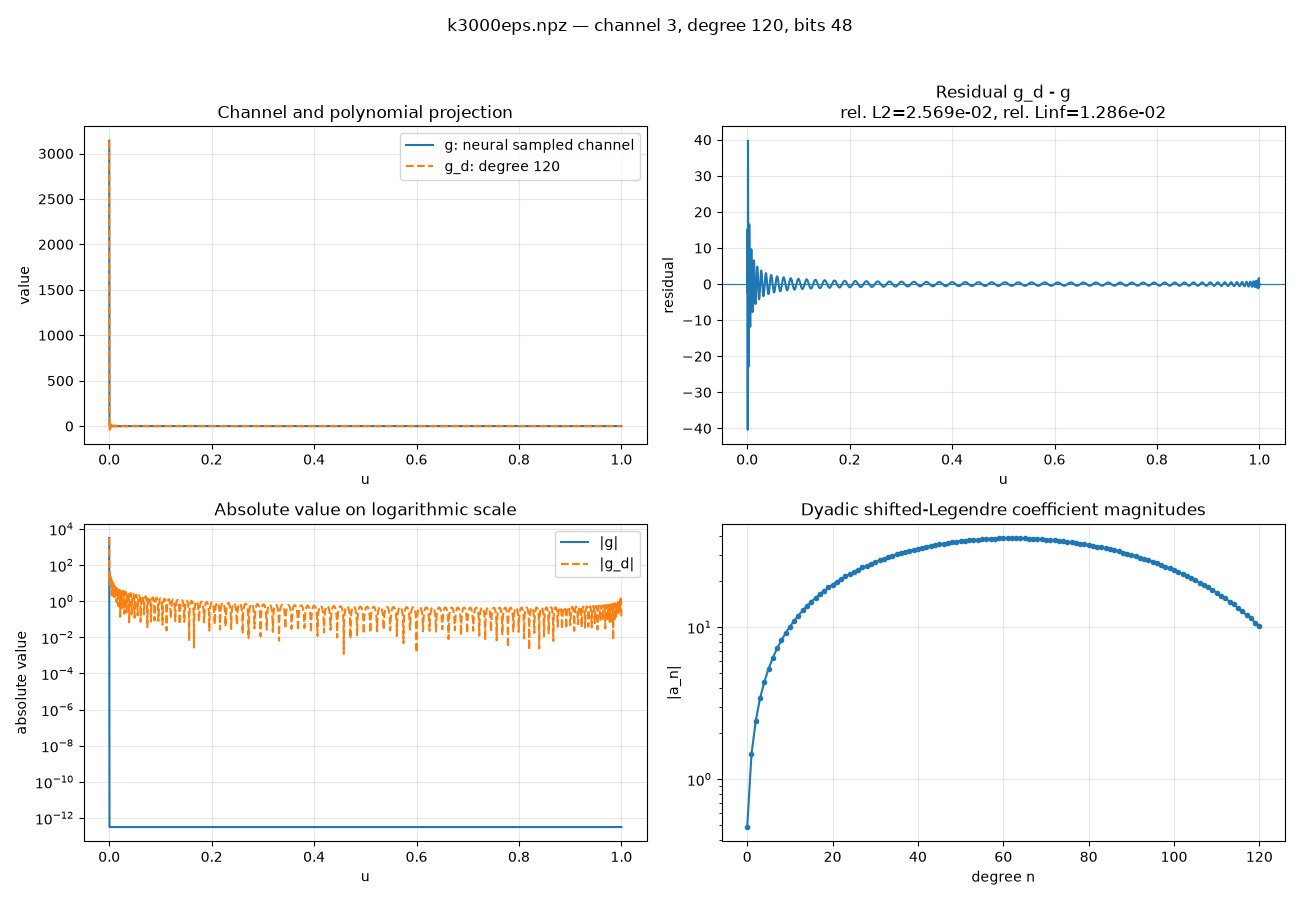}
\caption{\textbf{Diagnostic large-\(k\) failure of the early discovery pipeline.}
For \(k=3600\), the discovery stage reported the spurious value
\(R=8.1375743146\). Nearly all Ritz weight collapsed onto a single channel, shown here together with its degree-120 polynomial projection, residual structure, logarithmic-scale amplitude and dyadic shifted-Legendre coefficient magnitudes. Although the discovery objective was stable under grid refinement and remained below the Cauchy--Schwarz ceiling, the certification failed independent reconstruction and exposed the representation inconsistency; rejecting this candidate.}
    \label{fig:large-k-artifact1}
\end{figure}

\begin{figure}[t]
    \centering
    \includegraphics[width=0.85\linewidth]{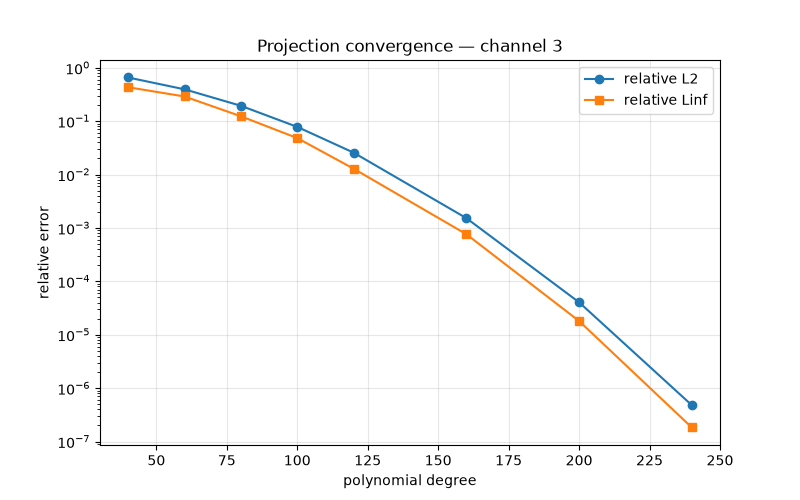}
    \caption{\textbf{Polynomial projection convergence for the single  Ritz weight holding channel of $k=3600$.} Relative \(L^2\) and \(L^\infty\) projection errors decrease with polynomial degree, illustrating that the transfer error can be reduced systematically once a numerically resolved channel has been obtained.}
    \label{fig:large-k-artifact2}
\end{figure}

\subsection{Diagnostic large-\texorpdfstring{$k$}{k} failure and rational distillation}
\label{sup:large-k-failure}

The large-\(k\) experiments exposed a failure mode that was not detected by the original discovery diagnostics. In an early run at \(k=3600\), the multi-channel optimizer collapsed onto a single sharply localized channel and reported a Rayleigh value above \(8\) (See Supplementary Figures \ref{fig:large-k-artifact1} and \ref{fig:large-k-artifact2}). The value was stable under refinement
of the representation grid and remained below the aforementioned ceiling ($R_k \le \frac{k}{k-1}\log k $) so the candidate initially appeared numerically credible. Also, across the validated separable runs, increasing \(k\) was accompanied by a progressive reduction in the number of channels needed to represent the best surviving candidates. This trend was empirical.

Upon certification, it appeared difficult to adequately represent the one-channel discovery solution. For a single exported channel \(q\), the scale-invariant quantity
\[
R_{\mathrm{1D}}
=
k\frac{\left(\int q\right)^2}{\int q^2}
\]
disagreed by orders of magnitude with the stored diagonal ratio
\(kB_{jj}/A_{jj}\). Since multiplication of \(q\) by an arbitrary scalar
cannot change either ratio, this discrepancy cannot be explained by a
normalization or frame convention. For the sharpest channel, the support of
the exported function also gave the Cauchy--Schwarz bound
$R_{\mathrm{1D}} \le k\,\sigma$,
where \(\sigma\) is the effective support measure; the stored discovery value
violated the corresponding scale by more than two orders of magnitude.

The discrepancy was caused by a scale-law inconsistency in the discovery evaluation of (increasingly) narrow channels. The convolution recursion operated in a normalized variable, whereas the exported channel lived in the physical coordinate. In the intermediate-rate regime the error followed approximately
the missing-Jacobian law
\[
R_{\mathrm{stored}}
\approx
\frac{\rho}{k}R_{\mathrm{direct}},
\]
with \(\rho\) the channel rate. At the highest rates a second resolution effect
appeared, and the sharpest exported channel was itself truncated. The optimizer
therefore concentrated its Ritz weight on the channel for which the numerical
artifact was largest. The apparent rank-one structure was consequently not,
by itself, evidence of a genuine extremizer.

The comparison is exact when the channel is genuinely compactly supported:
if $\operatorname{supp}(g)\subseteq[0,\sigma]$ with $k\sigma\le1$, then
$[0,\sigma]^k\subseteq\mathcal R_k$, so $I_k(F)=(\int g^2)^k$ and
$J^{(1)}_k(F)=(\int g)^2(\int g^2)^{k-1}$, giving
$R=k(\int g)^2/\int g^2$ identically. For a channel with only an
\emph{effective} support the identity acquires an uncontrolled tail term, so
the scalar test is applied as an empirical safeguard against gross
inflation rather than as an exact identity.

This observation is important because the artifact was close to the true large-\(k\) scale. The sharp-channel artifact approached a value proportional to the first inner quadrature weight, and for the quadrature order used in the run it crossed \(8\) in the same \(k\)-range in which the true variational constant approaches \(8\). Grid stability and upper bounds were therefore insufficient to identify artifact from valid.

The final discovery solver adds an independent scalar-channel gate. For each
sufficiently localized channel it compares \(kB_{jj}/A_{jj}\) with a direct
one-dimensional quadrature of
\[
k\frac{(\int g_j)^2}{\int g_j^2},
\]
using a quadrature path independent of the convolution assembly. Wide
channels are additionally checked against support-aware Cauchy--Schwarz
bounds. Narrow channels get the hard $\mu \times \text{inflation}$ rejection. Candidates failing these tests are rejected before export.

Conditional on candidates that survive these discovery checks, rational
distillation provides a second and logically independent filter. Richer
multi-cluster rational fits can reduce pointwise approximation error while
failing probability-mass, conditioning or Fourier-stability checks. Because
small one-dimensional transform errors are raised to the \((k-1)\)st power in
the large-\(k\) evaluation, pointwise agreement alone does not preserve the
Maynard functional.

In the stable distilled representation, the surviving family reduces to the
single rational pole
\[
g(t)=\frac{1}{c+(k-1)t}.
\]
Its Rayleigh quotient is then optimized and evaluated independently of the
neural model. Thus the rank-one collapse observed in the failed neural run and
the rank-one structure of the certified rational family have opposite logical
status: the former was an optimization response to a numerical artifact,
whereas the latter survives reconstruction and independent certification.
Neither the failed discovery value nor the neural rank collapse is used as
evidence for the final bound.

\begin{lstlisting}[style=neuralcert,
  caption={Checks used to diagnose unresolved or spuriously inflated neural
  channels before transitioning to the rational large-$k$ route. The all-width
  moment comparison is a conservative empirical discovery safeguard and is
  not part of the exact certificate.},
  label={lst:large-k-gates}]
FUNCTION LOG_INNER_MEAN(log_residual, rate, Y, z_cap)
    IF rate*Y < z_cap
        u <- Y*x_GL
        RETURN LOGSUMEXP_r(LOG(w_GL[r])+log_residual(u[r])-rate*u[r])
    ELSE
        z <- z_cap*x_GL
        // z=rate*u resolves the exponential boundary layer explicitly.
        RETURN LOG(z_cap/(rate*Y))
               +LOGSUMEXP_r(LOG(w_GL[r])+log_residual(z[r]/rate)-z[r])
    END
END

PROCEDURE LARGE_K_CHANNEL_CHECKS(g,A,B)
    rho <- SOFTPLUS(eta)
    (I1[j],I2[j],sigma[j]) <- FINE_GRADED_ONE_DIMENSIONAL_QUADRATURE(g_j)
    R_direct[j] <- k*I1[j]^2/I2[j]
    R_conv[j]   <- k*B[j,j]/A[j,j]

    REQUIRE R_conv[j] <= (1+tau_moment)*R_direct[j] FOR EVERY j
            // conservative empirical discovery filter
    IF k*sigma[j] <= 1 AND R_conv[j] > mu*R_direct[j]
        REJECT unresolved single-channel inflation
    END
    IF k*sigma[j] > 1 AND R_conv[j] > (1+tau_support)*k*sigma[j]
        FLAG the channel as support-inconsistent
    END

    curvature <- MEAN_SQUARED_CURVATURE(LOG_RAW, coordinate=rho*x)
    RETURN curvature                 // Adam search penalty only
END
\end{lstlisting}


\subsection{Neural-to-rational structural distillation}
\label{sup:rational_distillation}

At large $k$, the neural candidate is compressed into the clustered rational family
\begin{equation}
g_{\rm rat}(t)
=
\sum_{a=1}^{r}\sum_{j=1}^{p_{\max}}
u_{a,j}(c_a+nt)^{-j},
\qquad n=k-1.
\end{equation}
For fixed cluster locations $c=(c_1,\ldots,c_r)$, the coefficients enter linearly and are determined
by variable projection,
\begin{equation}
u^\star(c)
=
\arg\min_u
\sum_i\omega_i
\left[
g_{\rm NN}(t_i)
-
\sum_{a,j}u_{a,j}(c_a+nt_i)^{-j}
\right]^2.
\end{equation}
The outer optimizer therefore varies only the nonlinear pole locations. Negligible coefficients are
removed and the highest retained power at each pole determines an inferred multiplicity. Higher
powers are confluent directions, since
\[
\frac{\partial^q}{\partial c^q}(c+nt)^{-1}
=
(-1)^q q!(c+nt)^{-(q+1)}.
\]

Each distilled candidate is subsequently reconstructed in an independent deterministic evaluator.
Fourier-domain cross terms are reduced to
\begin{equation}
F_p(\xi)=\int_0^1e^{i\xi t}(c+nt)^{-p}\,dt ,
\end{equation}
with, for $p\ge2$,
\begin{equation}
F_p(\xi)
=
-
\frac{
e^{i\xi}(c+n)^{1-p}
-
c^{1-p}
-
i\xi F_{p-1}(\xi)}
{n(p-1)}.
\label{eq:Fp-recurrence-supp}
\end{equation}
Because upward recursion can amplify round-off at large $|\xi|$, the Fourier computation is
restricted to a stability window and requires the characteristic function to decay below a prescribed
threshold before truncation. Products involving distinct clusters are reduced by partial fractions;
cross-cluster cancellation is monitored and ill-conditioned configurations are rejected.

Further checks include conservation of probability mass, agreement of the first moment with its
analytic value, control of the negative Fourier-inversion noise floor, rank diagnostics for the Gram
matrix, and the analytic ceiling
\[
R_k\leq\frac{k}{k-1}\log k.
\]
A candidate that violates any check is rejected rather than assigned a numerical objective. A
surviving structure is locally refined in the deterministic evaluator, so the final rational value is a
newly optimized Rayleigh quotient rather than the neural objective or the pointwise fit.

\begin{lstlisting}[style=neuralcert,
  caption={Weighted variable-projection distillation of a neural channel,
  followed by independent Rayleigh optimization and evaluation in the
  resulting clustered rational family.},
  label={lst:rational-refinement}]
PROCEDURE DISTILL_NEURAL_CHANNEL(t, g_NN, omega, initial c, multiplicities mu, k)
    n <- k-1
    PARAMETERIZE c as positive and strictly ordered

    FUNCTION PROFILED_FIT_ERROR(c)
        X[:,a,j] <- (c[a]+n*t)^(-j),  j=1,...,mu[a]
        Xw <- DIAG(SQRT(omega))*X
        yw <- SQRT(omega)*g_NN
        scale each column of Xw to unit norm
        u <- LEAST_SQUARES(Xw,yw)
        undo the column scaling in u
        RETURN SUM_i omega[i]*(g_NN[i]-(X*u)[i])^2
    END

    c <- LBFGS_MINIMIZE(PROFILED_FIT_ERROR,c)
    u <- PROFILE_LINEAR_COEFFICIENTS(c)
    PRUNE terms with negligible scaled contribution
    INFER cluster multiplicities from the retained powers
    RETURN c, u, inferred multiplicities, relative fit error
END

PROCEDURE REFINE_RATIONAL_FAMILY(k, cluster_locations c, multiplicities mu)
    n <- k-1;  L <- 1+epsilon
    basis <- {(c[a]+n*t)^(-j) : j=1,...,mu[a]}
    PARAMETERIZE c as positive and strictly ordered

    FOR EACH channel pair
        IF both channels belong to one cluster
            F[1](theta) <- SINE_COSINE_INTEGRAL_FORMULA(theta,c,n,L)
            FOR p=2,...,p_max
                F[p] <- -(EXP(i*theta*L)*(c+n*L)^(1-p)-c^(1-p)
                          -i*theta*F[p-1])/(n*(p-1))
            END
        ELSE
            REDUCE the product by partial fractions
            REJECT excessive cross-cluster cancellation
        END

        theta_max <- n*(amplification_budget*(p_max-1)!)^(1/(p_max-1))
        REQUIRE the characteristic function is negligible at theta_max
        phi(theta) <- (F_pair(theta)/F_pair(0))^n FOR |theta|<=theta_max
        density <- REAL(FFT(phi))/period
        REQUIRE conservation of mass, first moment, and noise-floor control
        INTEGRATE A[j,l] and B[j,l] against density
    END

    PROFILE the linear coefficients by TOP_RITZ(A,B)
    OPTIMIZE c by Brent (one cluster) or finite-difference LBFGS (multiple)
    RETURN the best candidate re-evaluated on a finer grid
END
\end{lstlisting}

\subsection{Probabilistic reduction for the rational trial}
\label{sup:probabilistic_reduction}

For $k\ge3$, let $n=k-1$, choose an exact rational $c>0$, and define
\begin{equation}
F(t_1,\ldots,t_k)
=
\prod_{i=1}^kg(t_i)\mathbf 1_{\mathcal R_k}(t),
\qquad
g(t)=\frac{1}{c+nt},
\end{equation}
where $\mathcal R_k=\{t_i\ge0:\sum_i t_i\le1\}$. Put $w=g^2$ and
\[
m_0=\int_0^1w(t)\,dt=\frac{1}{c(c+n)}.
\]
Let $S=S_n$ denote the sum of $n$ independent variables with density $w/m_0$ on $[0,1]$.
Then
\begin{equation}
G(\rho)
=
\int_0^\rho g(t)\,dt
=
\frac{\log(1+n\rho/c)}{n},
\qquad
H(\rho)
=
\int_0^\rho w(t)\,dt
=
\frac1n\left(\frac1c-\frac1{c+n\rho}\right).
\end{equation}
Peeling one coordinate from the Maynard integrals gives
\begin{equation}
I_k(F)=m_0^nD,
\qquad
J_k^{(1)}(F)=m_0^nN,
\end{equation}
where
\begin{equation}
N=\mathbb E[G(1-S)^2\mathbf1_{S<1}],
\qquad
D=\mathbb E[H(1-S)\mathbf1_{S<1}],
\end{equation}
and consequently $M_k\ge kN/D$.

Let
\[
\widehat w(\theta)=\int_0^1e^{i\theta t}w(t)\,dt,
\qquad
\varphi(\theta)=\left(\frac{\widehat w(\theta)}{m_0}\right)^n,
\]
and, for $X\in\{G^2,H\}$,
\[
\Psi_X(\theta)=\int_0^1X(\rho)e^{i\theta\rho}\,d\rho.
\]
The required expectation has the Fourier representation
\begin{equation}
I_X
=
\frac1{2\pi}
\int_{\mathbb R}
\varphi(\theta)e^{-i\theta}\Psi_X(\theta)\,d\theta.
\end{equation}
The integrability needed for inversion follows from integration-by-parts decay of both factors.

\subsection{Aliasing bound}
\label{sup:aliasing}

Fix $P>2$ and $\Delta\theta=2\pi/P$. For the full trapezoidal sum $T_X$,
Poisson summation and compact support imply
\begin{equation}
0\le T_X-I_X
\le X_{\max}\mathbb P(S>P-1).
\end{equation}
The one-sided sign follows from nonnegativity of the aliased density. The residual probability is
bounded by certified moments. With $\widetilde M_0=1$ and
\begin{equation}
\widetilde M_r
=
\sum_{j=1}^{r}
\binom{r-1}{j-1}
n\,\mathbb E[t^j]\,
\widetilde M_{r-j},
\end{equation}
the supplied compound-Poisson domination gives
\begin{equation}
\mathbb P(S>y)
\le
\min_{1\le r\le r_{\max}}
\frac{\widetilde M_r}{y^r}.
\end{equation}
$E[S^r]$ expands over set partitions with falling-factorial coefficients $n^{\underline{b}} \leq n^{b}$, and the compound-Poisson moments have exactly $n^{b}$ with all terms nonnegative.
 
\subsection{Fourier truncation bound}
\label{sup:fourier_truncation}

Integration by parts yields, for $\theta\neq0$,
\begin{equation}
\left|\frac{\widehat w(\theta)}{m_0}\right|
\le
\frac{A}{|\theta|},
\qquad
A=\frac{2(c+n)}{c},
\end{equation}
and hence $|\varphi(\theta)|\le(A/|\theta|)^n$. The omitted positive
frequencies are divided into a finite band and an analytic tail. In the finite band, indices are
partitioned into blocks $[a_j,b_j]$ and Arb wide-ball evaluation supplies certified suprema
$\sigma_j$ for $|\widehat w/m_0|$ on each block. Thus
\begin{equation}
|T_X^{\rm band}|
\le
\frac{\Delta\theta}{\pi}X_{\rm int}
\sum_j(b_j-a_j+1)\sigma_j^n.
\end{equation}
For
\[
M_0=
\max\!\left(
\left\lfloor\Theta_{\rm ball}/\Delta\theta\right\rfloor,
n_{\rm near}
\right),
\qquad
\Theta_{\rm ball}=2A,
\]
the remaining tail satisfies
\begin{equation}
|T_X^{\rm tail}|
\le
\frac{\Delta\theta}{\pi}X_{\rm int}
\left(\frac{A}{\Delta\theta}\right)^n
\frac{M_0^{1-n}}{n-1}.
\end{equation}
Taking $M_0\ge n_{\rm near}$ ensures that the analytic tail begins where the computed sum ends.

\subsection{Arb evaluation and directed rounding}
\label{sup:arb_certification}

Every scalar entering the computation is produced by Arb ball arithmetic. The transform
$\widehat w/m_0$ is evaluated from a closed form involving sine and cosine integrals away from the
origin and by direct rigorous integration in the region where the special-function representation is
ill-conditioned. The transforms $\Psi_X$ and the moments entering the aliasing bound are likewise
evaluated by rigorous integration. The pole and logarithmic branch point lie at $t=-c/n<0$, outside
the real integration interval; subdivision is used whenever an evaluation enclosure approaches a
singularity or branch cut.

All error terms remain in ball arithmetic until the final quotient is assembled. Floating-point
exports are explicitly nudged until Arb comparisons prove outward rounding. If
$N_{\rm ball}$ and $D_{\rm ball}$ denote the truncated-sum enclosures and $E_N,E_D$ the certified
totals of the aliasing and truncation errors, then
\begin{equation}
R_{\rm low}
=
\operatorname{lb}\!\left(
\frac{k\,\operatorname{lb}(N_{\rm ball}-E_N)}
     {\operatorname{ub}(D_{\rm ball}+E_D)}
\right).
\end{equation}
The stored claim is additionally rounded downward so that independent recomputation with different
valid Arb subdivisions or enclosure radii cannot strengthen the stated result.

\subsection{Certified rank-one ladder and uniform finite-range bound}
\label{sup:rank1_ladder}

For \(n=k-1\) and \(c>0\), let
\[
g_c(t)=\frac{1}{c+nt},
\]
and denote by \(R_k^{(1)}\) the value attained by this rank-one rational
family. Every such value is a lower bound for \(M_k\).

The fixed headroom value used below was selected from the independently
optimized rank-one sweep rather than imposed a priori. Writing
\[
\eta(c,k)
=
\frac{1-\mathbb E[S]}{\operatorname{sd}(S)},
\]
the median optimized value of \(\eta\) moves steadily toward approximately
\(-0.475\) as \(k\) increases:
\[
\begin{array}{c|c}
\text{range of }k & \operatorname{median}\eta \\ \hline
k<10^3                & -0.4256 \\
10^3\le k<10^5        & -0.4593 \\
10^5\le k<10^7        & -0.4715 \\
10^7\le k<10^9        & -0.4747
\end{array}
\]
This convergence motivated the fixed choice
\[
\eta^\ast=-0.4746,
\]
which agrees most closely with the independently optimized scale in the
large-\(k\) regime containing the principal threshold results.

Fixing \(\eta\) incurs very little loss for the bulk of the optimized sweep.
Interpolating the fixed-headroom ladder to the \(k\)-values of the direct
rank-one sweep, the optimized certified value exceeds the corresponding
fixed-\(\eta\) value by a median of only
\[
4.2\times10^{-5},
\]
with maximum positive difference
\[
4.2\times10^{-3}.
\]
The comparison is not uniformly positive: the minimum difference is
\[
-4.2\times10^{-2},
\]
so at least one nominally optimized sweep row is outperformed by the
fixed-headroom construction at the same \(k\). This is consistent with the
isolated high-\(k\) outlier near \(\eta=-0.32\), and is one reason the direct
sweep is interpreted as a record of the numerical optimization landscape
rather than as a certified sequence of exact one-parameter optima.

To obtain a uniform finite-range statement rather than isolated certified
dimensions, we constructed an approximately log-uniform grid beginning at
\(k=100\), with target spacing
\[
\Delta\log k=0.02.
\]
At each grid point, \(c\) was determined by solving
\[
\eta(c,k)=\eta^\ast=-0.4746
\]
using the closed-form moments of the rank-one probability measure. The
resulting explicit trial function was then certified independently using the
ball-arithmetic procedure of Section~\ref{sup:arb_certification}. All retained
rungs passed certification.

For consecutive certified grid points \(k_i<k_{i+1}\), monotonicity of
\(M_k\) implies that every integer \(k\in[k_i,k_{i+1}]\) satisfies
\[
M_k
\ge
M_{k_i}
\ge
R_{\rm cert}^{(1)}(k_i)
\ge
\log k-
\left[
\log k_{i+1}
-
R_{\rm cert}^{(1)}(k_i)
\right].
\]
Thus the relevant covering deficit on the interval
\([k_i,k_{i+1}]\) is
\[
\delta_i
=
\log k_{i+1}
-
R_{\rm cert}^{(1)}(k_i).
\]

The largest interval deficit over the certified ladder is
\[
\max_i\delta_i
=
0.306754936,
\]
attained on the penultimate full-size interval
\[
k_i=1{,}834{,}543{,}653,
\qquad
k_{i+1}=1{,}880{,}000{,}000.
\]
The maximum does not occur on the final interval because the last grid step
was deliberately shortened to terminate at the round endpoint
\[
k=1.88\times10^9.
\]
Its logarithmic width is only
\[
\Delta\log k=0.004476,
\]
substantially smaller than the nominal spacing \(0.02\). The final truncated
interval therefore contributes a smaller covering deficit than the preceding
full-size rung.

The distinction between pointwise and interval-covering deficits is also
important. On the certified grid points themselves, the largest observed
pointwise deficit is
\[
\max_i
\left[
\log k_i-R_{\rm cert}^{(1)}(k_i)
\right]
=
0.286812601.
\]
Hence the certified rungs satisfy the sharper pointwise statement
\[
R_{\rm cert}^{(1)}(k_i)
\ge
\log k_i-0.287.
\]
The larger constant \(0.307\) is required only to cover every intermediate
integer \(k\) between successive certified rungs.

\paragraph{Consistency of the threshold certificates with the ladder.}

The four threshold certificates reported in the Results lie inside the
range covered by this ladder, and were produced independently of it: the
pole parameter $c$ was optimised at each threshold rather than selected
from the fixed-headroom relation, and each was certified as an isolated
trial function. They therefore provide an internal check on the
finite-range behaviour rather than additional coverage.

Supplementary Table~\ref{tab:threshold-deficit} records the certified deficit
$\log k-R^{(1)}_{\mathrm{cert}}$ at each threshold, together with the
value obtained from the empirical finite-size relation
\begin{equation}
\log k-R^{(1)}_k \;\approx\; \mathcal C_*-\frac{1.047}{\log k},
\qquad \mathcal C_*\simeq 0.3343 ,
\label{eq:finite-size-fit}
\end{equation}
which was introduced in Section~\ref{sup:rank1_scale} from the optimised
rank-one sweep. Agreement is within $3\times10^{-3}$ at every point,
across more than five orders of magnitude in $k$, and the final row shows
that the upper end of the certified ladder follows the same relation.

\begin{table}[htbp]
\centering
\caption{Certified deficits at the four thresholds and at the upper end
of the fixed-headroom ladder, compared with the finite-size relation Eq. 
\eqref{eq:finite-size-fit}. Thresholds were optimised and certified
independently of the ladder; the final row is the pointwise deficit at
the last ladder rung. The relation was not
fitted to these values.}
\label{tab:threshold-deficit}
\small
\begin{tabular}{@{}rrlrrr@{}}
\toprule
$k$ & $\log k$ & $R^{(1)}_{\mathrm{cert}}$
  & $\log k-R^{(1)}_{\mathrm{cert}}$
  & $\mathcal C_*-1.047/\log k$ & difference \\
\midrule
$3\,655$        & $8.203851$  & $8.0000648690$  & $0.203787$ & $0.206677$ & $-0.002891$ \\
$208\,910$      & $12.249659$ & $12.0000019900$ & $0.249657$ & $0.248828$ & $+0.000829$ \\
$11\,655\,069$  & $16.271252$ & $16.0000003924$ & $0.271251$ & $0.269953$ & $+0.001298$ \\
$644\,589\,002$ & $20.284123$ & $20.0000009143$ & $0.284123$ & $0.282683$ & $+0.001439$ \\
\midrule
$1\,880\,000\,000$ & $21.354538$ & \textit{(ladder rung)} & $0.286813$ & $0.285271$ & $+0.001542$ \\
\bottomrule
\end{tabular}
\end{table}

Two remarks on the status of this comparison. First, it is a consistency
check and not a proof component: the constant $\mathcal C_*$ is a
numerical evaluation of the limiting formula, and the coefficient
$1.047$ is an empirical finite-size fit for which
Section~\ref{sup:stable_law} establishes no quantitative rate. Neither
enters any certified inequality. Second, the agreement is nevertheless
informative precisely because nothing was tuned to produce it: the
thresholds were optimised for a different purpose (crossing the integer
values $8$, $12$, $16$ and $20$), the ladder was generated from a fixed
headroom target, and the relation \eqref{eq:finite-size-fit} was
extracted from a third sweep. That three independently constructed
families of certified rank-one trials trace the same $\mathcal C_*-
b/\log k$ curve supports the interpretation of the finite-$k$ behaviour
as the pre-asymptotic regime of the stable-law limit rather than as a
numerical coincidence. The single deviation of appreciable size, at
$k=3655$, is in the conservative direction and is consistent with the
relation being fitted at larger $k$.

The interval constant is close to saturated at the upper end of the computed
range. The final four full-size interval deficits are
\[
0.306613,\qquad
0.306660,\qquad
0.306708,\qquad
0.306755,
\]
increasing by approximately
\[
4.7\times10^{-5}
\]
per rung. Continuing the same fixed-headroom ladder for only several further
full-size steps would therefore push the covering deficit above \(0.307\).
Accordingly, rounding the observed maximum upward gives the finite-range
statement
\[
\boxed{
M_k\ge\log k-0.307,
\qquad
100\le k\le 1{,}880{,}000{,}000.
}
\]
No extension of the constant \(0.307\) beyond this explicitly certified range
is claimed.

\subsection{Rank-one scale selection and the headroom invariant}
\label{sup:rank1_scale}

Under the probability measure proportional to \(g_c(t)^2\,dt\), let
\(S=\sum_{i<k}t_i\) and write
\[
c^{-1}=\log n+a,\qquad n=k-1.
\]
Uniformly for \(a=O(\sqrt{\log n})\), the closed-form moments of this measure
imply
\begin{equation}
1-\mathbb E[S]
=
c\left(a+1-\log c^{-1}\right)+O\!\left(\frac{c^{2}\log n}{n}\right),
\qquad
\operatorname{sd}(S)=\sqrt c\left(1+O\!\left(\frac{\log n}{n}\right)\right).
\end{equation}
For fixed \(a\), \(\log c^{-1}=\log\log n+O(1/\log n)\).
Consequently,
\begin{equation}
\eta(c,k)
=
\frac{1-\mathbb E[S]}{\operatorname{sd}(S)}
=
\sqrt c\left(a+1-\log c^{-1}\right)+O\!\left(\frac{\log n}{n}\right).
\label{eq:eta_asymptotic}
\end{equation}
Fixing \(\eta=-\gamma_{\mathrm h}\), multiplying Eq.~\eqref{eq:eta_asymptotic} by
\(c^{-1/2}\) and substituting \(a=c^{-1}-\log n\) gives the implicit scale
equation
\begin{equation}
c^{-1}+\gamma_{\mathrm h}\,c^{-1/2}-\log c^{-1}
=
\log n-1+O\!\left(\frac{(\log n)^{3/2}}{n}\right).
\label{eq:fixed_eta_exact}
\end{equation}
The left-hand side is strictly increasing in \(c^{-1}\) for \(c^{-1}\ge 1\), so
Eq.~\eqref{eq:fixed_eta_exact} determines the fixed-headroom scale to the stated
accuracy. Writing \(x=c^{-1}\), it gives \(x=\log n+O(\sqrt{\log n})\), hence
\(\log x=\log\log n+O\big((\log n)^{-1/2}\big)\) and
\(\sqrt x=\sqrt{\log n}-\gamma_{\mathrm h}/2+O\big(\log\log n/\sqrt{\log n}\big)\).
Substituting these expansions yields
\begin{equation}
c^{-1}
=
\log n-\gamma_{\mathrm h}\sqrt{\log n}+\log\log n-1+\frac{\gamma_{\mathrm h}^{2}}{2}
+O\!\left(\frac{\log\log n}{\sqrt{\log n}}\right).
\label{eq:fixed_eta_scaling}
\end{equation}
The constant \(\gamma_{\mathrm h}^2/2\) arises from the shift of \(c^{-1/2}\) relative to
\(\sqrt{\log n}\). The remainder decays only like
\(\log\log n/\sqrt{\log n}\) and is not small over the certified range
(\(\log\log n/\sqrt{\log n}\approx0.66\) at \(\log n\approx21.35\)). The ladder
was therefore computed from the exact closed-form moments, which agree with
Eq.~\eqref{eq:fixed_eta_exact} to the stated accuracy, and not from the
truncated expansion~\eqref{eq:fixed_eta_scaling}.
Thus the empirically selected coefficient
\(\gamma_{\mathrm h}\simeq0.4746\) is not an independent asymptotic fit:
it is the same quantity as \(-\eta^\ast\) and follows directly from the
fixed-headroom parameterization.

Equations~\eqref{eq:fixed_eta_exact} and~\eqref{eq:fixed_eta_scaling} describe
the fixed-headroom parameterization of the certified ladder; they are not a
sharper asymptotic optimizer law. Indeed, fixed \(\eta=-\gamma_{\mathrm h}\) corresponds to
the shift
\[
a(L)=c^{-1}-L
=
\log L-1-\gamma_{\mathrm h}\sqrt L+\frac{\gamma_{\mathrm h}^{2}}{2}
+O\!\left(\frac{\log L}{\sqrt L}\right),
\qquad L=\log n,
\]
so \(a(L)\to-\infty\) and fixed \(\eta\) eventually leaves every fixed-shift
family. The leading-order part of \(a(L)\) is not monotone: its derivative
\(1/L-\gamma_{\mathrm h}/(2\sqrt L)\) vanishes at
\[
\sqrt L=\frac{2}{\gamma_{\mathrm h}},
\]
where it attains its maximum, and it decreases thereafter. Over the certified
range, however, the induced shift obtained from
Eq.~\eqref{eq:fixed_eta_exact} remains close to the small negative values where
the numerical limiting defect \(\mathcal C(a)\) is minimized. This explains the
effectiveness of the headroom surrogate over the finite range studied here: it
selects rational trials lying close to the same near-optimal fixed-shift
regime, even though the two parameterizations diverge asymptotically.

At the upper end of the certified ladder, the measured pointwise deficit is
approximately
\[
\log k-R(F_{k,c_k})\simeq0.2868,
\]
whereas numerical evaluation of the limiting fixed-shift formula gives
\[
\mathcal C_*\simeq0.3343.
\]
The difference is approximately
\[
0.3343-0.2868=0.0475.
\]
Over the finite range considered in the numerical sweep, this is close to
the magnitude of the empirically fitted correction
\[
\frac{1.047}{\log k}\simeq0.049
\qquad
(\log k\simeq21.35).
\]
Equivalently, the observed deficits are numerically consistent with the
finite-range relation
\[
\log k-R(F_{k,c_k})
\approx
\mathcal C_*-\frac{1.047}{\log k}.
\]
This relation is an empirical consistency check only. In particular, the
coefficient $1.047$ has not been derived analytically and is not identified
with the coefficient of a proved second-order expansion.

The distinction is important because the stable-law theorem concerns the
fixed-shift family
\[
c^{-1}=\log n+a,
\qquad a\in\mathbb R\ \text{fixed},
\]
whereas the fixed-headroom ladder corresponds to a shift $a=a(\log n)$
that eventually tends to $-\infty$. Therefore the finite-size expansion
for fixed $a$ cannot be applied directly to the optimized headroom ladder.
Even within the fixed-shift family, the proof establishes only
\[
\log k-R(F_{k,c_k})
=
\mathcal C(a)+o(1),
\]
because no sufficiently sharp convergence rate has been proved for
\[
\frac{L_n}{P_n}-\frac{L_a}{P_a}.
\]
Consequently, neither the numerical agreement above nor the fixed-headroom
parameterization establishes a rigorous $1/\log k$ correction. The
coefficient $1.047$ should therefore be reported solely as a numerical
finite-size observation.

\subsection{Spectrally positive $1$-stable limit}
\label{sup:stable_law}

Section~\ref{sec:rational-asymptotics} gives the complete asymptotic
analysis of the single-channel rational family. For
\[
n=k-1,
\qquad
c_k^{-1}=\log n+a,
\]
with fixed $a\in\mathbb R$, the rescaled boundary variable converges in
distribution to
\[
X_a=(a+1)-\Lambda,
\]
where $\Lambda$ is the spectrally positive $1$-stable random variable
defined by
\[
\mathbb E e^{is\Lambda}
=
\exp\left\{
\int_0^\infty
\left(
e^{isx}-1-isx\mathbf1_{\{x\le1\}}
\right)\frac{dx}{x^2}
\right\}.
\]
The proof supplements weak convergence with uniform density, small-ball,
tail, and logarithmic moment estimates, which are needed because the
Rayleigh quotient contains logarithmic observables at the simplex
boundary.

Theorem~\ref{thm:rational-asymptotic} consequently gives
\[
R(F_{k,c_k})
=
\log k-\mathcal C(a)+o(1),
\]
where
\[
\mathcal C(a)
=
a-2\,\mathbb E[\log X_a\mid X_a>0].
\]
Thus, with
\[
\mathcal C_*:=\inf_{a\in\mathbb R}\mathcal C(a),
\]
Corollary~\ref{cor:Mk-asym} yields
\[
M_k\ge \log k-\mathcal C_*-o(1).
\]

Here $\mathcal C_*$ is the optimal asymptotic defect only within the
fixed-shift family
\[
c_k^{-1}=\log(k-1)+a,
\qquad a\in\mathbb R.
\]
The proof does not establish optimality over arbitrary sequences $c_k$,
nor does it require the infimum to be attained. Numerical evaluation
suggests
\[
\mathcal C_*\approx0.3343,
\]
but this decimal value should be regarded as a numerical estimate unless
the global infimum is separately certified. A rigorous enclosure of
$\mathcal C(a_0)$ at one fixed shift $a_0$ already suffices for an explicit
asymptotic lower bound.

Finally, although the finite-$n$ identity in
Section~\ref{sec:rational-asymptotics} isolates a possible
$1/\log k$ correction, the required convergence rate for
\[
\frac{L_n}{P_n}-\frac{L_a}{P_a}
\]
has not been proved. Accordingly, no second-order asymptotic coefficient
is claimed here.

\paragraph{Relation to the optimized rank-one sweep.}
The optimized rank-one sweep explores a different asymptotic parameter
regime, including an empirically observed $\sqrt{\log n}$ correction.
Those numerical observations are not used in the fixed-shift theorem
above.

\begin{figure}[t]
    \centering
\includegraphics[width=\textwidth]{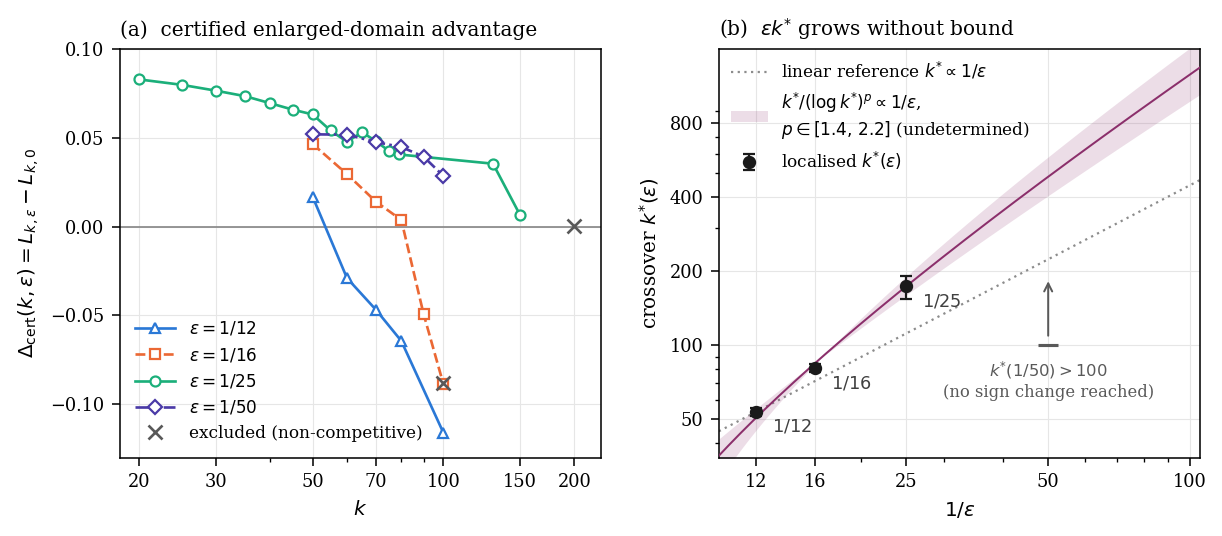}
\caption{\textbf{Enlarged-support crossover.}
  \textbf{(a)} The certified difference
  $\Delta_{\mathrm{cert}}(k,\varepsilon)=L_{k,\varepsilon}-L_{k,0}$, a difference of certified lower bounds rather than of the exact constants, against $k$ for four values of $\varepsilon$. It is positive at small $k$, decreases with $k$, and changes sign at a finite scale for $\varepsilon=1/12$ and $1/16$; no sign change is reached for $\varepsilon=1/50$ within the computed range. Grey crosses mark certified but non-competitive runs excluded from crossover inference (Section~\ref{sup:epsilon_crossover}).
  \textbf{(b)} The localised crossovers $k^{*}(\varepsilon)$ against
  $1/\varepsilon$ on logarithmic axes. The $\varepsilon=1/12$ and $1/16$ points are linear interpolations between the certified brackets; the $\varepsilon=1/25$ marker is the midpoint of the extrapolative range $[155,192]$, with the bar spanning that range. The points lie above the linear reference $k^{*}\propto1/\varepsilon$ (dotted) and pull away from it, giving $\varepsilon k^{*}=4.48,\,5.04,\,6.94$ and motivating
  $\varepsilon k^{*}(\varepsilon)\to\infty$. The band shows the family
  $k^{*}/(\log k^{*})^{p}\propto1/\varepsilon$ over the range $p\in[1.4,2.2]$ admitted by the data, with the constant refitted at each $p$; the exponent is not determined by these points.}
    \label{fig:epsilon-crossover}
\end{figure}

\subsection{Enlarged-support crossover analysis}
\label{sup:epsilon_crossover}

Across the ladder \(50\leq k\leq220\), the advantage conferred by \(\varepsilon\) fluctuates at small \(k\), then decreases and changes sign at a finite crossover \(k_\ast(\varepsilon)\), which we locate at \(k_\ast(1/12)\approx53.7\), \(k_\ast(1/16)\approx80.7\) and \(k_\ast(1/25)\in[155,192]\). We conjecture that such a crossover exists for every \(\varepsilon>0\) and that \(k_\ast(\varepsilon)\) grows superlinearly in \(1/\varepsilon\), so that \(\varepsilon\,k_\ast(\varepsilon)\to\infty\) (Supplementary Table~\ref{tab:epsilon_sweep}).

Write $L_{k,\varepsilon}$ for the certified lower bound produced for the enlarged-domain
construction and $L_{k,0}$ for the corresponding vanilla lower bound, and define
\[
\Delta_{\rm cert}(k,\varepsilon)
=
L_{k,\varepsilon}-L_{k,0}.
\]
For the reliable computed configurations, this difference is positive at smaller $k$, decreases
with $k$, and changes sign at a finite scale. The certified pairs used to localize the observed
crossovers include
\[
\begin{array}{c|cc}
\varepsilon & k_{\rm left} & k_{\rm right}\\
\hline
1/12 & 50 & 60\\
1/16 & 80 & 90
\end{array}
\]
with certified differences $+0.016961$ and $-0.028948$ for $\varepsilon=1/12$, and
$+0.003771$ and $-0.049093$ for $\varepsilon=1/16$. Linear interpolation gives
\[
k^*(1/12)\approx53.7,
\qquad
k^*(1/16)\approx80.7.
\]
For $\varepsilon=1/25$, the last reliable positive certified points are
$k=130$ and $150$, with differences $+0.035443$ and $+0.006581$; the supplied analysis reports
the broader extrapolative range
\[
k^*(1/25)\in[155,192]
\]
rather than treating a degraded $k=200$ pair as a direct bracket. No sign change was reached for
$\varepsilon=1/50$ within the reliable computed range.

All certified runs are in Supplementary Table \ref{tab:epsilon_sweep}. Several runs are excluded from crossover inference because they satisfy numerical certification of
the particular trial function but are demonstrably noncompetitive discovery outcomes. These include
a reproducible local failure at $(k,\varepsilon)=(100,1/25)$, degraded paired runs at $k=200$,
and incomplete or nonmonotone optimization at $k=140$ and $220$. The exclusion criterion is
therefore not failure of the certificate: it is failure of the discovery stage to provide a
competitive construction. Monotonicity in $k$, the discovery-to-certification gap and retained
effective rank provide inexpensive diagnostics for these cases.

The observed crossover locations give increasing values of
$\varepsilon k^*(\varepsilon)$ as $\varepsilon$ decreases (see Supplementary Figure \ref{fig:epsilon-crossover}), motivating the conjecture that for each
fixed $\varepsilon>0$ a finite crossover exists and
\[
\varepsilon k^*(\varepsilon)\to\infty
\qquad(\varepsilon\to0).
\]
The available points are compatible with a scaling variable
\[
x=\frac{\varepsilon k}{(\log k)^p},
\]
but do not determine $p$: the supplied fits admit approximately $p\in[1.4,2.2]$, and the natural
$p=1$ and $p=2$ alternatives cannot be distinguished by the current data. We therefore do not
assign a specific exponent.

\paragraph{Scope of the crossover statement.}
A crucial distinction is that $\Delta_{\rm cert}$ is a difference of certified lower bounds.
Therefore $\Delta_{\rm cert}(k,\varepsilon)<0$ does not prove
$M_{k,\varepsilon}<M_k$; it proves only that the certified vanilla construction found by the
pipeline outperforms the certified enlarged-domain construction found at that $k$. Extending the
observed crossover to the exact variational constants is consequently conjectural. Likewise,
interpolated or extrapolated values of $k^*(\varepsilon)$ are descriptive summaries of the computed
ladder, not independently certified roots.

\subsection{Supplementary data: optimized rank-one sweep}
\label{sup:rankone_R_sweep}

Supplementary Data \texttt{maynard\_R\_sweep\_eta.csv} contains the direct
rank-one optimization sweep for
\[
g_c(t)=\frac{1}{c+(k-1)t},
\]
with one optimized pole parameter \(c\) for each reported \(k\).  The file
records the discovery value \(R\), the independently certified lower bound,
the discovery--certification gap, the Cauchy--Schwarz ceiling, the optimized
scale \(c\), the standardized simplex headroom
\[
\eta(c,k)=\frac{1-\mathbb E[S]}{\operatorname{sd}(S)},
\]
and numerical diagnostics used during optimization.  All retained solutions
have rank one in both numerator and denominator Gram representations.

The sweep contains \(729\) attempted configurations over
\(100\le k\le2\times10^9\).  Of these, \(727\) pass certification and two are
explicitly marked \texttt{FAILED}; failed rows are retained in the data for
provenance but are excluded from mathematical claims.  The certified
discovery gap is small throughout the successful runs, with median
\(9.9\times10^{-4}\) and maximum \(2.3\times10^{-3}\).  The optimized
headroom is not fixed: its variation across the sweep is used diagnostically
to identify the approximately constant headroom regime that motivates the
geometric ladder below.  This file is therefore a record of the numerical
optimization landscape, not itself the basis for the uniform finite-range
theorem.

\subsection{Supplementary data: fixed-headroom geometric ladder}
\label{sup:geometric_eta_sweep}

Supplementary Data
\texttt{maynard\_geometric\_eta\_sweep.csv} contains the certified
fixed-headroom ladder used to convert isolated rank-one trial functions into
a uniform bound over a continuous range of integer \(k\).  At each rung,
\(c\) is chosen from
\[
\eta(c,k)=-0.4746,
\]
and the resulting explicit rank-one trial is certified independently.  The
file contains \(828\) certified rungs from \(k=100\) to
\(k=1.88\times10^9\); every row has certification status \texttt{PASS}.
The target headroom is reproduced to numerical precision throughout the
sweep.

The ladder is approximately uniform in \(\log k\), with asymptotic spacing
\(\Delta\log k\simeq0.02\); the small deviations at the lower end and final
endpoint arise from integer rounding.  For consecutive certified rungs
\(k_i<k_{i+1}\), monotonicity of \(M_k\) gives
\[
M_k\ge R_{\mathrm{cert}}(k_i)
\qquad
(k_i\le k\le k_{i+1}),
\]
and therefore
\[
M_k\ge
\log k-
\left[
\log k_{i+1}-R_{\mathrm{cert}}(k_i)
\right].
\]
The largest interval deficit in the supplied data is
\[
0.3067549357,
\]
attained between
\[
k_i=1\,834\,543\,653,
\qquad
k_{i+1}=1\,871\,603\,894.
\]
Rounding this deficit upward yields the reported uniform statement
\[
M_k\ge \log k-0.307
\]
throughout the certified range.  The file additionally records the
pointwise deficit \(\log k-R_{\mathrm{cert}}\), the discovery--certification
gap, the Cauchy--Schwarz ceiling and the scale diagnostics used to relate the
finite ladder to the fixed-shift asymptotic analysis.

\subsection{Conversion to admissible-tuple bounds}
\label{sup:tuple_conversion}
A certified inequality $M_k>4m$ supplies the variational input to the
Maynard--Tao bounded-gap theorem. To obtain a numerical bound on $H_m$,
one additionally requires an admissible $k$-tuple and its diameter. We
therefore distinguish the variational certificate from the admissible-tuple
construction and write $D_k$ for the diameter of the explicit admissible
$k$-tuple used in the conversion.

For the first threshold, $M_{3655}>8$, a stronger admissible-tuple diameter
is already recorded in the Engelsma--Sutherland database
(\url{https://math.mit.edu/~primegaps/}). The entry for $k=3655$, submitted
by Sutherland on 27 June 2013, gives
\[
D_{3655}\le33118.
\]
Our independently constructed hybrid sieve gives the slightly weaker control
\[
D_{3655}\le33270,
\]
and is retained as a reproducibility check.

For $k=3655$ and $k=208910$, admissible tuples were constructed with a
hybrid shifted-Schinzel/greedy sieve. The algorithm first applies a
Schinzel-type presieve, then greedily removes a least-populated residue
class for successive primes, extracts the narrowest block of $k$ survivors,
and finally applies local diameter-reducing moves. Crucially, the discovery
trajectory is not trusted: the final tuple is rechecked from scratch for
every prime $p\le k$.

Each constructed tuple is exported as an
\texttt{admissible-tuple-certificate/1} record containing the ordered gap
sequence and, for every prime $p\le k$, an explicitly missing residue class.
A standalone verifier shares no construction code path. It regenerates the
primes up to $k$, reconstructs the tuple from the gap list, checks its
cardinality and diameter, verifies that the witness list contains exactly
the primes $p\le k$, and confirms by direct enumeration that every stated
residue class is absent. Since a $k$-element set cannot occupy all residue
classes modulo any prime $p>k$, these checks establish admissibility.

The resulting independently verified hybrid-sieve bounds are
\[
D_{3655}\le33270,\qquad
D_{208910}\le2718108.
\]
For $k=3655$, the stronger database value $D_{3655}\le33118$ is used in the
final bounded-gap comparison, while our independently constructed tuple
serves as a reproducibility control.

For the larger thresholds, explicit window sieving is unnecessary. Let
\[
\mathcal H_k=
\{p_{\pi(k)+1},\ldots,p_{\pi(k)+k}\},
\]
the first $k$ primes strictly exceeding $k$. Every element of
$\mathcal H_k$ is prime and greater than $k$, so residue class
$0\bmod p$ is empty for every prime $p\le k$; for $p>k$, a
$k$-element set cannot meet all $p$ residue classes. Hence
$\mathcal H_k$ is admissible by construction. Segmented prime enumeration
gives
\[
D_{11655069}\le214099720,
\qquad
D_{644589002}\le14541349288.
\]
The code additionally evaluates explicit prime-counting and prime-location
bounds as an independent large-$k$ consistency check.

Combining these tuple diameters with the certified variational thresholds
$M_k>4m$ gives
\[
H_m\le D_k,
\]
and hence the numerical bounded-gap statements reported in the Results.
The tuple-construction code, certificate files, and independent verifier
are supplied as Supplementary Software/Data.

\subsection{Standalone verification and certificate schema}
\label{sup:standalone_verifier}

The large-$k$ certificate is a self-contained JSON record containing the exact rational trial parameter
$c$, Fourier discretization parameters, intermediate summaries and the claimed lower bound. Stored
intermediate numerical enclosures are not trusted. The verifier reconstructs the analytic Fourier-tail
constant, regenerates the far-field block partition, recomputes wide-ball suprema, moment bounds,
Fourier nodes, trapezoidal sums, aliasing and truncation errors, and independently assembles
$R_{\rm low}$ in ball arithmetic.

The certificate proves only the Rayleigh quotient of the fixed admissible trial function encoded by
the exact rational $c$. It does not prove optimality of $c$, re-prove the external sieve criterion or
general upper bounds on $M_k$, or certify admissible-tuple diameters. Monte--Carlo comparisons and
external small-$k$ values are falsification checks rather than proof components.

%% file: Section1/input004asymptotic_v4.tex

\section{Asymptotics of the single-channel rational family}
\label{sec:rational-asymptotics}

We now explain analytically why the single-channel rational trial
functions used in the large-$k$ computations remain within a bounded
distance of the logarithmic scale $\log k$. 

For $k\ge 2$, write
\[
n:=k-1,
\]
and consider the symmetric trial function
\begin{equation}
F_{k,c}(t_1,\ldots,t_k)
=
\mathbf 1_{\{t_i\ge0,\ \sum_{i=1}^k t_i\le1\}}
\prod_{i=1}^k \frac{1}{c+n t_i},
\qquad c>0.
\label{eq:asym-trial}
\end{equation}
Its overall scalar normalization is immaterial.

The exact probabilistic reduction established in
Section~\ref{sup:probabilistic_reduction} gives
\begin{equation}
R(F_{k,c})
=
k\frac{N_{k,c}}{D_{k,c}},
\label{eq:asym-ratio-start}
\end{equation}
where
\[
w(t)=\frac{1}{(c+nt)^2},
\qquad
m_0=\int_0^1 w(t)\,dt
=\frac{1}{c(c+n)},
\]
and, if $T_1,\ldots,T_n$ are independent with density
\[
\frac{w(t)}{m_0}
=
\frac{c(c+n)}{(c+nt)^2},
\qquad 0\le t\le1,
\]
and
\[
S_n:=T_1+\cdots+T_n,
\]
then
\begin{align}
N_{k,c}
&=
\mathbb E\!\left[
G(1-S_n)^2\,;\,S_n<1
\right],
\\
D_{k,c}
&=
\mathbb E\!\left[
H(1-S_n)\,;\,S_n<1
\right],
\end{align}
with
\begin{align}
G(\rho)
&=
\frac{1}{n}
\log\left(1+\frac{n\rho}{c}\right),
\label{eq:asym-G}
\\
H(\rho)
&=
\frac{1}{n}
\left(
\frac1c-\frac{1}{c+n\rho}
\right).
\label{eq:asym-H}
\end{align}

The relevant scale for $c$ is
\begin{equation}
\frac1{c_k}
=
\log(k-1)+a,
\qquad a\in\mathbb R
\quad\text{fixed}.
\label{eq:fixed-shift-scaling}
\end{equation}
The numerical optimization of the rational family independently exhibits
this $c_k\asymp1/\log k$ scaling.  The result below identifies the limiting
$O(1)$ defect from $\log k$ for every fixed shift $a$.

\subsection{The limiting stable law}

We first isolate the random variable governing the boundary of the
simplex.  Define
\begin{equation}
X_n
:=
\frac{1-S_n}{c}.
\label{eq:Xn-def}
\end{equation}
Thus
\[
S_n<1
\quad\Longleftrightarrow\quad
X_n>0.
\]

It is convenient to rescale a single summand by writing
\[
V:=\frac{T}{c}.
\]
A direct change of variables gives the density
\begin{equation}
q_n(v)
=
\frac{c+n}{(1+nv)^2},
\qquad
0\le v\le\frac1c.
\label{eq:V-density}
\end{equation}
Hence
\begin{equation}
X_n
=
\frac1c-\sum_{j=1}^n V_j,
\label{eq:Xn-V}
\end{equation}
where $V_1,\ldots,V_n$ are independent with density
\eqref{eq:V-density}.

We use the following normalization of the totally right-skewed
$1$-stable law.

\begin{definition}[The limiting $1$-stable variable]
\label{def:Lambda}
Let $\Lambda$ be the infinitely divisible random variable with
characteristic function
\begin{equation}
\mathbb E e^{is\Lambda}
=
\exp\left\{
\int_0^\infty
\left(
e^{isx}-1-isx\mathbf 1_{\{x\le1\}}
\right)
\frac{dx}{x^2}
\right\}.
\label{eq:Lambda-cf}
\end{equation}
This is a spectrally positive stable law of index $1$, in the
L\'evy--Khintchine normalization used below.  Its law has a continuous,
everywhere positive density on $\mathbb R$.
\end{definition}

Because $\Lambda$ has only positive jumps, its Laplace transform is finite
in the damping direction.  The following elementary computation is used
repeatedly: it controls the left tail of $\Lambda$, supplies the
integrability needed in Lemma~\ref{lem:C-coercive}, and provides a
one-dimensional quadrature route to the constants $P_a$, $L_a$, $Q_a$
introduced below.

\begin{lemma}[Laplace transform and left tail of $\Lambda$]
\label{lem:laplace}
For every $\theta>0$,
\begin{equation}
\mathbb E\,e^{-\theta\Lambda}
=
\exp\bigl\{\theta\log\theta-\theta+\gamma_{\mathrm E}\theta\bigr\},
\label{eq:laplace-Lambda}
\end{equation}
where $\gamma_{\mathrm E}$ is the Euler--Mascheroni constant. To avoid collision we use $\gamma_{\mathrm E}$ instead of $\gamma_{\mathrm h}$ from Section \ref{sup:rank1_scale}. Consequently, for all
$y\ge0$,
\begin{equation}
\mathbb P(\Lambda\le -y)
\le
\exp\left\{-e^{\,y-\gamma_{\mathrm E}}\right\},
\label{eq:left-tail}
\end{equation}
and in particular
\begin{equation}
\kappa:=\mathbb E\bigl[(-\Lambda)^+\bigr]<\infty .
\label{eq:kappa-finite}
\end{equation}
\end{lemma}

\begin{proof}
Analytic continuation of \eqref{eq:Lambda-cf} to $s=i\theta$ gives
\[
J(\theta):=\log\mathbb E\,e^{-\theta\Lambda}
=
\int_0^\infty
\left(e^{-\theta x}-1+\theta x\mathbf 1_{\{x\le1\}}\right)
\frac{dx}{x^2},
\]
the integral being absolutely convergent for $\theta>0$.
Differentiating under the integral sign,
\[
J'(\theta)
=
\int_0^1\frac{1-e^{-\theta x}}{x}\,dx
-
\int_1^\infty\frac{e^{-\theta x}}{x}\,dx
=
\bigl(\log\theta+\gamma_{\mathrm E}+E_1(\theta)\bigr)-E_1(\theta)
=
\log\theta+\gamma_{\mathrm E},
\]
where $E_1$ is the exponential integral and we used the standard identity
$\int_0^z(1-e^{-t})t^{-1}dt=\gamma_{\mathrm E}+\log z+E_1(z)$.  Since $J(0^+)=0$,
integration gives \eqref{eq:laplace-Lambda}.

For \eqref{eq:left-tail}, Markov's inequality applied to $e^{-\theta\Lambda}$
yields, for every $\theta>0$,
\[
\mathbb P(\Lambda\le-y)
\le
e^{-\theta y}\,\mathbb E\,e^{-\theta\Lambda}
=
\exp\left\{\theta(\log\theta-1+\gamma_{\mathrm E}-y)\right\}.
\]
The exponent is minimized at $\theta=e^{\,y-\gamma_{\mathrm E}}$, where it equals
$-\theta$, giving \eqref{eq:left-tail}.  Finally
$\mathbb E[(-\Lambda)^+]=\int_0^\infty\mathbb P(\Lambda\le-y)\,dy
\le\int_0^\infty\exp\{-e^{y-\gamma_{\mathrm E}}\}\,dy<\infty$.
\end{proof}

\begin{lemma}[Stable limit]
\label{lem:stable-limit}
Suppose
\[
\frac1c=\log n+a
\]
with fixed $a\in\mathbb R$.  Then
\begin{equation}
X_n
\ \xrightarrow{\ d\ }\
X_a:=(a+1)-\Lambda .
\label{eq:stable-convergence}
\end{equation}
\end{lemma}

\begin{proof}
The characteristic function of $X_n$ is
\[
\phi_n(s)
=
e^{is/c}
\left(
\mathbb E e^{-isV}
\right)^n.
\]
Introduce the finite L\'evy measure
\[
\mu_n(dx)
:=
nq_n(x)\,dx
=
\frac{n(c+n)}{(1+nx)^2}
\mathbf 1_{\{0\le x\le1/c\}}\,dx,
\]
of total mass $\mu_n([0,1/c])=n$.
For every fixed $x>0$,
\[
\frac{n(c+n)}{(1+nx)^2}
\longrightarrow
\frac1{x^2}.
\]

For fixed $s$ and sufficiently large $n$, choose the logarithm of
$1+z_n(s)/n$ near $1$ and use the corresponding characteristic exponent
for $\phi_n(s)$.  With this convention, we write
\begin{align}
\log\phi_n(s)
&=
\frac{is}{c}
+
n\log\left(
1+\frac{z_n(s)}{n}
\right),
\qquad
z_n(s):=\int(e^{-isx}-1)\,\mu_n(dx).
\label{eq:logphi-pre}
\end{align}
Fix $s$.  Using $|e^{-isx}-1|\le\min(2,|s|x)$ and splitting the integral at
$x=1/|s|$,
\begin{equation}
|z_n(s)|
\ \le\
|s|\int_0^{1/|s|}x\,\mu_n(dx)+2\,\mu_n\bigl((1/|s|,1/c]\bigr)
\ =\ O_s(\log n),
\label{eq:zn-bound}
\end{equation}
by the elementary evaluation \eqref{eq:drift-integral} below.  Thus
$z_n(s)/n\to0$, and the principal branch of the logarithm satisfies
\[
n\log\left(1+\frac{z_n}{n}\right)
=
z_n+O\!\left(\frac{|z_n|^2}{n}\right)
=
z_n+O_s\!\left(\frac{\log^2 n}{n}\right)
=
z_n+o(1).
\]
(The hypothesis needed here is $z_n=o(\sqrt n)$, not boundedness of
$z_n$; \eqref{eq:zn-bound} supplies it.)  Hence
\begin{align}
\log\phi_n(s)
&=
\frac{is}{c}
+
\int_0^{1/c}
(e^{-isx}-1)\,\mu_n(dx)
+o(1)
\nonumber\\
&=
is\left[
\frac1c-\int_0^1x\,\mu_n(dx)
\right]
\nonumber\\
&\quad+
\int_0^{1/c}
\left(
e^{-isx}-1+isx\mathbf 1_{\{x\le1\}}
\right)
\mu_n(dx)
+o(1).
\label{eq:LK-decomp}
\end{align}

The compensated integral converges, by dominated convergence on compact
subintervals and elementary control at $0$ and $\infty$, to
\[
\int_0^\infty
\left(
e^{-isx}-1+isx\mathbf 1_{\{x\le1\}}
\right)
\frac{dx}{x^2}.
\]

It remains to identify the drift.  Substituting $u=nx$,
\begin{align}
\int_0^1x\,\mu_n(dx)
&=
\left(1+\frac cn\right)\int_0^n\frac{u\,du}{(1+u)^2}
=
\left(1+\frac cn\right)
\left[
\log(1+n)-1+\frac{1}{1+n}
\right].
\label{eq:drift-integral}
\end{align}
Since $c=1/(\log n+a)$,
\[
\frac1c-\int_0^1x\,\mu_n(dx)
\longrightarrow a+1.
\]
Consequently
\begin{align}
\log\phi_n(s)
\longrightarrow
is(a+1)
+
\int_0^\infty
\left(
e^{-isx}-1+isx\mathbf 1_{\{x\le1\}}
\right)
\frac{dx}{x^2}.
\end{align}
By Definition~\ref{def:Lambda}, the limiting characteristic function is
that of $(a+1)-\Lambda$.  L\'evy's continuity theorem gives
\eqref{eq:stable-convergence}.
\end{proof}

\subsection{Uniform density and logarithmic moment control}

Weak convergence alone is insufficient for the logarithmic observables
appearing in the Rayleigh quotient, because $\log x$ is singular at
$x=0$.  We therefore record two uniform estimates.

\begin{lemma}[Uniform Fourier integrability and bounded densities]
\label{lem:uniform-density}
For fixed $a$, with $c^{-1}=\log n+a$, there exists a constant
$C_a<\infty$ such that, for all sufficiently large $n$,
\begin{equation}
\int_{\mathbb R}|\phi_n(s)|\,ds
\le C_a.
\label{eq:phi-L1}
\end{equation}
Consequently $X_n$ possesses a density $f_n$ satisfying
\begin{equation}
\sup_{n\ge n_0}\|f_n\|_\infty<\infty.
\label{eq:density-uniform}
\end{equation}
In particular, uniformly for $0<\delta\le1$,
\begin{equation}
\mathbb P(0<X_n\le\delta)
\ll_a \delta.
\label{eq:small-ball}
\end{equation}
\end{lemma}

\begin{proof}
Since translation does not affect the modulus of the characteristic
function,
\[
|\phi_n(s)|
=
|\psi_n(s)|^n,
\qquad
\psi_n(s):=\mathbb E e^{-isV},
\]
and $|\psi_n|$ is even in $s$, so we may assume $s>0$.

\emph{An explicit two-interval bound.}
For $\beta>0$ let
\[
I_1(\beta):=\Bigl[0,\tfrac{\pi}{4}\beta\Bigr],
\qquad
I_2(\beta):=\Bigl[\tfrac{3\pi}{4}\beta,\tfrac{5\pi}{4}\beta\Bigr].
\]
These are disjoint, and for $v\in I_1(\beta)$, $v'\in I_2(\beta)$ we have
$v'-v\in[\tfrac{\pi}{2}\beta,\tfrac{5\pi}{4}\beta]$.  Taking
$\beta=1/s$ therefore gives $s(v'-v)\in[\pi/2,5\pi/4]$, on which
$\cos\le0$ and hence $1-\cos\bigl(s(v-v')\bigr)\ge1$.  Writing $V'$ for an
independent copy of $V$ and discarding all other configurations,
\begin{equation}
1-|\psi_n(s)|^2
=
\mathbb E\bigl[1-\cos\bigl(s(V-V')\bigr)\bigr]
\ \ge\
2\,\mathbb P\bigl(V\in I_1(1/s)\bigr)\,
\mathbb P\bigl(V\in I_2(1/s)\bigr).
\label{eq:two-interval}
\end{equation}
Both intervals lie in $[0,1/c]$ once $5\pi/(4s)\le1/c$, which holds for all
$s\ge1$ and $n$ large, since $1/c=\log n+a\to\infty$.

From \eqref{eq:V-density}, for $0\le\alpha<\beta\le1/c$,
\begin{equation}
\mathbb P(V\in[\alpha,\beta])
=
\frac{c+n}{n}
\left[\frac{1}{1+n\alpha}-\frac{1}{1+n\beta}\right].
\label{eq:V-interval-mass}
\end{equation}
Put $b:=\pi n/(4s)$.  Then \eqref{eq:V-interval-mass} gives
\[
\mathbb P\bigl(V\in I_1(1/s)\bigr)=\frac{c+n}{n}\cdot\frac{b}{1+b},
\qquad
\mathbb P\bigl(V\in I_2(1/s)\bigr)=\frac{c+n}{n}\cdot\frac{2b}{(1+3b)(1+5b)} .
\]

\emph{Low frequencies $1\le s\le n$.}  Here $b\ge\pi/4$, so
$b/(1+b)\ge(\pi/4)/(1+\pi/4)>\tfrac13$ and
$2b/\{(1+3b)(1+5b)\}\ge 2b/\{(3+4/\pi)b\cdot(5+4/\pi)b\}\gg 1/b\gg s/n$.
With \eqref{eq:two-interval},
\begin{equation}
1-|\psi_n(s)|^2
\gg \frac{s}{n},
\qquad
1\le s\le n,
\end{equation}
and therefore, for some $\kappa>0$ independent of $n$,
\begin{equation}
|\phi_n(s)|
=
\bigl(|\psi_n(s)|^2\bigr)^{n/2}
\le
\left(1-\frac{\kappa s}{n}\right)^{n/2}
\le e^{-\kappa s/2}
\qquad
(1\le s\le n).
\label{eq:lowfreq-decay}
\end{equation}

\emph{Intermediate frequencies $n\le s\le 8(c+n)$.}  Now $b\le\pi/4$, so
$b/(1+b)\gg b$ and $2b/\{(1+3b)(1+5b)\}\gg b$, whence
$1-|\psi_n(s)|^2\gg b^2\gg n^2/s^2$.  Since $s=O(n)$ throughout this range,
\begin{equation}
|\phi_n(s)|
\le e^{-\kappa' n}.
\label{eq:midfreq-decay}
\end{equation}

\emph{High frequencies.}  Integration by parts in
\[
\psi_n(s)
=
\int_0^{1/c}
e^{-isv}\frac{c+n}{(1+nv)^2}\,dv
\]
gives
\begin{equation}
|\psi_n(s)|
\le
\frac{2(c+n)}{s}.
\label{eq:highfreq-single}
\end{equation}
Thus, for $s\ge8(c+n)$,
\begin{equation}
|\phi_n(s)|
\le
\left(\frac14\right)^n
\left(\frac{8(c+n)}{s}\right)^n
\le
\left(\frac14\right)^{n},
\label{eq:highfreq-decay}
\end{equation}
and more precisely $|\phi_n(s)|\le(2(c+n)/s)^n$, whose integral over
$s\ge8(c+n)$ is $O\bigl((c+n)4^{-n}\bigr)$, hence uniformly bounded.

Combining \eqref{eq:lowfreq-decay},
\eqref{eq:midfreq-decay}, and
\eqref{eq:highfreq-decay}, together with the trivial bound
$|\phi_n(s)|\le1$ for $|s|\le1$, proves \eqref{eq:phi-L1}.

Fourier inversion then gives
\[
\|f_n\|_\infty
\le
\frac{1}{2\pi}\|\phi_n\|_{L^1},
\]
which proves \eqref{eq:density-uniform}; \eqref{eq:small-ball} follows
immediately.
\end{proof}

\begin{lemma}[Right-tail bound and logarithmic uniform integrability]
\label{lem:log-UI}
For fixed $a$, the random variables
\[
\log X_n\,\mathbf 1_{\{X_n>0\}}
\quad\text{and}\quad
(\log X_n)^2\,\mathbf 1_{\{X_n>0\}}
\]
are uniformly integrable.  Consequently, writing
\begin{align}
P_n&:=\mathbb P(X_n>0),\\
L_n&:=\mathbb E[\log X_n\,;\,X_n>0],\\
Q_n&:=\mathbb E[(\log X_n)^2\,;\,X_n>0],
\end{align}
we have
\begin{align}
P_n&\longrightarrow P_a:=\mathbb P(X_a>0)>0,
\label{eq:P-conv}\\
L_n&\longrightarrow L_a:=\mathbb E[\log X_a\,;\,X_a>0],
\label{eq:L-conv}\\
Q_n&\longrightarrow Q_a:=\mathbb E[(\log X_a)^2\,;\,X_a>0].
\label{eq:Q-conv}
\end{align}
\end{lemma}

\begin{proof}
The negative logarithmic tail is controlled by
Lemma~\ref{lem:uniform-density}.  For $y\ge0$,
\[
\mathbb P(0<X_n<e^{-y})
\ll e^{-y}.
\]
Hence the negative parts of $\log X_n$ have uniformly bounded moments of
every fixed order.

For the positive tail, note from \eqref{eq:Xn-V} that
\[
X_n>x
\quad\Longleftrightarrow\quad
\sum_{j=1}^nV_j<\frac1c-x.
\]
For $\lambda>0$, Chernoff's inequality gives
\begin{equation}
\mathbb P(X_n>x)
\le
\exp\!\left(
\lambda\left(\frac1c-x\right)
\right)
\left(\mathbb E e^{-\lambda V}\right)^n.
\label{eq:Chernoff-start}
\end{equation}

We claim that there is a constant $C_a$, independent of $n$ and
$\lambda$, such that
\begin{equation}
n\bigl(1-\mathbb E e^{-\lambda V}\bigr)
\ \ge\
\lambda
\left(
\log\frac n\lambda-C_a
\right),
\qquad 1\le\lambda\le n .
\label{eq:Laplace-lower}
\end{equation}
Substituting $u=nv$ in \eqref{eq:V-density} and writing $\beta:=\lambda/n\in(0,1]$,
\begin{align}
n\bigl(1-\mathbb E e^{-\lambda V}\bigr)
&=
(c+n)\int_0^{n/c}
\bigl(1-e^{-\beta u}\bigr)
\frac{du}{(1+u)^2}.
\label{eq:Laplace-exact}
\end{align}
Rather than splitting the range, evaluate the integral exactly.  Integration
by parts gives, for the untruncated integral,
\[
\int_0^\infty\bigl(1-e^{-\beta u}\bigr)\frac{du}{(1+u)^2}
=
\beta\int_0^\infty\frac{e^{-\beta u}}{1+u}\,du
=
\beta\,e^{\beta}E_1(\beta).
\]
To make the estimate uniform over the entire range $0<\beta\le1$, note that
$e^{\beta}E_1(\beta)-\log(1/\beta)$ is continuous on $(0,1]$ and tends to
$-\gamma_{\mathrm E}$ as $\beta\downarrow0$.  Consequently,
\[
\sup_{0<\beta\le1}
\left|e^{\beta}E_1(\beta)-\log(1/\beta)\right|<\infty.
\]
The discarded tail satisfies
$\int_{n/c}^\infty(1+u)^{-2}du\le c/n$, contributing at most
$(c+n)c/n=O(1)\le O(\lambda)$ since $\lambda\ge1$.  Multiplying by $(c+n)=n(1+c/n)$ and using $\beta=\lambda/n$,
with $(c/n)\log(n/\lambda)\le(c/n)\log n=O(n^{-1})$, yields
\[
n\bigl(1-\mathbb E e^{-\lambda V}\bigr)
=
\lambda\left(\log\frac n\lambda-\gamma_{\mathrm E}+O(1)\right),
\]
which is \eqref{eq:Laplace-lower}.  (The coefficient $1$ in front of
$\log(n/\lambda)$ is essential: it is what cancels $1/c=\log n+a$
in \eqref{eq:Chernoff-start}.  A cruder split using
$1-e^{-y}\ge y/2$ produces the coefficient $1/2$ and the argument fails.)

Since $\log z\le z-1$ for $z>0$,
\[
n\log\mathbb E e^{-\lambda V}
\le
-n\bigl(1-\mathbb E e^{-\lambda V}\bigr).
\]
Using $c^{-1}=\log n+a$ in \eqref{eq:Chernoff-start} therefore gives
\begin{equation}
\mathbb P(X_n>x)
\le
\exp\left\{
\lambda(\log\lambda+A_a-x)
\right\},
\qquad 1\le\lambda\le n,
\label{eq:right-tail-preopt}
\end{equation}
with $A_a:=a+C_a$, enlarging $C_a$ if necessary so that $C_a\ge0$.

Choosing
\[
\lambda=e^{x-A_a-1}
\]
whenever this lies in $[1,n]$ makes the exponent equal to $-\lambda$, giving
the double-exponential bound
\begin{equation}
\mathbb P(X_n>x)
\le
\exp\left\{-e^{\,x-A_a-1}\right\},
\qquad
A_a+1\le x\le \log n+A_a+1 .
\label{eq:right-tail}
\end{equation}
For $x<A_a+1$ the trivial bound $\mathbb P\le1$ suffices, and for
$x>\log n+A_a+1$ we have $x>1/c$ (as $C_a>-1$), so
$\mathbb P(X_n>x)=0$ because $X_n\le1/c$ almost surely.

Thus both the positive and negative logarithmic tails are uniformly
integrable.  Combining this with
Lemma~\ref{lem:stable-limit}, and noting that the limiting stable law has
a continuous positive density, so that $\mathbb P(X_a=0)=0$ and $P_a>0$,
gives \eqref{eq:P-conv}--\eqref{eq:Q-conv}.
\end{proof}

\subsection{Asymptotics of the Rayleigh quotient}

We now return to the exact Rayleigh quotient.

Since $1-S_n=cX_n$, equations
\eqref{eq:asym-G} and \eqref{eq:asym-H} give, on $\{X_n>0\}$,
\begin{align}
G(1-S_n)
&=
\frac1n\log(1+nX_n),
\label{eq:G-X}\\
\frac{H(1-S_n)}{m_0}
&=
\frac{(c+n)X_n}{1+nX_n}.
\label{eq:H-X}
\end{align}
Therefore
\begin{equation}
R(F_{k,c})
=
c\left(1+\frac1n\right)
\left(1+\frac cn\right)
\frac{
\mathbb E[
\log^2(1+nX_n);\,X_n>0]
}{
\mathbb E[
\frac{(c+n)X_n}{1+nX_n};\,X_n>0]
}.
\label{eq:R-X-exact}
\end{equation}

Let
\[
\ell:=\log n,
\qquad
c=\frac1{\ell+a}.
\]

\begin{lemma}[Denominator asymptotics]
\label{lem:den-asym}
For fixed $a$,
\begin{equation}
\mathbb E\!\left[
\frac{(c+n)X_n}{1+nX_n};\,X_n>0
\right]
=
P_n+O\bigl(n^{-1/2}\bigr).
\label{eq:den-asym}
\end{equation}
\end{lemma}

\begin{remark}
The explicit rate, rather than a bare $o(1)$, is what the proof of
Theorem~\ref{thm:rational-asymptotic} requires: the denominator is divided
into a numerator of size $\ell^2$, so an additive error $\eta$ here
propagates to an error of size $\ell\eta$ in $R$.  An $o(1)$ statement
would only give $o(\ell)$, which is insufficient.
\end{remark}

\begin{proof}
For $0<X_n\le1/c$,
\[
0\le
1-
\frac{(c+n)X_n}{1+nX_n}
=
\frac{1-cX_n}{1+nX_n}
\le
\frac{1}{1+nX_n}.
\]
Choose
\[
\delta_n=n^{-1/2}.
\]
On $0<X_n\le\delta_n$, Lemma~\ref{lem:uniform-density} gives probability
$O(\delta_n)$.  On $X_n>\delta_n$,
\[
\frac{1}{1+nX_n}
\le
\frac1{n\delta_n}
=
n^{-1/2}.
\]
Hence the difference between the left-hand side of
\eqref{eq:den-asym} and $P_n$ is $O(n^{-1/2})$.
\end{proof}

\begin{lemma}[Numerator asymptotics]
\label{lem:num-asym}
For fixed $a$,
\begin{equation}
\mathbb E[
\log^2(1+nX_n);\,X_n>0]
=
\ell^2P_n+2\ell L_n+Q_n+O\bigl(\ell^2n^{-1/2}\bigr).
\label{eq:num-asym}
\end{equation}
\end{lemma}

\begin{proof}
Let $\delta_n=n^{-1/2}$.

\emph{Region $X_n>\delta_n$.}  Here
\[
\log(1+nX_n)
=
\ell+\log X_n+r_n,
\qquad
0\le r_n=\log\left(1+\frac1{nX_n}\right)\le\frac1{nX_n}\le n^{-1/2},
\]
so that
\[
\log^2(1+nX_n)
=
(\ell+\log X_n)^2+2(\ell+\log X_n)r_n+r_n^2 .
\]
By Lemma~\ref{lem:log-UI} the quantities $\mathbb E[|\log X_n|;X_n>0]$ are
bounded uniformly in $n$, so
\[
\mathbb E\bigl[\,\bigl|2(\ell+\log X_n)r_n\bigr|+r_n^2;\,X_n>\delta_n\bigr]
=
O\bigl(\ell n^{-1/2}\bigr).
\]

\emph{Region $0<X_n\le\delta_n$.}  The uniform density bound of
Lemma~\ref{lem:uniform-density} gives
\begin{align}
\mathbb E\!\left[
(\ell+\log X_n)^2;\,0<X_n\le\delta_n
\right]
&\le
C\int_0^{\delta_n}(\ell+\log x)^2\,dx
\nonumber\\
&=
C\delta_n\Bigl[(\ell+\log\delta_n)^2-2(\ell+\log\delta_n)+2\Bigr]
=
O\bigl(\ell^2n^{-1/2}\bigr),
\label{eq:small-region-log}
\end{align}
using $\ell+\log\delta_n=\ell/2$.  Likewise, since $nX_n\le n\delta_n=n^{1/2}$
on this region,
\[
\mathbb E[\log^2(1+nX_n);\,0<X_n\le\delta_n]
\le
C\int_0^{\delta_n}\log^2(1+nx)\,dx
=
O\bigl(\ell^2 n^{-1/2}\bigr).
\]

Adding the two regions and recognising
$\mathbb E[(\ell+\log X_n)^2;X_n>0]=\ell^2P_n+2\ell L_n+Q_n$
proves \eqref{eq:num-asym}.
\end{proof}

\subsection{Main asymptotic theorem}

\begin{theorem}[Asymptotic defect of the rational trial family]
\label{thm:rational-asymptotic}
Fix $a\in\mathbb R$.  For all sufficiently large $k$, put
\[
n=k-1,
\qquad
c_k=\frac{1}{\log n+a}.
\]
For all sufficiently large $k$, $c_k>0$.  Let $F_{k,c_k}$ be the
single-channel rational trial function \eqref{eq:asym-trial}.  Then
\begin{equation}
R(F_{k,c_k})
=
\log k-\mathcal C(a)+o(1),
\qquad
k\to\infty,
\label{eq:main-asymptotic}
\end{equation}
where
\begin{equation}
\boxed{
\mathcal C(a)
=
a
-
2\,
\mathbb E\!\left[
\log X_a\mid X_a>0
\right]
}
\label{eq:C-a}
\end{equation}
and
\[
X_a=(a+1)-\Lambda
\]
with $\Lambda$ defined by \eqref{eq:Lambda-cf}.

Equivalently,
\begin{equation}
\lim_{k\to\infty}
\left(
\log k-R(F_{k,c_k})
\right)
=
\mathcal C(a).
\label{eq:defect-limit}
\end{equation}
\end{theorem}

\begin{proof}
Write $\mathrm{Num}$ and $\mathrm{Den}$ for the numerator and denominator
in \eqref{eq:R-X-exact}.  By Lemmas~\ref{lem:den-asym} and
\ref{lem:num-asym},
\[
\mathrm{Num}=\ell^2P_n+2\ell L_n+Q_n+O\bigl(\ell^2n^{-1/2}\bigr)
=O(\ell^2),
\qquad
\mathrm{Den}=P_n+O\bigl(n^{-1/2}\bigr),
\]
and $P_n\to P_a>0$ by Lemma~\ref{lem:log-UI}, so $\mathrm{Den}$ is bounded
away from $0$ for large $n$.  Hence
\[
\frac{\mathrm{Num}}{\mathrm{Den}}
=
\frac{\ell^2P_n+2\ell L_n+Q_n}{P_n}
+O\bigl(\ell^2n^{-1/2}\bigr).
\]
Multiplying by $c(1+\tfrac1n)(1+\tfrac cn)=c\bigl(1+O(n^{-1})\bigr)$ and
using $c=(\ell+a)^{-1}$,
\begin{equation}
R(F_{k,c})
=
\frac{
\ell^2
+
2\ell\,L_n/P_n
+
Q_n/P_n
}{
\ell+a
}
+O\bigl(\ell\,n^{-1/2}\bigr).
\label{eq:R-expansion-pre}
\end{equation}
A direct rearrangement yields the finite-$n$ expansion
\begin{align}
\ell-R(F_{k,c})
&=
a
-
2\frac{L_n}{P_n}
\nonumber\\
&\quad-
\frac{
Q_n/P_n
-
2aL_n/P_n
+
a^2
}{
\ell+a
}
+O\bigl(\ell\,n^{-1/2}\bigr).
\label{eq:finite-n-expansion}
\end{align}
By Lemma~\ref{lem:log-UI},
\[
\frac{L_n}{P_n}
\longrightarrow
\frac{L_a}{P_a}
=
\mathbb E[\log X_a\mid X_a>0],
\]
while the fraction on the second line of
\eqref{eq:finite-n-expansion} tends to zero.  Hence
\[
\ell-R(F_{k,c})
\longrightarrow
a
-
2\mathbb E[\log X_a\mid X_a>0].
\]
Finally,
\[
\log k-\log(k-1)=O(n^{-1}),
\]
so $\ell$ may be replaced by $\log k$.  This proves
\eqref{eq:main-asymptotic} and \eqref{eq:defect-limit}.
\end{proof}

\subsection{Consequence for $M_k$}

Define
\begin{equation}
\mathcal C_*
:=
\inf_{a\in\mathbb R}\mathcal C(a).
\label{eq:C-star}
\end{equation}

\begin{corollary}[Asymptotic lower bound for $M_k$]
\label{cor:Mk-asym}
The Maynard variational constant satisfies
\begin{equation}
\liminf_{k\to\infty}
\bigl(M_k-\log k\bigr)
\ge
-\mathcal C_*.
\label{eq:Mk-liminf}
\end{equation}
Equivalently,
\begin{equation}
M_k
\ge
\log k-\mathcal C_*-o(1).
\label{eq:Mk-Cstar}
\end{equation}
\end{corollary}

\begin{proof}
For every $\varepsilon>0$, choose a fixed
$a_\varepsilon\in\mathbb R$ such that
\[
\mathcal C(a_\varepsilon)
\le \mathcal C_*+\varepsilon.
\]
Since $M_k$ is the supremum over admissible trial functions,
Theorem~\ref{thm:rational-asymptotic} gives
\[
M_k
\ge
R(F_{k,c_k})
=
\log k-\mathcal C(a_\varepsilon)+o(1)
\]
for the choice
\[
c_k^{-1}=\log(k-1)+a_\varepsilon.
\]
Thus
\[
\liminf_{k\to\infty}
(M_k-\log k)
\ge
-\mathcal C_*-\varepsilon.
\]
Letting $\varepsilon\downarrow0$ proves \eqref{eq:Mk-liminf}.
\end{proof}

The infimum in \eqref{eq:C-star} may be restricted to a half-line bounded
above, since $\mathcal C$ grows without bound in the positive direction.

\begin{lemma}[Growth of $\mathcal C$ for large positive shift]
\label{lem:C-coercive}
With $\kappa=\mathbb E[(-\Lambda)^+]<\infty$ as in \eqref{eq:kappa-finite},
there is $a_0$ such that for all $a\ge a_0$,
\begin{equation}
\mathcal C(a)
\ \ge\
a-2\log\bigl(2(a+1+\kappa)\bigr)
\ \xrightarrow[a\to\infty]{}\ \infty .
\label{eq:C-coercive}
\end{equation}
In particular $\inf_{a\in\mathbb R}\mathcal C(a)
=\inf_{a\le a_1}\mathcal C(a)$ for some finite $a_1$.
\end{lemma}

\begin{proof}
Since $X_a^+=((a+1)-\Lambda)^+\le(a+1)+(-\Lambda)^+$ for $a\ge-1$,
\[
\mathbb E[X_a;X_a>0]\le(a+1)+\kappa .
\]
Also $P_a=\mathbb P(\Lambda<a+1)\to1$, so $P_a\ge\tfrac12$ for $a\ge a_0$.
By Jensen's inequality applied to the concave function $\log$ under the
conditional law,
\[
\mathbb E[\log X_a\mid X_a>0]
\le
\log\mathbb E[X_a\mid X_a>0]
=
\log\frac{\mathbb E[X_a;X_a>0]}{P_a}
\le
\log\bigl(2(a+1+\kappa)\bigr),
\]
which gives \eqref{eq:C-coercive}.
\end{proof}

\begin{remark}[Behaviour as $a\to-\infty$, and attainment of the infimum]
\label{rem:left-coercive}
Numerical results suggest that $\mathcal C(a)\to\infty$ as
$a\to-\infty$ as well, and that the infimum in \eqref{eq:C-star} is
attained.  These assertions are not proved here.  The proposed mechanism
is that conditioning on $\{X_a>0\}=\{\Lambda<a+1\}$
for very negative $a$ forces $X_a$ to concentrate at $0$: writing
$y=-(a+1)$, the ratio
$\mathbb P(\Lambda<-y-t)/\mathbb P(\Lambda<-y)$ is expected to behave like
$\exp\{-c\,e^{y}t\}$ for small $t>0$, so that conditionally $X_a$ is
approximately exponential with rate of order $e^{-(a+1)}$, giving
$\mathbb E[\log X_a\mid X_a>0]=(a+1)+O(1)$ and hence
$\mathcal C(a)=-a+O(1)$.  A rigorous proof requires a two-sided
left-tail estimate for $\Lambda$; Lemma~\ref{lem:laplace} supplies only the
upper bound \eqref{eq:left-tail}, and the Jensen argument of
Lemma~\ref{lem:C-coercive} is too lossy in this regime because $P_a$ is
doubly-exponentially small.  We therefore state
\eqref{eq:C-star} as an infimum; none of the results below depend on
attainment.
\end{remark}

\begin{remark}[Scope of the optimization]
\label{rem:fixed-shift-optimality}
The quantity $\mathcal C_*$ is the optimal asymptotic constant
\emph{among the fixed-shift scalings}
\[
c_k^{-1}=\log(k-1)+a,
\qquad a\in\mathbb R.
\]
The theorem does not by itself prove that the same constant is optimal
over every possible sequence $c_k>0$ in the full single-channel rational
family.  Such a statement would additionally require showing that every
asymptotically competitive sequence satisfies
\[
c_k^{-1}-\log k=O(1),
\]
and controlling non-convergent or divergent shift sequences.  We do not
need this stronger statement for the lower bound
\eqref{eq:Mk-Cstar}.
\end{remark}

\subsection{Numerical interpretation}

Theorem~\ref{thm:rational-asymptotic} is an analytic statement and does
not depend on the numerical discovery or certification pipeline.  Numerical
quadrature of the limiting stable law can, however, be used to evaluate
$\mathcal C(a)$ and to locate a favorable fixed shift.  Two independent
routes are available: direct quadrature against the density of $\Lambda$,
and inversion of the explicit Laplace transform \eqref{eq:laplace-Lambda}.
Our numerical evaluation indicates a minimum near a small negative value of
$a$ and an asymptotic defect of approximately
\begin{equation}
\mathcal C_*
\approx 0.3343 .
\label{eq:Cstar-numerical}
\end{equation}
The value in \eqref{eq:Cstar-numerical} is a numerical estimate of the
infimum, not a certified enclosure established by this asymptotic analysis.
Certifying a finite-$k$ Rayleigh quotient and certifying the limiting
constant are separate tasks.

For an explicit asymptotic lower bound it suffices to choose a single
fixed shift $a_0$ and establish, by rigorous interval quadrature with
controlled truncation errors, an upper bound $\mathcal C(a_0)\le U$.
Theorem~\ref{thm:rational-asymptotic} then gives
\[
\liminf_{k\to\infty}(M_k-\log k)\ge-U,
\]
without any proof of global minimality.  By contrast, a rigorous enclosure
of $\mathcal C_*$ itself also requires a lower bound valid for every
$a\in\mathbb R$.  Any numerical search restricted to a finite interval
must therefore be supplemented by rigorous control of both excluded
tails, including $a\to-\infty$.  The upper bound on the left tail of
$\Lambda$ proved above does not by itself supply that control.

The finite-$k$ optimization exhibits the same qualitative scaling:
\[
c_k\sim\frac1{\log k},
\qquad
R(F_{k,c_k})=\log k-O(1).
\]
The fact that the discrepancy
$\log k-R(F_{k,c_k})$ varies only slowly over more than five orders of
magnitude in $k$ is therefore explained by the stable-law limit rather
than being an empirical coincidence.

\subsection{Finite-size corrections}

Expansion \eqref{eq:finite-n-expansion} also isolates the finite-size
correction.  All error terms accumulated in its derivation are
$O(\ell\,n^{-1/2})$, hence $o(\ell^{-m})$ for every fixed $m$; the
replacement of $\log n$ by $\log k$ contributes $O(n^{-1})$.  Consequently
\begin{align}
\log k-R(F_{k,c})
&=
a
-
2\frac{L_n}{P_n}
-
\frac{
Q_n/P_n
-
2aL_n/P_n
+
a^2
}{
\ell+a
}
+O\bigl(\ell\,n^{-1/2}\bigr),
\label{eq:second-order-identity}
\end{align}
Define
\[
d_n(a):=\frac{L_n}{P_n}-\frac{L_a}{P_a},
\qquad
B(a):=\frac{Q_a}{P_a}-2a\frac{L_a}{P_a}+a^2.
\]
Since $P_n\to P_a>0$, $L_n\to L_a$, and $Q_n\to Q_a$, the expansion above
implies
\begin{equation}
\log k-R(F_{k,c})
=\mathcal C(a)-2d_n(a)-\frac{B(a)}{\ell}+o(\ell^{-1}).
\label{eq:second-order-reduction}
\end{equation}
Here $d_n(a)\to0$, but no quantitative rate for this convergence has been
proved.  Thus $d_n(a)$ may dominate the displayed $1/\ell$ term, and
\eqref{eq:second-order-reduction} is not yet a two-term asymptotic expansion
with a determined coefficient.

A sufficient additional estimate is
\[
d_n(a)=o(\ell^{-1}),
\]
in which case
\[
\log k-R(F_{k,c})
=\mathcal C(a)-\frac{B(a)}{\log k}+o\bigl((\log k)^{-1}\bigr).
\]
More generally, if $\ell d_n(a)\to b(a)$, the coefficient of
$1/\log k$ is $-[B(a)+2b(a)]$.  A bound of only $d_n(a)=O(\ell^{-1})$
does not determine that coefficient.  Accordingly, a quantitative
stable-limit argument must establish an appropriate rate for these
conditional logarithmic moments, or their first-order asymptotics;
a distributional convergence rate alone requires an additional argument
to handle the logarithmic singularity.  Any fitted coefficient of
$1/\log k$ remains a numerical finite-size observation until such an
estimate is supplied.

\subsection{Interpretation}

The asymptotic mechanism can be summarized as follows.  Squaring the
single rational channel induces the one-dimensional density
\[
q_n(v)=\frac{c+n}{(1+nv)^2},
\]
whose centered $n$-fold sum converges to a spectrally positive
$1$-stable law.  Spectral positivity refers to the jumps; the limiting
law has support on all of $\mathbb R$.  Choosing
\[
c^{-1}=\log n+a
\]
centers the boundary variable
\[
X_n=\frac1c-\sum_{j=1}^nV_j
\]
at a non-degenerate limit.  Conditional on the event $X_n>0$, which is
exactly the simplex event $S_n<1$, the numerator of the Rayleigh quotient
contains
\[
\log^2(1+nX_n)
=
\bigl(\log n+\log X_n+o(1)\bigr)^2,
\]
whereas the denominator converges to the same survival probability.
The leading $\log^2 n$ term is therefore divided by
$c^{-1}\sim\log n$, producing a Rayleigh quotient of size $\log n$.
The conditional logarithmic moment of the limiting stable variable
supplies the non-trivial $O(1)$ defect.

Thus the rational family is not merely a finite-dimensional numerical
ansatz that happens to perform well at large $k$.  Its observed scaling is
the finite-$k$ manifestation of a stable-law limit, and the family gives
the explicit asymptotic lower bound
\[
M_k\ge\log k-\mathcal C_*-o(1).
\]

%% file: Section2/04Delsarte_SUPPLEMENTARY_v5.tex
\section{Supplementary Methods: higher-order Delsarte hierarchy}

Notation follows the Methods. Throughout, $\mathsf H(\cdot)$ denotes Shannon
entropy in bits and $H_2(p)=-p\log_2p-(1-p)\log_2(1-p)$ the binary entropy
function; $G$ is a normalized configuration on $\mathbb F_2^\ell$, i.e.\ a
probability distribution, and $\Phi(G,v)=\sum_u\sqrt{G(u)G(u+v)}$ is the
translation affinity at a shift $v\neq0$.

\subsection{Scale-up of the hierarchy to \texorpdfstring{$r\ge3$}{r>=3}}
\label{sup:scaleup}

The finite-length hierarchy was first implemented at $r=2$, where
$\mathrm{GL}(2,2)\cong S_3$ acts as the full symmetric group on the three
non-zero vectors of $\mathbb F_2^2$ and orbit representatives can be generated
by sorting. This shortcut cannot be generalized to $r=3$:
$|\mathrm{GL}(3,2)|=168$ whereas $|S_7|=5040$. Sorting therefore over-merges
genuinely distinct Fano-plane orbits; at $n=6$, $45$ true orbits collapse to
$30$ sorted classes. We instead enumerate $\mathrm{GL}(r,2)$ explicitly and
construct genuine type orbits under the group action.

For $S\subset[r]$, partial Fourier transforms factor after splitting
$\mathbb F_2^r=\mathbb F_2^S\oplus\mathbb F_2^{S^c}$:
\[
K^S_\alpha(\beta)
=
\prod_{w\in\mathbb F_2^{S^c}}
K^{(|S|)}_{\alpha(\cdot,w)}
\!\left(\beta(\cdot,w)\right),
\]
with zero value unless the $S^c$-marginals of $\alpha$ and $\beta$ agree. Each
factor is a lower-dimensional full Krawtchouk transform, so the implementation
recurses through the same exact row generator. Counting pairs of matrices of
types $(\alpha,\beta)$ in the two possible orders gives the reciprocity
\[
\binom{n}{\beta}K_\alpha(\beta)
=
\binom{n}{\alpha}K_\beta(\alpha).
\]
Combined with $\mathrm{GL}(r,2)$ equivariance, this allows a row to be generated
from one representative expansion per valid orbit. At $r=3$, $n=14$, $d=6$,
this reduces the number of expansions from $1235$ to $112$.

The resulting $r=3$ implementation reproduced the $r=2$ program when run at
$r=2$ and produced certified strict level-$3$ separations, including
$V_2(9,5)=44/9>V_3(9,5)=4$ and
$V_2(12,5)=24.260255\ldots>V_3(12,5)=16$.

\subsection{Exact dual certification and implementation checks}
\label{sup:exactcert}

The floating-point solver is used as a proposal mechanism for an active dual
set. The corresponding basic system is reconstructed at arbitrary precision,
rounded to dyadic rationals, and verified against every dual inequality. The
repair identity quoted in the Methods supplies an exactly feasible direction
and an explicit rational repair cost, so the certified value is an exact
rational upper bound on the level-$r$ optimum obtained by weak duality.

Three independent anchors were used. First, full and partial Krawtchouk rows
were compared with brute-force character sums over
$\mathbb F_2^{r\times n}$ for small $n$. Second, $r=1$ was checked against an
independent classical Delsarte implementation. Third, the published $r=2$
values of Loyfer and Linial were reproduced, including
$24.260255\ldots$ at $(13,6)$, $131.720587\ldots$ at $(16,6)$,
$32$ at $(16,8)$ and $(17,8)$, and $256$ at $(17,6)$ and $(20,8)$.

Structural invariant tests detected defects that value reproduction did not.
An early active-set implementation factorized a matrix shifted by one column
relative to the solved system, yielding $32.0000887\ldots$ instead of $32$ at
$(17,8)$. Two precision-state errors in dyadic rounding and polishing partially
masked this defect. These issues were corrected before the experiments
reported here; all final bounds use the repaired exact-verification path.

\subsection{Lift experiments and finite-length structural map}
\label{sup:lift}

Coregliano et al.\ construct an explicit lift of a level-one dual to level
$\ell$ with objective $V_1^\ell$ \cite{coregliano2025higher}. Because
the level-$\ell$ hierarchy bounds $|C|^\ell$, the lift must be compared in this
normalization. An initial level-one-normalized comparison reproduced $V_1$ but
violated higher-level feasibility by $+9949$, which served as a normalization
control before the systematic sweep.

The lift has no mass on partial families below the full transform. We therefore
compared the full-Fourier-only hierarchy with the complete partial-Fourier
program. Across twenty tested instances the full-only optimum never exceeded
$V_1^r$. At $r=2$ it was frequently exactly equal to the lift, for example
\[
36=6^2\quad(9,5),\qquad
576=24^2\quad(11,5),\qquad
144=12^2\quad(10,5),
\]
while the partial-Fourier program improved strictly. At $r=3$, by contrast, the
full-only hierarchy itself generally improved on the lift. The correction to a
lifted level-one solution is therefore not a common object across levels.

A separate scan at approximately fixed relative distance showed why these
finite instances were not used for asymptotic learning. At $\delta=0.40$, the
accessible finite-$n$ rates were $0.33$--$0.50$, compared with the first-MRRW
value $0.0815$, and the level-one-to-level-two gain decreased with $n$ at fixed
$d$. Extrapolation from the accessible $n$ range would therefore have mixed
strong finite-size effects with the asymptotic target.

\subsection{Normalization of the configuration objective}
\label{sup:normalization}

For a level-$\ell$ dual,
\[
|C|^\ell\le f(0),
\]
so an asymptotic bound obtained from a normalized configuration $G$ is
\[
R(C)\le \frac{\mathsf H(G)}{\ell}+O\!\left(\frac{\log n}{n}\right),
\]
and $\mathsf H(G)$ must not be compared directly with an ordinary code-rate
bound. Re-deriving this normalization exposed two discrepancies in the cited
preprint. First, the leading small-$\varepsilon$ coefficient of the first MRRW
expression is $\tfrac12\varepsilon^2\log_2(1/\varepsilon)$ rather than
$\tfrac14\varepsilon^2\log_2(1/\varepsilon)$: with
$p_\varepsilon=\tfrac{\varepsilon^2}{4}+O(\varepsilon^4)$ and
$\log_2(1/p_\varepsilon)=2\log_2(1/\varepsilon)+2+O(\varepsilon^2)$, one gets
$h(\varepsilon)=\tfrac12\varepsilon^2\log_2(1/\varepsilon)+O(\varepsilon^2)$.
Second, the closed form stated in their Lemma~6.10(i) is a lower bound on the
walk count rather than an equality, since its proof passes through
$\binom{m/2}{F}^2\ge\binom{m/2}{F}$. Neither affects the qualitative
conclusion of that work: the stated rate bound is off by a factor $2\ell$
relative to $\mathsf H(G)/\ell$ --- the $\ell$ being the rate normalization
above and the $2$ the same halving --- but the factor is applied on both sides
of the comparison made there.

\subsection{Reduced configuration search}
\label{sup:search}

With $G=p^2$, the configuration problem becomes
\[
\min_{\|p\|_2=1,\;p\ge0}
\frac1\ell \mathsf H(p^2)
\]
subject to quadratic translation-affinity constraints
$p^{\mathsf T}S_vp\ge\varepsilon$. The largest search space considered
($\ell=4$) has $15$ free probability coordinates. We therefore used direct
multistart sequential quadratic programming with analytic gradients rather than
a neural network. The campaign comprised $28$ instance solves over
$\ell\le4$ and $\varepsilon\in\{0.30,0.20,0.10,0.05\}$, using $400$--$600$
randomized starts per instance ($13{,}600$ local optimizations in total). One
start per instance used the CJJ vertex-uniform configuration and the remaining
starts were randomized positive-sphere points with varying concentration.

No run produced a configuration below the first MRRW benchmark. Recovered
objectives differed from MRRW by between $-3.2\times10^{-14}$ and
$-2.1\times10^{-16}$, consistent with floating-point noise, and the recovered
configurations agreed coordinatewise with the tensor-product quasirandom family
to within $10^{-11}$.

\subsection{The tensorization obstruction}
\label{sup:obstruction}

\subsubsection{Setting and hypotheses}

Let $G$ be a normalized configuration on $\mathbb F_2^{\ell}$, that is,
$G(u)\ge0$ and $\sum_uG(u)=1$. Write $\mathsf H(G)$ for its Shannon entropy
in bits, $H_2$ for binary entropy, and
\[
\Phi(G,v)=\sum_u\sqrt{G(u)G(u+v)}
\]
for its translation affinity. By Cauchy--Schwarz,
$0\le\Phi(G,v)\le1$.

For an even integer $m\ge2$ and $v\neq0$, the \emph{paired walk weight} is
\begin{equation}
W_m(G,v)
=\sum_F\frac{m!}{\prod_uF(u)!}\prod_uG(u)^{F(u)},
\label{eq:Wm}
\end{equation}
where the sum runs over $F:\mathbb F_2^{\ell}\to\mathbb Z_{\ge0}$ with
$\sum_uF(u)=m$ and $F(u)=F(u+v)$ for all $u$, and $0^0=1$.
For fixed $G,v,m$, this is the leading coefficient in the configuration
asymptotics of \cite{coregliano2025higher}:
\begin{equation}
A_v^m\Lambda_{n,G}(X)=n^mW_m(G,v)+o(n^m),
\qquad X\in\mathrm{config}_{n,\ell}^{-1}(nG).
\label{eq:cjj-asymptotics}
\end{equation}
Here $A_v$ is the adjacency operator of the column-translation graph and
$\Lambda_{n,G}$ is the indicator of the configuration fibre.
The combinatorial definition \eqref{eq:Wm} applies to every probability
distribution $G$; no limit in $n$ is needed for
Lemmas~\ref{lem:walk}--\ref{lem:jensen} or
Theorem~\ref{thm:obstruction}.

To interpret \eqref{eq:cjj-asymptotics} for a fixed $G$, let
\[
N_G=\{n\in\mathbb N:nG\in\mathrm{Config}_{n,\ell}\}
\]
and assume that $N_G$ is infinite. This holds whenever $G$ has rational
entries. Along $n\to\infty$ in $N_G$, every positive entry satisfies
$nG(u)=\Omega(n)$, as required by the configuration asymptotics.
This lattice condition is needed only for the connection to finite
configuration fibres, not for the entropy inequality below.

Fix $\varepsilon\in(0,1)$. We isolate the following two hypotheses:
\begin{itemize}
\item[\textbf{(A1)}]
\emph{Spanning.} The set $V\subseteq\mathbb F_2^{\ell}\setminus\{0\}$
of tested shifts spans $\mathbb F_2^{\ell}$.
\item[\textbf{(A2)}]
\emph{Finite-walk lower bound.} For every $v\in V$ there exist an even
integer $m_v\ge2$ and a constant $c_v\ge1$ such that
\[
W_{m_v}(G,v)\ge c_v\varepsilon^{m_v}.
\]
The orders $m_v$ and constants $c_v$ are fixed independently of $n$.
\end{itemize}

The spanning condition is equivalent to requiring that, for every
$i\neq0$, some $v\in V$ satisfies $\langle i,v\rangle=1$: failure to span
is equivalent to containment in $i^{\perp}$ for some non-zero $i$.
Condition~(25) of \cite{coregliano2025higher} is a sufficient finite-$n$
condition for the cited spectral construction. When it holds for a fixed
$G$ and a fixed even walk order along infinitely many admissible
blocklengths, \eqref{eq:cjj-asymptotics} implies \textup{(A1)}--\textup{(A2)}
with $c_v=2^{2\ell-1}$. Indeed, there are only finitely many possible choices
of witnessing shifts, so they can be held fixed on an infinite subsequence;
division by $n^m$ and passage to the limit then gives \textup{(A2)}.
We use only $c_v\ge1$. Conversely, \textup{(A2)} alone does not assert
finite-$n$ feasibility of the spectral construction.

\subsubsection{The obstruction}

\begin{lemma}[Walk--affinity comparison]
\label{lem:walk}
For every normalized configuration $G$, every $v\ne0$ and every even
integer $m\ge2$,
\[
W_m(G,v)\le\Phi(G,v)^m.
\]
The inequality is strict whenever $\Phi(G,v)>0$.
\end{lemma}

\begin{proof}
Translation by $v\ne0$ is a fixed-point-free involution of
$\mathbb F_2^{\ell}$, so its orbits are the $2^{\ell-1}$ unordered pairs
$P=\{u,u+v\}$. A multiplicity function $F$ admissible in \eqref{eq:Wm} is
determined by $a_P:=F(u)=F(u+v)$, with $\sum_Pa_P=m/2$.
Writing $x_P=G(u)G(u+v)\ge0$ gives
\begin{equation}
W_m(G,v)
=\sum_{\sum_Pa_P=m/2}
\frac{m!}{\prod_P(a_P!)^2}\prod_Px_P^{a_P}.
\label{eq:Wm-paired}
\end{equation}
Both elements of $P$ contribute $\sqrt{x_P}$ to $\Phi(G,v)$, so
$\Phi(G,v)=2\sum_P\sqrt{x_P}$ and
\begin{equation}
\Phi(G,v)^m
=2^m\!\!\sum_{\sum_Pc_P=m}
\frac{m!}{\prod_Pc_P!}\prod_Px_P^{c_P/2}.
\label{eq:Phi-expansion}
\end{equation}
Every term of \eqref{eq:Phi-expansion} is non-negative. Discard those with
some $c_P$ odd and match the remaining terms to \eqref{eq:Wm-paired} by
$c_P=2a_P$. The ratio of the corresponding coefficients is
\[
\frac{m!\big/\prod_P(a_P!)^2}
     {2^m m!\big/\prod_P(2a_P)!}
=\prod_P\frac{\binom{2a_P}{a_P}}{4^{a_P}}\le1.
\]
In fact this ratio is strictly less than one, since $m\ge2$ implies that
some $a_P\ge1$, and $\binom{2a}{a}<4^a$ for $a\ge1$.
If $\Phi(G,v)>0$, some $x_P>0$; assigning $a_P=m/2$ to this pair and zero
to the others gives a positive matched term. Summing therefore proves
strict inequality in this case. If $\Phi(G,v)=0$, both sides vanish.
\end{proof}

\begin{lemma}[Affinity floor]
\label{lem:floor}
Under \textup{(A2)}, $\Phi(G,v)>\varepsilon$ for every $v\in V$.
In particular, $\Phi(G,v)\ge\varepsilon$.
\end{lemma}

\begin{proof}
Hypothesis \textup{(A2)} gives
$W_{m_v}(G,v)\ge c_v\varepsilon^{m_v}>0$, so $\Phi(G,v)>0$.
The strict part of Lemma~\ref{lem:walk} yields
\[
\Phi(G,v)^{m_v}>W_{m_v}(G,v)
\ge c_v\varepsilon^{m_v}\ge\varepsilon^{m_v}.
\]
Taking $m_v$-th roots proves the claim. No limit in $m$ or $n$ is used.
\end{proof}

\begin{lemma}[Entropy profile]
\label{lem:h}
Let $p_\varepsilon=\tfrac12(1-\sqrt{1-\varepsilon^2})$ and
$h(\varepsilon)=H_2(p_\varepsilon)$ for $\varepsilon\in[0,1]$.
Then $h$ is strictly increasing and strictly convex on $(0,1)$,
$h(0)=0$, $h(1)=1$, and
\begin{equation}
h(\varepsilon)=H_2\!\left(\tfrac12-\sqrt{\delta(1-\delta)}\right),
\qquad \delta=\tfrac{1-\varepsilon}{2}.
\label{eq:h-is-mrrw}
\end{equation}
\end{lemma}

\begin{proof}
Put $s=\sqrt{1-\varepsilon^2}$, so $p_\varepsilon=(1-s)/2$ and
$\delta(1-\delta)=s^2/4$. Thus
$\tfrac12-\sqrt{\delta(1-\delta)}=p_\varepsilon$.
For $0<\varepsilon<1$, differentiation gives
\[
H_2'(p_\varepsilon)=\frac{2\operatorname{artanh}(s)}{\ln2},
\qquad \frac{dp_\varepsilon}{d\varepsilon}=\frac{\varepsilon}{2s},
\]
and hence
\[
h'(\varepsilon)=\frac{\varepsilon\operatorname{artanh}(s)}{s\ln2}>0,
\qquad
h''(\varepsilon)=\frac{\operatorname{artanh}(s)-s}{s^3\ln2}>0.
\]
The second inequality follows from
$\operatorname{artanh}(s)=s+s^3/3+\cdots>s$ for $0<s<1$;
$h''$ extends continuously to $h''(1)=1/(3\ln2)$.
The endpoint values follow from $p_0=0$ and $p_1=1/2$.
\end{proof}

\begin{lemma}[Conditional affinity and Jensen]
\label{lem:jensen}
Let $X=(X_1,\dots,X_\ell)\sim G$ and, for fixed $i$, write $Y=X_{-i}$ and
$p_Y=\Pr(X_i=1\mid Y)$. Then
\begin{equation}
\begin{aligned}
\Phi(G,e_i)
&=\mathbb E_Y\!\left[2\sqrt{p_Y(1-p_Y)}\right],\\
\mathsf H(X_i\mid X_{-i})
&=\mathbb E_Y\!\left[h\!\left(2\sqrt{p_Y(1-p_Y)}\right)\right].
\end{aligned}
\label{eq:cond-identities}
\end{equation}
Consequently,
$\mathsf H(X_i\mid X_{-i})\ge h\big(\Phi(G,e_i)\big)$.
\end{lemma}

\begin{proof}
Identify $u$ with $(y,x_i)$. For each $y$ of positive probability, the two
points $(y,0)$ and $(y,1)$ each contribute
\[
\sqrt{G(y,0)G(y,1)}=\Pr(Y=y)\sqrt{p_y(1-p_y)}
\]
to $\Phi(G,e_i)$, giving the first identity. Values of $p_y$ on
zero-probability events may be chosen arbitrarily.
For $b(p):=2\sqrt{p(1-p)}$, symmetry of binary entropy gives
$H_2(p)=h(b(p))$ for every $p\in[0,1]$. Averaging over $Y$ gives the second
identity. Jensen's inequality and Lemma~\ref{lem:h} give the conclusion.
\end{proof}

\begin{theorem}[Tensorization obstruction]
\label{thm:obstruction}
Fix $\ell\ge1$ and $\varepsilon\in(0,1)$. Let $G$ be a normalized
configuration on $\mathbb F_2^{\ell}$, and suppose that a spanning set
$V\subseteq\mathbb F_2^{\ell}\setminus\{0\}$ satisfies
$\Phi(G,v)\ge\varepsilon$ for every $v\in V$. Then
\[
J_\ell(G)=\frac{\mathsf H(G)}{\ell}\ge h(\varepsilon)
=H_2\!\left(\tfrac12-\sqrt{\delta(1-\delta)}\right),
\qquad \delta=\tfrac{1-\varepsilon}{2}.
\]
The right-hand side is the first MRRW expression.
This bound is sharp for the relaxed problem with constraints
$\Phi(G,e_i)\ge\varepsilon$, $1\le i\le\ell$:
$G_\varepsilon=\operatorname{Bernoulli}(p_\varepsilon)^{\otimes\ell}$
attains equality in that problem.
Under the stronger hypotheses \textup{(A1)}--\textup{(A2)}, one has
$J_\ell(G)>h(\varepsilon)$.
\end{theorem}

\begin{proof}
Choose a basis $v_1,\dots,v_\ell$ contained in $V$ and an invertible linear
map $T$ with $Te_i=v_i$. Define $\widetilde G(u)=G(Tu)$.
Then
\[
\mathsf H(\widetilde G)=\mathsf H(G),
\qquad \Phi(\widetilde G,e_i)=\Phi(G,v_i)\ge\varepsilon.
\]
Let $X\sim\widetilde G$. The entropy chain rule and conditioning give
\[
\mathsf H(G)
=\sum_{i=1}^{\ell}\mathsf H(X_i\mid X_1,\dots,X_{i-1})
\ge\sum_{i=1}^{\ell}\mathsf H(X_i\mid X_{-i})
\ge\sum_{i=1}^{\ell}h\big(\Phi(\widetilde G,e_i)\big)
\ge\ell h(\varepsilon),
\]
where the second inequality uses Lemma~\ref{lem:jensen} and the last uses
monotonicity of $h$. Division by $\ell$ proves the bound.

For sharpness of the relaxed problem, affinity is multiplicative over
tensor products, so
\[
\Phi(G_\varepsilon,e_i)
=2\sqrt{p_\varepsilon(1-p_\varepsilon)}=\varepsilon,
\qquad
\mathsf H(G_\varepsilon)=\ell H_2(p_\varepsilon)=\ell h(\varepsilon).
\]
This proves attainment in the stated relaxation, not under
\textup{(A2)}.

Finally, under \textup{(A1)}--\textup{(A2)}, Lemma~\ref{lem:floor} gives
$\Phi(G,v_i)>\varepsilon$ for every basis shift. Strict monotonicity of
$h$ makes the last inequality in the entropy chain strict, yielding
$J_\ell(G)>h(\varepsilon)$.
\end{proof}

\begin{corollary}
\label{cor:no-improvement}
Consider the entropy-based rate estimate
\[
R(C)\le J_\ell(G)+O(\log n/n)
\]
for the fixed single-orbit spectral construction
(Section~\ref{sup:normalization}), with $\ell$ and $G$ fixed and with its
feasibility hypotheses satisfied. If $G$ satisfies
\textup{(A1)}--\textup{(A2)}, then its leading term obeys
$J_\ell(G)>h(\varepsilon)$.
Consequently, minimizing this entropy-based estimate over configurations
satisfying \textup{(A1)}--\textup{(A2)} cannot give an asymptotic rate bound
below the first MRRW expression.
The relaxed affinity-constrained minimum equals $h(\varepsilon)$, but this
does not assert attainment by a configuration satisfying the finite-walk
conditions.
\end{corollary}

\begin{proof}
The rate normalization follows because a level-$\ell$ dual bounds
$|C|^\ell$. Apply Theorem~\ref{thm:obstruction} to the leading term
$J_\ell(G)$ of the stated estimate.
\end{proof}

\subsubsection{Scope and relation to the cited construction}

\begin{remark}[Scope]
\label{rem:scope}
The obstruction concerns the entropy objective $J_\ell(G)$ under a
spanning set of affinity floors, and therefore under
\textup{(A1)}--\textup{(A2)}. It does not rule out every single-orbit
spectral construction. In particular, it does not exclude feasibility
mechanisms that bypass \textup{(A2)}, sharper estimates of a dual objective
that are not controlled by this entropy bound, alternative sign functions
$\phi$, spectral functions supported on several configuration orbits,
other uses of the partial-Fourier dual families, or general dual feasible
solutions. The higher-order hierarchy itself remains complete.
\end{remark}

\begin{remark}[Sufficiency of the walk condition]
\label{rem:fourier}
Condition~(25) of \cite{coregliano2025higher} is a sufficient finite-$n$
condition obtained by making the terms of an odd-cardinality expansion
of $M_m$ separately non-negative. Hypothesis \textup{(A2)} is its
asymptotic consequence for fixed walk orders, as explained above; it is
not an equivalent formulation of finite-$n$ feasibility. Configurations
that fail \textup{(A2)} could in principle still give non-negative
Fourier transforms through cancellation in that expansion.
The present obstruction does not address such configurations.
\end{remark}

\begin{remark}[Finite-$m$ admissibility]
\label{rem:finitem}
The equality configuration $G_\varepsilon$ is feasible for the
affinity relaxation, but not for \textup{(A2)}. Indeed, multiplicativity
gives
\[
\Phi(G_\varepsilon,v)=\varepsilon^{|v|}\le\varepsilon
\qquad(v\ne0),
\]
where $|v|$ is Hamming weight. Lemma~\ref{lem:walk} therefore gives,
for every even $m\ge2$,
\[
W_m(G_\varepsilon,v)<\Phi(G_\varepsilon,v)^m\le\varepsilon^m.
\]
This fails \textup{(A2)} even with $c_v=1$.
The theorem asserts sharpness only of the affinity relaxation.
Any claim that a sequence of feasible finite-walk constructions approaches
this value requires a separate approximation argument; no such claim is
needed for Corollary~\ref{cor:no-improvement}.
\end{remark}

\begin{remark}[Equality in the relaxation]
\label{rem:equality}
Work in coordinates determined by a basis of the tested shifts.
If equality holds in the relaxed entropy bound, then every step of the
entropy chain in the proof must be an equality. Strict convexity of $h$
and strict monotonicity at the affinity floor imply
$2\sqrt{p_Y(1-p_Y)}=\varepsilon$ almost surely for every coordinate.
Thus the conditional marginals lie in
$\{p_\varepsilon,1-p_\varepsilon\}$, and equality must also hold in the
conditioning step. The tensor-product configuration explicitly satisfies
these equalities and explains the value recovered in the relaxed numerical
search. We make no uniqueness claim, and these equality statements do not
imply finite-walk admissibility.
\end{remark}

\begin{remark}[Relation to the hypotheses of \cite{coregliano2025higher}]
\label{rem:cjj-hypotheses}
The general sufficient criterion in Theorem~6.7 of
\cite{coregliano2025higher} requires a witnessing shift for each non-zero
$i$ with $\langle i,v\rangle=1$; it does not restrict all witnessing
shifts to Hamming weight one. Their explicit configuration examples
verify the walk estimates at the standard basis shifts, which already
span $\mathbb F_2^\ell$.
The change of basis above accommodates the general spanning formulation.
Their stated finite-$n$ criterion also assumes strictly positive
configuration entries, whereas the entropy inequality and walk--affinity
comparison hold for all $G\ge0$. Hence they also apply to configurations
with proper support, provided the stated affinity or walk hypotheses
hold. This extension is an entropy obstruction, not by itself a proof of
feasibility for configurations with zero entries in the cited spectral
construction.
\end{remark}

\subsection{Computational verification of the algebraic reductions}
\label{sup:verification}

None of the computations in this section is used as part of the analytic
proof. They test that the implementation computes the objects
Theorem~\ref{thm:obstruction} concerns, which is where a hidden factor of
$m!$, a $2^m$, or an ordered-versus-unordered pair convention would appear.

The literal sum \eqref{eq:Wm} over multiplicity functions was enumerated in
exact rational arithmetic and matched the paired form
\eqref{eq:Wm-paired} as an exact rational identity, not to a tolerance, for
all tested $\ell\in\{1,2,3\}$, all $v\ne0$ and all even $m\le8$.

All downstream computation takes place on the configuration lattice, whereas
$A_v$ acts on $\mathbb F_2^{\ell\times n}$. To confirm that this reduction
loses nothing, $A_v$ was constructed literally on the full vertex set and
applied as a matrix, rather than assuming that $A_v^m\Lambda(X)$ depends on $X$
only through $\operatorname{config}(X)$. In every tested instance the result was
constant on the entire fibre $\operatorname{config}^{-1}_{n,\ell}(g_0)$ and
agreed exactly with an independent configuration-lattice walk implementation
(Supplementary Table~\ref{tab:matrixcheck}). Blocklengths are necessarily
small, the largest cases being $4^5=1024$ and $8^4=4096$ matrices; all
$v\ne0$ were tested and all agree, with $v=e_1$ shown for legibility.

\renewcommand{\tablename}{Supplementary Table}

\begin{table}[h]
\centering
\begin{tabular}{ccl cc rr c}
\hline
$\ell$ & $n$ & $g_0$ & $v$ & $m$ & matrix values & config.\ walk & agree \\
\hline
2 & 4 & $(1,1,1,1)$ & $e_1$ & 2 & $\{8\}$   & 8   & yes \\
2 & 4 & $(1,1,1,1)$ & $e_1$ & 4 & $\{128\}$ & 128 & yes \\
2 & 4 & $(2,1,1,0)$ & $e_1$ & 2 & $\{8\}$   & 8   & yes \\
2 & 4 & $(2,1,1,0)$ & $e_1$ & 4 & $\{104\}$ & 104 & yes \\
2 & 5 & $(2,1,1,1)$ & $e_1$ & 2 & $\{11\}$  & 11  & yes \\
2 & 5 & $(2,1,1,1)$ & $e_1$ & 4 & $\{245\}$ & 245 & yes \\
3 & 4 & $(1,1,1,1,0,0,0,0)$ & $e_1$ & 2 & $\{8\}$   & 8   & yes \\
3 & 4 & $(1,1,1,1,0,0,0,0)$ & $e_1$ & 4 & $\{128\}$ & 128 & yes \\
3 & 4 & $(2,1,1,0,0,0,0,0)$ & $e_1$ & 2 & $\{8\}$   & 8   & yes \\
3 & 4 & $(2,1,1,0,0,0,0,0)$ & $e_1$ & 4 & $\{104\}$ & 104 & yes \\
\hline
\end{tabular}
\caption{The operator $A_v^m\Lambda$ computed literally on
$\mathbb F_2^{\ell\times n}$, against the configuration-lattice walk count.
``Matrix values'' is the set of values taken by $A_v^m\Lambda(X)$ over the
entire fibre $\operatorname{config}^{-1}_{n,\ell}(g_0)$; that this set is a
singleton in every row is precisely the assertion that the quantity depends on
$X$ only through $\operatorname{config}(X)$. All entries are exact integers.}
\label{tab:matrixcheck}
\end{table}

The finite-$n$ walk count was compared with the leading term $n^mW_m(G,v)$ of
\eqref{eq:cjj-asymptotics}. The ratios converge to $1$, and do so
\emph{from above}, so the finite-$n$ correction is positive:
\[
\begin{array}{lccccc}
\ell=2,\ m=2: & n=8    & 16     & 32     & 64     & 128    \\
              & 1.6667 & 1.3333 & 1.1667 & 1.0833 & 1.0417 \\[3pt]
\ell=3,\ m=2: & n=12   & 24     & 48     & 96     & 192    \\
              & 1.7500 & 1.3750 & 1.1875 & 1.0938 & 1.0469
\end{array}
\]
with the same behaviour at $m=4$ (for $\ell=3$, $3.1067\rightarrow1.0998$ over
the same sequence). The sign of this correction is not used anywhere in the
argument, which requires only the $o(n^m)$ bound, but it is recorded because the
direction is not obvious a priori.

The inequality of Lemma~\ref{lem:walk} was checked on $480$ exact rational
cases, with largest observed ratio $0.4957$. The extremal value $\tfrac12$ is
attained already at $\ell=1$, $m=2$, where there is a single pair and a single
admissible multiplicity, so the coefficient comparison reduces exactly to
$\binom21/4$.

The entropy inequalities carry complete proofs, so the remaining checks are
consistency tests rather than evidence. In $60$-digit arithmetic, the closed
forms of Lemma~\ref{lem:h} agreed with numerical differentiation to
$10^{-57}$; the entropy chain of Theorem~\ref{thm:obstruction} held on $1600$
random distributions across $\ell=1,\dots,4$ with minimum gap
$-5.6\times10^{-61}$, i.e.\ zero to working precision; and the
conditional-affinity identity \eqref{eq:cond-identities} held on $800$ further
distributions with maximum deviation $4.6\times10^{-61}$.

%% file: Section3/04supplementary-methods.tex
\section*{Supplementary Methods: sign-uncertainty discovery and diagnostics}

\subsection*{Rigorous tripwires}

Two lower bounds are evaluated at run time and used to invalidate the
numerics rather than to prove anything. For $d=1$, $s=+1$,
\[
A_{+}(1)\ge \frac{1}{2(1+\lambda_{\mathrm{BCK}})}
=0.4107675\ldots,
\]
with $\lambda_{\mathrm{BCK}}=-\min_x\sin x/x$ \cite{bourgain2010principe}. More generally,
Theorem~2.1 of \cite{cohn2022sign} gives
\[
u_0\ge\lambda(d,n),
\]
where $\lambda(d,n)$ is the smallest root of
$L^{(d/2-1)}_{\lfloor n/2\rfloor+1}$. The latter follows from Gauss
quadrature against $e^{-u}u^{d/2-1}\,du$: positivity of the quadrature
weights forces a sign change at or beyond the first node. Any run reporting
feasibility below either applicable bound is therefore rejected. During
development, an early grid-only relaxation returned
$\rho=0.398942=1/\sqrt{2\pi}$ in $d=1$, immediately falsified by the BCK
tripwire.

\subsection*{Discovery campaign and numerical diagnosis}

The initial search deliberately explored exact Fourier-eigenfunction mixtures
with trainable Gaussian scales and contact locations, using local optimization
and Adam-based learned proposals. Low-rank constructions reproduced the
expected Bourgain--Clozel--Kahane controls and rapidly improved the objective,
but the $d=1$ search plateaued near $\rho=0.5785$, above the published
Cohn--Gon\c{c}alves value.

The failure was diagnostic. When scales crowded or the imposed contact count
was incompatible with the local root structure, the coefficient solve became
extremely ill-conditioned. Candidate Adam trajectories that appeared to
improve the objective entered systems with condition numbers of order
$10^{17}$--$10^{18}$. At this scale, forward accuracy could not be guaranteed
in the floating-point formats used here. Multiprecision reevaluation showed
that the apparent descent was dominated by numerical error rather than a
reproducible improvement. Learned widths and contacts were therefore retained
only as proposals, with a hard conditioning guard before verification.

The experiments also showed that the number of double contacts should not be
rigidly tied to the number of basis functions. Decoupling the two made
comparisons across basis sizes meaningful and showed that additional widths
at fixed contact count contributed little. Multistart runs converged to
distinct plateaus near $0.5785$, $0.5929$ and higher values, and verified
evaluations could jump when the outermost root pattern changed. This behaviour
is consistent with the discontinuity mechanism reported by Cohn and
Gon\c{c}alves for Newton continuation at $d\le2$. We use this only as a
diagnostic, not as an optimality theorem.

\subsection*{Implementation details of the convex exchange solver}

Three implementation choices proved load-bearing. First, bare LP feasibility
with zero objective returns arbitrary vertices that hug zero at many grid
points and can dip between them. Maximizing a relative margin instead produces
strict positivity between most nodes and sharply reduces the number of
exchange rounds. Second, column scaling by the supremum over the full grid is
dominated by the far tail and can exclude genuine feasible solutions; rows are
therefore equilibrated while columns are left unscaled. Third, multiple far
tail anchors are nearly parallel after normalization and degrade the simplex.
The tail is imposed instead through a single structural sign constraint on the
leading coefficient.

At each bisection value $u_0$, the LP is solved on the current grid, the
candidate derivative $P'$ is evaluated at high precision, and all real
critical points on the active interval are located. Any negative critical
value is inserted as a new cutting row. The process stops only after the
candidate passes the independent continuous scan to the prescribed numerical
tolerance.

\subsection*{Convergence status of the exchange loop}

The oracle is Remez-like in mechanics but not in guarantee. The truncated
Laguerre space $V_N$ is not a Haar system on $(0,\infty)$:
$L^{(\alpha)}_{2N}(2u)$ belongs to $V_N$ and has $2N$ positive roots, exceeding
the Haar budget of an $(N+1)$-dimensional space. Consequently, uniqueness of
the contact configuration and classical equioscillation arguments are not
available.

The relevant framework is convex semi-infinite programming. Every finite grid
gives an outer relaxation. Under compactness and a Slater point, the exchange
sequence drives the maximum constraint violation to zero and the corresponding
margins converge to the semi-infinite value. Together with monotonicity and
continuity in $u_0$, this gives a numerical bracket for the fixed-degree
boundary. This concerns optimality within the chosen finite space only.
Validity of every quoted upper bound is established by the exact rational
certificate and does not depend on convergence of the exchange loop.

\subsection*{Contact degeneracy and leave-one-out reconstruction}

The active contact count is discovered at the feasibility boundary. The
certified $d=1$ degree-$22$, $30$ and $38$ constructions use
$m=5$, $7$ and $9$ double contacts, respectively. The non-Haar structure can
make the active set non-unique. At degree $38$, the exchange oracle reports
ten near-active contacts although the square reconstruction uses nine.
Collocating all ten produces an admissible but irrelevant function whose
outer sign change moves to $u=274$ ($\rho=9.34$).

We therefore perform leave-one-out reconstruction: each candidate contact is
removed in turn, the remaining contacts are collocated in the square Laguerre
system, and only reconstructions whose exact certificate agrees with the
convex estimate are retained. At degree $38$, four of the ten subsets certify
the same $0.5725887$ value to nine digits. The same repair is required at
$d=2$. The multiplicity of successful subsets is consistent with a degenerate
optimal facet and is not interpreted as uniqueness of the extremizer.

\subsection*{Exact Sturm implementation}

After rational reconstruction,
\[
P(u)=\prod_i(u-\tau_i)^2R(u),
\]
and only the sign of $R$ on the ray remains to be certified. We build a
primitive pseudo-remainder sequence over $\mathbb Z$, using only positive
rescalings so that Sturm sign-variation counts are preserved. This controls
coefficient growth substantially better than naive rational Euclidean
remainders, whose bit lengths become impractical by degree $38$.

The verifier checks the exact Laguerre-to-monomial expansion, zero remainder
after division by every squared contact factor, the sign of the leading
coefficient, the exact value $R(\bar u)>0$, and equality of Sturm
sign-variation counts at $\bar u$ and $+\infty$. No grid, floating-point root
finder or interval sampling appears in the final admissibility decision.

\subsection*{Relation to the non-Gaussian search}

A numerical A/B analysis explains why the earlier Gaussian-mixture search did
not improve the final polynomial construction at small $d$. Both spaces are
evaluated through the same convex machinery, so the comparison concerns
function spaces rather than optimizer quality. Projecting individual mixture
directions onto the degree-$22$ polynomial space over the active region gives
relative sup-norm residuals of approximately
\[
1.3\times10^{-15}\ (a=1.5),\quad
3.0\times10^{-10}\ (a=2.5),\quad
9.2\times10^{-7}\ (a=4),
\]
\[
6.3\times10^{-5}\ (a=6),\qquad
1.6\times10^{-3}\ (a=10).
\]
Wider directions are genuinely distinct, but their optimized coefficients
are of order $10^{-3}$ and did not improve the objective beyond the numerical
noise floor. This containment analysis is diagnostic only and is not used in
the certificate.

\subsection*{Higher-dimensional asymptotic diagnostic}

Using the same convex machinery, the scale-free quantity
$\rho\sqrt{2\pi/d}$ along one degree--dimension path was
\[
1.435267\ (d{=}1),\quad
1.340341\ (d{=}2),\quad
1.212474\ (d{=}4),\quad
1.079183\ (d{=}8),\quad
1.026193\ (d{=}12),\quad
1.005480\ (d{=}16).
\]
The sequence approaches the sublinear-degree asymptote $1$ rather than
crossing it. This is only a diagnostic path, not an asymptotic theorem. At
$d=8$, increasing degree from $14$ to $38$ changes the reported value
non-monotonically over a range of approximately $5.6\times10^{-4}$ even
though the exact fixed-degree optima are nested. We therefore interpret the
spread as a numerical resolution limit of the discovery oracle.

\subsection*{Degree ceiling of the numerical stage}

The current floating-point exchange solver is reliable only to polynomial
degree approximately $38$. Beyond this, the raw polynomial values span an
extreme dynamic range: at degree $54$, $P$ reaches about $9.7\times10^{73}$ at
the far end of its own grid. Row equilibration cannot repair the resulting
column-direction scaling, and the LP begins returning points that violate its
own constraints.

Near the ceiling, the finite grid can also acquire a recession direction
that is non-negative on all nodes while dipping between them, producing an
``unbounded'' rather than infeasible LP, and the cut pool can accumulate
near-duplicate rows across the bisection. We cap the margin variable, refresh
the cut pool when necessary, and discard runs whose bisection collides with a
rigorous lower tripwire. Multiplication by the positive factor $e^{-u}$ reduces
the dynamic range but relocates the failure to underflow in the far tail.
Higher-dimensional asymptotic work will therefore require higher precision or
iterative/exact refinement inside the optimizer. The certificates reported
here are unaffected because all lie at degree at most $38$ and are validated
by the exact rational stage.

%% file: 06supplementary-compute-package.tex
\renewcommand{\tablename}{Supplementary Table}
\renewcommand{\figurename}{Supplementary Figure}

\section{NeuralCert structure}

NeuralCert is organized around explicit interfaces and data boundaries rather than a single shared numerical backend. A problem-specific discovery module produces an explicit candidate representation together with the metadata required to interpret it. A separate certification module reconstructs and evaluates that candidate using rigorous arithmetic and emits a self-contained certificate. An independent verifier consumes only this certificate and recomputes the claimed result without relying on the discovery or certification implementation. Discovery methods can therefore be modified or replaced without changing the trust assumptions underlying certificates that have already been issued.
The three applications considered here share this discovery–certification–verification architecture, but not their numerical representations or proof mechanisms. For the Maynard problem, finite-\(k\) polynomial candidates are certified through exact rational evaluation of the associated Gram quadratic forms, whereas large-\(k\) rational candidates are certified through Fourier-domain evaluation in Arb ball arithmetic. Delsarte bounds are certified by exact rational verification of dual feasibility, while sign-uncertainty candidates are reduced to rational polynomial factorizations and verified using exact Sturm-sequence arguments.

%
%


\begin{figure}[tbp]
  \centering
  \begin{minipage}{0.96\linewidth}
  \small
  \begin{verbatim}
NeuralCert repository
|-- neuralcert/                  reusable framework
|   |-- core/                   problem and result contracts
|   |-- discovery/              neural models and optimization
|   |-- distill/                structural compression
|   |-- refine/                 deterministic local refinement
|   |-- exact/                  reusable exact arithmetic
|   |-- verify/                 independent verification primitives
|   |-- data/                   reproducible datasets
|   |-- pipeline.py             typed stage orchestration
|   `-- problems/
|       |-- maynard/            Maynard plugin and adapters
|       |-- delsarte/           coding-theory LP plugin
|       `-- sign_uncertainty/   Gaussian/Laguerre/Sturm plugin
`-- maynard_tools/              specialized compatibility layer
    |-- discovery/              factored, gated, ratio
    |-- certification/          Karatsuba, CRT, FLINT, Arb
    `-- verifier/               independent Maynard checks
  \end{verbatim}
  \end{minipage}
  \caption{Compact directory view of the NeuralCert package architecture.}
  \label{fig:neuralcert-package-tree}
\end{figure}

\section{Computational performance}

The variational problem shards along the pair axis: for a pair $(j, \ell)$ the entire convolution chain touches only channels $j$ and $\ell$ and produces two scalars, so the compute-to-communication ratio permits distribution over commodity interconnect. Because the Hellmann–Feynman route detaches the Ritz vector, both quadratic forms are plain sums over pairs, and the exact gradient follows from two scalar all-reductions and one gradient all-reduction per iteration; no autograd-aware collectives are required. Memory is dominated by activations indexed by pair, growing as $m(m + 1)/2$, so parameter-sharding schemes (FSDP, ZeRO) address the wrong resource; DDP is also inapplicable, since it averages gradients while the reduction here is a sum inside a nonlinear ratio.
Memory is not dominated by the assembled Gram matrices ($m \times m$), but by the activations of the convolution chain, which are indexed by pair and scale as $N \cdot n_q \cdot m(m+1)/2$ per chain step. Because Hellmann–Feynman  requires these intermediates for the backward pass, they are retained across all $\sim 2 \log_2 k $ steps of the doubling chain. Supplementary Table \ref{tab:computation} shows GPU memory use with different settings for the Maynard-Tao polynomial fitting discovery procedure.

\textbf{Certification cost:} The CRT backend at k=3000, degree 60 needs ~155,000 primes; ball arithmetic replaces it with one adaptive-precision evaluation.  The resolution requirement as a scaling law. $n_{rep}$ grows roughly as $\approx k^{0.78}$ to hold the $g \equiv x$ control at $10^{-3}$.

\begin{table}[H]
\centering
\caption{Overview of memory use and runtime for different settings and options for the Maynard-Tao \textbf{Polynomial discovery} regime since other implementations in the package have seconds runtime while using MegaBytes.}
\label{tab:computation}
\begin{tabular}{lllr}
\toprule
k       & Channels  & $N_{grid}$ & VRAM \\
\midrule

100     &  32       & 3000  & 36    \\
100     &  48       & 3000  & 77    \\
100     &  64       & 3000  & 133   \\
100     &  32       & 4000  & 48    \\
100     &  48       & 4000  & 102   \\
100     &  64       & 4000  & 110   \\

500     &  32       & 3000     & 53    \\
500     &  48       & 3000     & 110   \\
500     &  64       & 3000     & 196 \\
500     &  32       & 4000     & 70    \\
500     &  48       & 4000     & 84    \\
500     &  64       & 4000     & 260   \\

1000    & 32        & 5000     & 66 \\
1000    & 48        & 5000     & 218 \\
1000    & 64        & 5000     & 226 \\
1000    & 64        & 6000     & 274 \\
1000    & 48        & 4000     & 171 \\

3000    & 32        & 5000     & 121 \\
3000    & 48        & 2500     & 253 \\

3600    & 10        & 5000     & 22  \\

5000    & 10        & 5000     & 22  \\
5000    & 20        & 5000     & 52  \\
5000    & 30        & 5000     & 102  \\
5000    & 40        & 5000     & 172  \\
5000    & 50        & 5000     & 260  \\

\bottomrule
\end{tabular}

\par\smallskip
\parbox{0.95\linewidth}{%
\footnotesize
\textit{Note.} Heavy compute analyses were performed using multiple NVIDIA H100 NVL or NVIDIA H200 SXM GPU cards with respectively 96 or 141 GB VRAM. GPU-consumption = GB VRAM at it's peak. $N_{grid}$ = Representation nodes $N$; 
The Chebyshev–Lobatto nodes per channel which is the resolution of a one-dimensional function, and the $k$-dimensional integral is handled by the separable structure; memory scales as $N \cdot 3 m(m+1)/2$. This is the vanilla setup ($\epsilon = 0$). 
Runtime in $k=100, N_{grid} = 3000, \text{Channels} = 32$ is 1 hour and 14 minutes for 4000 Adam iterations and 8 minutes for 20 steps LBFGS.
}
\end{table}

%% file: Section1/supplementary_certificate_example.tex
\renewcommand{\tablename}{Supplementary Table}
\renewcommand{\figurename}{Supplementary Figure}

\section{Anatomy of a certificate}
\label{sup:certificate-anatomy}

A certificate is a compact, exactly specified record from which a claimed
inequality can be independently recomputed. It is expensive to produce
and cheap to check: verification requires neither the discovery pipeline
nor the hardware used to generate the candidate, and a reader who
distrusts every stage of the numerical search can nonetheless confirm the
stated bound by running a short standalone program. The asymmetry is
deliberate: generating a competitive trial function requires the full
optimisation and certification stack, whereas checking the resulting
inequality requires only the exact trial parameters and a few seconds of
interval arithmetic on commodity hardware.

This note makes that concrete by walking through one complete example at
$k=5$ --- the smallest case, chosen so that every artefact fits on the
page with nothing elided. All files shown are shipped verbatim in
\texttt{demo/} and are the unmodified output of the two commands in
Panels~A and~C.

\subsection*{Panel A --- The claim}

The object being certified is one explicit function, fixed in advance:
\begin{equation}
F(t_1,\dots,t_5)=\prod_{i=1}^{5}g(t_i)\,\mathbf 1_{\mathcal R_5}(t),
\qquad
g(t)=\frac{1}{c+4t},
\qquad
c=\frac{552637315007}{978860977902}.
\label{eq:demo-trial}
\end{equation}
Because $M_k$ is a supremum over admissible trial functions, \emph{any}
such $F$ yields a lower bound, and the certified claim is
\[
M_5\ \ge\ 1.9717711764 .
\]
Three properties of this statement are worth making explicit before the
mechanics, because they are what the certificate does and does not
assert.

\begin{enumerate}
\item It is a statement about $F$, not about the optimiser. The value of
$c$ was found by numerical search; the search is not part of the claim,
and a different (or worse) $c$ would give a different (or worse) valid
bound.
\item It is a lower bound and is deliberately not sharp. The discovery
stage reported $R=1.9966686079$ for this same $F$; directed rounding of
the aliasing and truncation budgets costs about $2.5\times10^{-2}$, and
the certified number is the conservative end of that interval.
\item It is weaker than the best published constructions at $k=5$, and
this is the expected behaviour of an honest lower-bound machine. The
rank-one rational family is not the optimal family here: Maynard's
explicit polynomial trial gives $M_5\ge 1417255/708216=2.00116\ldots$,
and Bogaert's Krylov computation gives $2.0071451444$. The gap matters
mathematically, since under the Elliott--Halberstam conjecture the
threshold $M_5>2$ implies $\mathrm{DHL}[5,2]$; the certified rank-one
value does not reach it, whereas the polynomial trial does. At
$k=5$ the certificate is therefore a demonstration of the machinery
rather than a source of number-theoretic content.
\end{enumerate}

\subsection*{Panel B --- The two files}

The pipeline writes two records, and the separation between them is the
trust boundary.

\panel{B1. The discovery record (\texttt{k5-eps0-ratio.npz}).}
Produced by the search stage. It carries the trial function in exact
rational form and nothing else of mathematical weight; the floating-point
fields are provenance, not evidence.

\begin{lstlisting}[style=certfile]
k             : 5
epsilon       : 0/1                    (vanilla functional)
mu            : [1]
c_num, c_den  : ['552637315007'], ['978860977902']
power         : [1]                    g(t) = (c + n t)^{-1}
w_num, w_den  : ['1'], ['1']           mixing weight (rank one)
R_discovery   : 1.99666861             float, not certified
ceiling       : 2.01179739             k/(k-1) log k
canonical     : k=5|552637315007/978860977902^-1*1
sha256        : dfbd528eb3900e18771468e642b8cbcde4b2ed9074a593701419a856e2c596e6
\end{lstlisting}

\noindent
The \texttt{canonical} string is the complete mathematical content of the
file: dimension, exact pole, power, weight. Everything the certifier
needs is recoverable from that one line, and the \texttt{sha256} binds
the certificate below to this exact trial function.

\panel{B2. The certificate (\texttt{k5-eps0-ratio.json}).}
Produced by the certification stage and reproduced here in full.

\begin{lstlisting}[style=certfile]
{
  "format": "maynard-Mk-certificate/2",
  "trial_function": {
    "g": "g(t) = 1/(c + n t),  n = k-1",
    "form": "F = prod_{i=1..k} g(t_i) on the simplex",
    "c_exact": "552637315007/978860977902",
    "u_exact": "1",
    "note": "R is invariant under scaling of the weight; any fixed weight
             yields a valid lower bound."
  },
  "parameters": {
    "P": 8.0,              "prec": 200,
    "nnode": 153,          "M0": 153,
    "rmax": 12,            "npanels": 80,
    "A_upper": 16.170030887466964,
    "Theta_ball": 32.34006177493393,
    "theta_far": 120.0,    "block_ratio": 2.0,
    "tol_bits": 45,        "assembly": "directed-arb/2"
  },
  "certified_quantities": {
    "N_enclosure": { "lower": 0.0877189203290238,  "upper": 0.0877189203290238  },
    "D_enclosure": { "lower": 0.2196630290925637,  "upper": 0.21966302909256372 }
  },
  "error_terms": {
    "N_band": 0.0, "N_tail": 0.0005786544597255379, "N_alias": 3.51115912146848e-07,
    "D_band": 0.0, "D_tail": 0.0013051009573133488, "D_alias": 4.99060886139429e-07,
    "P_S_gt": 1.2860941636237798e-06, "y": 7.0
  },
  "far_field_blocks": [],
  "final_arithmetic": {
    "N_lower_used": 0.08713991475338612,
    "D_upper_used": 0.22096862911076323,
    "R_lower":      1.9717711763896169,
    "R_midpoint":   1.9966700971800757,
    "ceiling":      2.0117973905426254
  },
  "claim": {
    "k": 5,
    "statement": "M_5 >= 1.971771176",
    "M_k_lower_bound": 1.971771176,
    "known_upper_bound_ceiling": 2.0117973905426254,
    "ceiling_source": "M_k < k/(k-1) log k  (Polymath8b)",
    "primes_criterion": "r_k = ceil(theta M_k / 2), theta = 1/2 - eps
                         (Bombieri-Vinogradov; Maynard Prop. 4.2)",
    "primes_implied": 1
  },
  "assembly_note": "All error bounds assembled in Arb ball arithmetic; floats
                    exit only through provably outward-rounded conversions
                    (nudge-verified against the ball).",
  "source": {
    "npz": "./k5-eps0-ratio.npz",
    "sha256": "dfbd528eb3900e18771468e642b8cbcde4b2ed9074a593701419a856e2c596e6",
    "canonical": "k=5|552637315007/978860977902^-1*1",
    "R_discovery": 1.9966686078849576,
    "u_exact": "1"
  },
  "self_hash_sha256": "7f12647dfe9d7c9e02fc1f636eb8aec4ae05ac359675a0412cdbb551a328ee20"
}
\end{lstlisting}

\noindent
That the file fits on one page is itself the point: there is nowhere for
an unstated assumption to hide. The fields divide into three kinds, and
Panel~C explains why the distinction matters.

\begin{center}
\begin{tabular}{@{}lll@{}}
\toprule
Kind & Fields & Status \\
\midrule
Mathematical input & \texttt{c\_exact}, \texttt{u\_exact}, \texttt{k}
  & load-bearing; exact rationals \\
Computational settings & \texttt{P}, \texttt{nnode}, \texttt{rmax},
  \texttt{prec}, \dots & affect sharpness, not validity \\
Reported intermediates & \texttt{*\_enclosure}, \texttt{error\_terms},
  \texttt{final\_arithmetic} & recomputed, never trusted \\
\bottomrule
\end{tabular}
\end{center}

\subsection*{Panel C --- Verification}

\panel{C1. The transcript.}
The certifier reads the discovery record and writes the certificate:

\begin{lstlisting}[style=certfile]
$ neuralcert certify --method ratio --npz k5-eps0-ratio.npz \
                     --cert-json k5-eps0-ratio.json

trial function : k=5|552637315007/978860977902^-1*1
  sha256       = dfbd528eb3900e18...  (verified)
  k            = 5   R_discovery = 1.9966686079
  A = 16.17   Theta_ball = 32.3401   theta_far = 120   nodes = 307   far blocks = 0
  aliasing: P(S > 7) <= 1.286e-06   (Markov, r <= 12)
  N in [0.08771892032902380 +/- 6.19e-18]   err(band,tail,alias) = (0.00e+00, 5.79e-04, 3.51e-07)
  D in [0.2196630290925637 +/- 2.42e-17]    err = (0.00e+00, 1.31e-03, 4.99e-07)
  R midpoint   = 1.9966700972   |mid - discovery| = 1.49e-06

  CERTIFIED (directed):  M_5 >= 1.9717711764   [0s]
  => at least 1 primes infinitely often (theta = 1/2 - eps)
\end{lstlisting}

\panel{C2. What is recomputed, and what is ignored.}
Every quantity below is regenerated by the verifier from
\texttt{c\_exact} and the computational settings alone. None of the
stored values is read as input; they are recomputed and compared.

\begin{center}
\begin{tabular}{@{}lll@{}}
\toprule
Quantity & Reconstructed from & Value \\
\midrule
$A=2(c+n)/c$                    & $c$                  & $16.170030887466964$ \\
$\Theta_{\mathrm{ball}}=2A$     & $A$                  & $32.34006177493393$ \\
$\Delta\theta=2\pi/P$           & $P=8$                & $0.7853981634$ \\
$\lfloor\Theta_{\mathrm{ball}}/\Delta\theta\rfloor$ & $A,P$ & $41$ \\
$M_0=\max(41,\,n_{\mathrm{near}})$ & above, $n_{\mathrm{near}}=153$ & $153$ \\
Band error                      & $M_0=n_{\mathrm{near}}$ & \emph{empty} \\
$N_{\mathrm{tail}}$, $D_{\mathrm{tail}}$ & Lemma (E2) with $A,\Delta\theta,M_0$
  & $5.7865\times10^{-4}$, $1.3051\times10^{-3}$ \\
$\mathbb P(S>7)$                & certified moments, $r\le 12$ & $1.286\times10^{-6}$ \\
$R_{\mathrm{low}}=kN_{\mathrm{low}}/D_{\mathrm{up}}$ & above & $1.9717711763896$ \\
\bottomrule
\end{tabular}
\end{center}

\noindent
Two features of this table are worth reading off. First, the far-field
band is empty because $M_0=n_{\mathrm{near}}=153$ exceeds
$\lfloor\Theta_{\mathrm{ball}}/\Delta\theta\rfloor=41$: the analytic tail
begins exactly where the computed sum ends, which is the choice discussed
in the truncation bound and which at $k=5$ is worth more than an order of
magnitude in the final bound. Second, the entire chain descends from the
single rational $c$; the verifier never needs the neural stage, a GPU,
or any stored enclosure.

\noindent
Test~1 shows that an inflated claim cannot survive. Test~2 shows the
converse and less obvious half: a stored intermediate cannot be used to
strengthen a claim either, because the verifier does not read it. Only
\texttt{c\_exact} (guarded by the \texttt{sha256} against the discovery
record) and the computational settings enter, and altering the latter
changes only how sharp the recomputed bound is, never whether the
recomputation is valid.

\panel{C4. Independent reimplementation.}
The strongest form of the check does not use our code at all. Every
entry in the Panel~C2 table is an elementary formula in $c$, $P$,
$n_{\mathrm{near}}$ and $r_{\max}$, and can be reproduced in a few dozen
lines with any interval-arithmetic library. We have verified the present
example this way as well as through the shipped verifier.

%% file: Tables.tex
\section{Maynard--Tao Discovery and Certification Tables}

\begingroup
\renewcommand{\tablename}{Supplementary Table}

\begin{longtable}{r r r r r r r}
\caption{Vanilla discovery and certification results ($\epsilon = 0$).}
\label{tab:vanilla-results}\\

\toprule
$k$ & Discovery & $Ch_{\mathrm{disc}}$ & $Ch_{\mathrm{cert}}$ &
Degree & Certification & $Fr_{\mathrm{ceiling}}$ \\
\midrule
\endfirsthead

\multicolumn{7}{c}{\tablename\ \thetable\ -- continued from previous page} \\
\toprule
$k$ & Discovery & $Ch_{\mathrm{disc}}$ & $Ch_{\mathrm{cert}}$ &
Degree & Certification & $Fr_{\mathrm{ceiling}}$ \\
\midrule
\endhead

\midrule
\multicolumn{7}{r}{Continued on next page} \\
\endfoot

\bottomrule
\multicolumn{7}{@{}p{\dimexpr\textwidth-2\tabcolsep\relax}@{}}{
\footnotesize\textit{Note.}
Values $k = 25$ and $k > 30$ have certified values above Bogaert's Krylov subspace method and the published numbers on the Polymath8b Wiki. $Ch_{\mathrm{disc}}$; amount of channels in the discovery, $Ch_{\mathrm{cert}}$; amount of channels in the certification (after pruning). $Fr_{\mathrm{ceiling}}$; the gap to the Cauchy Schwartz ceiling.
}\\
\endlastfoot

20  & 3.128345991229  & 36 & 33 & 60 & 3.127557950496731089199291       & 0,992 \\
21  & 3.170106837298  & 36 & 30 & 60 & 3.169687744135085908710644       & 0,992 \\
22  & 3.210448579010  & 36 & 29 & 60 & 3.209791425249774301029605       & 0,991 \\
23  & 3.249812450248  & 36 & 30 & 60 & 3.249018102873413922834334       & 0,991 \\
24  & 3.287477573848  & 36 & 32 & 60 & 3.286239994066373110168097       & 0,991 \\
25  & 3.322508727072  & 36 & 28 & 60 & 3.322151135960770560379734       & 0,991 \\
26  & 3.358033109070  & 36 & 30 & 60 & 3.356593408977485398051369       & 0,991 \\
27  & 3.390797412499  & 36 & 30 & 60 & 3.389882327187024762678030       & 0.990 \\
28  & 3.423550501369  & 36 & 29 & 60 & 3.421891175367646121306316       & 0.990 \\
29  & 3.453912899535  & 36 & 33 & 60 & 3.452879774954988104889254       & 0.990 \\
30  & 3.484519905740  & 42 & 35 & 60 & 3.482920804908356499173898       & 0.990 \\
31  & 3.514154741570  & 42 & 35 & 60 & 3.512405412198264166419745       & 0.990 \\
32  & 3.542948097289  & 42 & 34 & 60 & 3.540478095710103535569937       & 0.990 \\
33  & 3.569408303647  & 42 & 35 & 60 & 3.568168082140816922057409       & 0.990 \\
34  & 3.595856014881  & 42 & 36 & 60 & 3.594862215205947831119258       & 0,989 \\
35  & 3.622637140460  & 42 & 36 & 60 & 3.620729784609115150023189       & 0,989 \\
36  & 3.649049376929  & 42 & 36 & 60 & 3.645939449133237314781111       & 0,989 \\
37  & 3.672649021837  & 42 & 36 & 60 & 3.670749358191800897953616       & 0,989 \\
38  & 3.697450514132  & 42 & 38 & 60 & 3.694212659617672335430314       & 0,989 \\
39  & 3.719968833432  & 42 & 37 & 60 & 3.718057515927273283488455       & 0,989 \\
40  & 3.743412144728  & 48 & 45 & 60 & 3.740850129189844929384497015046 & 0,989 \\
41  & 3.765478132249  & 48 & 39 & 60 & 3.763326248211363188195226978670 & 0,989 \\
42  & 3.786443640815  & 48 & 40 & 60 & 3.784836398710852090166363045263 & 0,989 \\
43  & 3.808393391526  & 48 & 44 & 60 & 3.806112899332449843513725348999 & 0,988 \\
44  & 3.831089249576  & 48 & 41 & 60 & 3.826734286114500842550587364787 & 0,988 \\
45  & 3.850548910793  & 48 & 44 & 60 & 3.847582396814691842088947101999 & 0,988 \\
46  & 3.870708253098  & 48 & 45 & 60 & 3.867062274460791570055053587204 & 0,988 \\
47  & 3.890981548956  & 48 & 42 & 60 & 3.887050039712155465458257739811 & 0,988 \\
48  & 3.911606976491  & 48 & 40 & 60 & 3.906150910904812934847425559032 & 0,988 \\
49  & 3.927617139801  & 48 & 42 & 60 & 3.925218177946628714063038556246 & 0,988 \\
50  & 3.943707217858  & 40 & 33 & 50 & 3.942445971786197344558734862254 & 0,988 \\
51  & 3.962201156342  & 40 & 34 & 50 & 3.961742909943176882384440577615 & 0,988 \\
52  & 3.980085920896  & 40 & 33 & 50 & 3.979361743352805048779887827229 & 0,988 \\
53  & 3.997392090465  & 40 & 35 & 50 & 3.996569546603520792189153865045 & 0,988 \\
54  & 4.014417349448  & 40 & 33 & 50 & 4.013421392728337468685880863364 & 0,987 \\
55  & 4.031158093047  & 40 & 33 & 50 & 4.029930315883373661648075171587 & 0,987 \\
56  & 4.048740196154  & 40 & 32 & 50 & 4.047211183498640314770985656114 & 0,987 \\
57  & 4.064675943291  & 40 & 33 & 50 & 4.063658911265930366015897090005 & 0,987 \\
58  & 4.080897241650  & 40 & 33 & 50 & 4.079662539803768949658734868762 & 0,987 \\
59  & 4.096107675723  & 40 & 33 & 50 & 4.094929027520673522821234872044 & 0,987 \\
60  & 4.111224221271  & 40 & 32 & 50 & 4.110244679900046307236517232052 & 0,987 \\
61  & 4.123180413889  & 40 & 37 & 50 & 4.122593396892570578487592704836 & 0,986 \\
62  & 4.141415806845  & 40 & 35 & 50 & 4.140421271440397388299091151343 & 0,987 \\
63  & 4.151452852615  & 40 & 37 & 50 & 4.150583732266163623322424042114 & 0,986 \\
64  & 4.170843665400  & 40 & 33 & 50 & 4.169733035246490856721916297730 & 0,987 \\
65  & 4.185117958392  & 40 & 36 & 50 & 4.184200456291009710938969100824 & 0,987 \\
66  & 4.199258111085  & 40 & 34 & 50 & 4.198212352631011423270651543778 & 0,987 \\
67  & 4.212801300650  & 40 & 32 & 50 & 4.211734375259938467462491656228 & 0,987 \\
68  & 4.226913707347  & 40 & 34 & 50 & 4.225671537232203572914298251005 & 0,987 \\
69  & 4.240641054970  & 40 & 27 & 50 & 4.239530843403299471806859841686 & 0,987 \\
70  & 4.252043044524  & 40 & 34 & 50 & 4.250709595458423509274271903750 & 0,986 \\
71  & 4.262077283275  & 40 & 34 & 50 & 4.260781471180062527612810820600 & 0,985 \\
72  & 4.280170671493  & 40 & 31 & 50 & 4.278486659483071740524507466677 & 0,987 \\
73  & 4.291652331079  & 40 & 36 & 50 & 4.290439906377909296284295483681 & 0,986 \\
74  & 4.302039567418  & 40 & 32 & 50 & 4.300766657962825879583571273685 & 0,986 \\
75  & 4.315037290653  & 40 & 36 & 50 & 4.313819804163989594408982684260 & 0,986 \\
76  & 4.330839005745  & 40 & 33 & 50 & 4.328884563460976817456708107930 & 0,986 \\
77  & 4.342588581101  & 40 & 32 & 50 & 4.340903139563253839977695048758 & 0,986 \\
78  & 4.353490658660   & 40 & 30 & 50 & 4.351450026525314703393729301722 & 0,986 \\
79  & 4.366332032220  & 40 & 25 & 50 & 4.363547069467314352935761632810 & 0,986 \\
80  & 4.375870655363  & 50 & 36 & 50 & 4.375305575875195280796809950641 & 0,986 \\
81  & 4.388391084261  & 50 & 34 & 50 & 4.387822836632632171272303086405 & 0,986 \\
82  & 4.394774870701  & 50 & 40 & 50 & 4.394326901844365176252467034760 & 0,985 \\
83  & 4.401879201937  & 50 & 43 & 50 & 4.401472671518373508912545234614 & 0,984 \\
84  & 4.417373705040  & 50 & 40 & 50 & 4.416953094949757411430914506920 & 0,985 \\
85  & 4.346441814015  & 50 & 41 & 50 & 4.346494488379658185500563460563 & 0,967 \\
86  & 4.431105148849  & 50 & 42 & 50 & 4.430626969481280584146214724073 & 0,983 \\
87  & 4.452908961195  & 50 & 40 & 50 & 4.451913396990625293127698559451 & 0,985 \\
88  & 4.464302979612  & 50 & 41 & 50 & 4.463426157257708417355375315939 & 0,986 \\
89  & 4.472355081593  & 50 & 42 & 50 & 4.471988530980810754826227770313 & 0,985 \\
90  & 4.487119371775  & 50 & 37 & 50 & 4.485910734210900551045280545626 & 0,986 \\
91  & 4.496698527592  & 50 & 37 & 50 & 4.495645093532629510509939471985 & 0,986 \\
92  & 4.505829693079  & 50 & 39 & 50 & 4.504547705976628150465431574044 & 0,985 \\
93  & 4.518347339954  & 50 & 36 & 50 & 4.516339208377947949080443322071 & 0,986 \\
94  & 4.528059040470  & 50 & 29 & 50 & 4.526635937140459614406972086093 & 0,986 \\
95  & 4.533806772177  & 50 & 41 & 50 & 4.532436225927372584839614117923 & 0,985 \\
96  & 4.545404066027  & 50 & 26 & 50 & 4.544084221152406068483876743217 & 0,985 \\
97  & 4.554320603674  & 50 & 31 & 50 & 4.553588807459129165253574203320 & 0,985 \\
98  & 4.565278947850  & 50 & 28 & 50 & 4.564004527939718801984156460497 & 0,985 \\
99  & 4.509354418861  & 50 & 35 & 50 & 4.508561572862963088224836065494 & 0,971 \\
100 & 4.585389371115  & 50 & 39 & 50 & 4.583559169120379861185947389471 & 0,985 \\
100 & 4.582636529114  & 50 & 39 & 50 & 4.581993985670381852509553180491 & 0,985 \\
130 & 4.822341520115  & 50 & 13 & 50 & 4.808713824125718836287734368717 & 0.980 \\
150 & 4.961245663975  & 50 & 6  & 50 & 4.955366718182981077061200798091 & 0.982 \\
200 & 5.090813348651  & 50 & 4  & 60 & 5.087280688405278804164639536817 & 0.955 \\
220 & 5.258389442417  & 50 & 12 & 60 & 5.247997786415732591367214256322 & 0.969 \\
250 & 5.384166994782  & 50 & 5  & 60 & 5.349929087705455637995951938908 & 0.965 \\
350 & 5.610666510192  & 50 & 3  & 60 & 4.921779442483077605884943545259 & 0.838 \\
500 & 6.080822857000  & 50 & 3  & 60 & 5.032655240858842848757497207358 & 0.808 \\

\end{longtable}
\endgroup

\begingroup
\renewcommand{\tablename}{Supplementary Table}

\begin{longtable}{r r r r r r r}
\caption{Discovery and certification results for $\varepsilon = 1/25$.}
\label{tab:epsilon-results}\\

\toprule
$k$ & Discovery & $Ch_{\mathrm{disc}}$ & $Ch_{\mathrm{cert}}$ &
Degree & Certification \\
\midrule
\endfirsthead

\multicolumn{7}{c}{\tablename\ \thetable\ -- continued from previous page} \\
\toprule
$k$ & Discovery & $Ch_{\mathrm{disc}}$ & $Ch_{\mathrm{cert}}$ &
Degree & Certification \\
\midrule
\endhead

\midrule
\multicolumn{7}{r}{Continued on next page} \\
\endfoot

\bottomrule
\multicolumn{7}{@{}p{\dimexpr\textwidth-2\tabcolsep\relax}@{}}{
\footnotesize\textit{Note.}
$Ch_{\mathrm{disc}}$; amount of channels in the discovery, $Ch_{\mathrm{cert}}$; amount of channels in the certification (after pruning).
}\\
\endlastfoot

20  & 3.210936069786 & 36 & 33 & 60 & 3.210366921521749737183691731395 \\
21  & 3.252680335644 & 36 & 30 & 60 & 3.252298834127242084124289323548 \\
22  & 3.292779266722 & 36 & 29 & 60 & 3.292126369692432458798975243945 \\
23  & 3.331173442193 & 36 & 30 & 60 & 3.330322094752801530355410455276 \\
24  & 3.368036770635 & 36 & 30 & 60 & 3.366908489372213299919519038109 \\
25  & 3.402371365754 & 36 & 29 & 60 & 3.402010619443212328887383163464 \\
26  & 3.437028098636 & 36 & 31 & 60 & 3.436157517278351176844746714320 \\
27  & 3.468827024244 & 36 & 29 & 60 & 3.468320542253183831988945583237 \\
28  & 3.500415876452 & 36 & 30 & 60 & 3.499444954997692030882286887925 \\
29  & 3.531186516502 & 36 & 32 & 60 & 3.530010877716944327578014764437 \\
30  & 3.561700179340 & 42 & 39 & 60 & 3.559458746545551314730876537135 \\
31  & 3.590689744972 & 42 & 40 & 60 & 3.588557568639265632860385653635 \\
32  & 3.618732963879 & 42 & 38 & 60 & 3.616127200711337138001580976885 \\
33  & 3.644253262445 & 42 & 36 & 60 & 3.642965802069061441160922851345 \\
34  & 3.671450599407 & 42 & 39 & 60 & 3.668624924763460294356636286090 \\
35  & 3.696250084782 & 42 & 40 & 60 & 3.694175287948157320412197179779 \\
36  & 3.720908518093 & 42 & 39 & 60 & 3.718279485810969785557788969075 \\
37  & 3.744273426646 & 42 & 38 & 60 & 3.742226163630752703816176588194 \\
38  & 3.768134969368 & 42 & 41 & 60 & 3.766021610226650587509688536386 \\
39  & 3.790213730678 & 42 & 39 & 60 & 3.788245621497978381906985116415 \\
40  & 3.813521764234 & 48 & 45 & 60 & 3.810310054297441303862285913738 \\
41  & 3.833616627961 & 48 & 41 & 60 & 3.831637106843201917337701123609 \\
42  & 3.855906628402 & 48 & 45 & 60 & 3.852600769154481058291704449249 \\
43  & 3.876753550031 & 48 & 42 & 60 & 3.873483797869255723761005318927 \\
44  & 3.895577913875 & 48 & 46 & 60 & 3.893301604331862091582323108956 \\
45  & 3.915649733022 & 48 & 45 & 60 & 3.913474652772154201185956494902 \\
46  & 3.936937139139 & 48 & 43 & 60 & 3.932552485373460445967272733645 \\
47  & 3.956703033242 & 48 & 46 & 60 & 3.952079215054289455935133148337 \\
48  & 3.975823596168 & 48 & 44 & 60 & 3.970095831528526602285865965250 \\
49  & 3.991450730929 & 48 & 43 & 60 & 3.988676716542276674807876834435 \\
50  & 4.006026386327 & 40 & 32 & 50 & 4.005637471340251334822223632678 \\
51  & 4.021337913283 & 40 & 21 & 50 & 4.020552375066701497376356751144 \\
52  & 4.040635119278 & 40 & 31 & 50 & 4.040258681163316069160504782304 \\
53  & 4.058215551546 & 40 & 24 & 50 & 4.057203993845967521227902017906 \\
54  & 4.072969225060 & 40 & 28 & 50 & 4.072322886063011063853846972776 \\
55  & 4.084802202882 & 40 & 18 & 50 & 4.084085711139175414397956372460 \\
56  & 4.100028783679 & 40 & 26 & 50 & 4.099206855766737568002165290524 \\
57  & 4.115585844764 & 40 & 23 & 50 & 4.115039486353743806900256901164 \\
58  & 4.135152191312 & 40 & 25 & 50 & 4.134242300759367593439951066209 \\
59  & 4.149187895198 & 40 & 30 & 50 & 4.148472609734504851305763152187 \\
60  & 4.158966932600 & 40 & 19 & 50 & 4.157998757293132279032736615835 \\
61  & 4.175975507879 & 40 & 31 & 50 & 4.175077906373857205061979626025 \\
62  & 4.189862734973 & 40 & 32 & 50 & 4.188904070211809938022792508286 \\
63  & 4.210516336726 & 40 & 33 & 50 & 4.208857616045980796408818491228 \\
64  & 4.219415161437 & 40 & 30 & 50 & 4.218124437987688545452024305223 \\
65  & 4.238501800615 & 40 & 36 & 50 & 4.237384124837738623814213247461 \\
66  & 4.252120545298 & 40 & 29 & 50 & 4.250793037094066225539732443814 \\
67  & 4.261402748608 & 40 & 22 & 50 & 4.259777456025320664617972564576 \\
68  & 4.270379618942 & 40 & 26 & 50 & 4.269124177371744769988813084968 \\
69  & 4.283984925481 & 40 & 5  & 50 & 4.277274261946282382619508640328 \\
70  & 4.301560117306 & 40 & 34 & 50 & 4.298733418769167808786843497725 \\
71  & 4.316885083242 & 40 & 31 & 50 & 4.314413443106201423239218270654 \\
72  & 4.329503334605 & 40 & 35 & 50 & 4.327034144580560001136968528718 \\
73  & 4.334908974772 & 40 & 30 & 50 & 4.332768096731747934119777486576 \\
74  & 4.342039995848 & 40 & 28 & 50 & 4.339799356192207942523849650386 \\
75  & 4.358917246793 & 40 & 32 & 50 & 4.356375192119249086787668845356 \\
76  & 4.375902255118 & 40 & 31 & 50 & 4.372974259197422394873126180262 \\
77  & 4.388092689456 & 40 & 35 & 50 & 4.385396831879956294549952588437 \\
78  & 4.399184535712 & 40 & 22 & 50 & 4.395990292345199428975680930752 \\
79  & 4.407214416241 & 40 & 30 & 50 & 4.404183984805812816254054716642 \\
80  & 4.195941988767 & 50 & 40 & 50 & 4.197608268366179620141468586531 \\
81  & 4.075117366264 & 50 & 35 & 50 & 4.093066546579362186151057435374 \\
82  & 4.192978689328 & 50 & 43 & 50 & 4.195259401348774146934234241013 \\
83  & 4.425644201176 & 50 & 35 & 50 & 4.424852131550489892382995112584 \\
84  & 4.096825473475 & 50 & 31 & 50 & 4.099802658570555784696980423382 \\
85  & 4.407475205627 & 50 & 39 & 50 & 4.406760295616284828171584604987 \\
86  & 4.457072509121 & 50 & 31 & 50 & 4.458412761836136236841801895336 \\
87  & 4.453265265989 & 50 & 45 & 50 & 4.453060114331590189551387134208 \\
88  & 4.344434571908 & 50 & 40 & 50 & 4.344422869750479566119971669552 \\
89  & 4.410538854701 & 50 & 39 & 50 & 4.417151906430070350388796841891 \\
90  & 4.470745463723 & 50 & 45 & 50 & 4.471149582958362588083170625461 \\
91  & 4.175121841242 & 50 & 41 & 50 & 4.176550564513205748419816639538 \\
92  & 4.381872185197 & 50 & 45 & 50 & 4.389037025811565058867953246127 \\
93  & 4.312917570269 & 50 & 44 & 50 & 4.314963190359444288509182754059 \\
94  & 4.272652151109 & 50 & 38 & 50 & 4.273951035321736008358360794466 \\
95  & 4.420217255929 & 50 & 38 & 50 & 4.419415306457317943090466947241 \\
96  & 4.311614804044 & 50 & 34 & 50 & 4.311428453792439589002649755954 \\
97  & 4.203613002767 & 50 & 35 & 50 & 4.206802724014883123574465837334 \\
98  & 4.217114551052 & 50 & 39 & 50 & 4.225342000024770175903490961201 \\
99  & 4.548298852384 & 50 & 39 & 50 & 4.546923190682762183944464394853 \\
100 & 4.496536005802 & 50 & 42 & 50 & 4.495512133810554663324778495290 \\
100 & 4.466983632765 & 50 & 40 & 50 & 4.466802522812985639158255185770 \\
130 & 4.847986562151 & 50 & 10 & 50 & 4.844156480287745811911814959650 \\
150 & 4.977475423203 & 50 & 8  & 50 & 4.961947334363613966866801798777 \\
200 & 5.094346768152 & 50 & 7  & 60 & 5.087393690044834674577815885368 \\
220 & 4.978911765299 & 50 & 8  & 60 & 4.976764489557045481952499319142 \\
220 & 4.978911765299 & 50 & 8  & 60 & 4.976764489557045481952499319142 \\
300 & 4.086884354936 & 50 & 1  & 60 & 3.689963955854296035616007427184 \\
400 & 5.381948554900 & 50 & 13 & 60 & 5.342988749384368640747405260075 \\
500 & 4.512115207075 & 50 & 14 & 60 & 4.192789933536702947199563856796 \\

\end{longtable}
\endgroup

\newpage
\begingroup
\renewcommand{\tablename}{Supplementary Table}

\begin{longtable}{r c l r}
\caption{\textbf{Certified enlarged-support sweep underlying the
\(\varepsilon\)-crossover analysis.}
For each \(k\), the table reports the certified lower bound obtained for the
corresponding enlarged-support parameter \(\varepsilon\), together with its
within-\(k\) ranking among the tested values. Rank \(1\) denotes the strongest
certified value at that \(k\).}
\label{tab:epsilon_sweep}\\

\toprule
\(k\) & \(\varepsilon\) & Certified value & Rank \\
\midrule
\endfirsthead

\multicolumn{4}{c}%
{{\tablename\ \thetable{} -- continued from previous page}}\\
\toprule
\(k\) & \(\varepsilon\) & Certified value & Rank \\
\midrule
\endhead

\midrule
\multicolumn{4}{r}{{Continued on next page}}\\
\endfoot

\bottomrule
\endlastfoot

50  & \(0\)    & 3.942127742490453604663761985558 & 5 \\
50  & \(1/5\)  & 3.712509239545921204317529071368 & 8 \\
50  & \(1/6\)  & 3.764504991324559702999466932710 & 7 \\
50  & \(1/8\)  & 3.874384180153464414710468398524 & 6 \\
50  & \(1/12\) & 3.959088983406359921645433483059 & 4 \\
50  & \(1/16\) & 3.988718438122563563392522885252 & 3 \\
50  & \(1/25\) & 3.997997740337627844075583884562 & 1 \\
50  & \(1/50\) & 3.994094475900581244335373138334 & 2 \\

\addlinespace
60  & \(0\)    & 4.107867690167866055301464199575 & 4 \\
60  & \(1/5\)  & 3.823138358991218145879633189922 & 8 \\
60  & \(1/6\)  & 3.901108973124642823154420824199 & 7 \\
60  & \(1/8\)  & 4.008638429188118037606214685517 & 6 \\
60  & \(1/12\) & 4.078919312220589574136195730201 & 5 \\
60  & \(1/16\) & 4.137429406260009516632280091156 & 3 \\
60  & \(1/25\) & 4.154030210268053992886263760513 & 2 \\
60  & \(1/50\) & 4.159671531351647672765832327114 & 1 \\

\addlinespace
70  & \(0\)    & 4.250768996013457035617670337865 & 4 \\
70  & \(1/5\)  & 3.932874621775866265776104392099 & 8 \\
70  & \(1/6\)  & 3.992467605762845919987563382701 & 7 \\
70  & \(1/8\)  & 4.097181032238870681623243665487 & 6 \\
70  & \(1/12\) & 4.203873349967048026646358848602 & 5 \\
70  & \(1/16\) & 4.264751558230381563655123754834 & 3 \\
70  & \(1/25\) & 4.284181954725759229768882824961 & 2 \\
70  & \(1/50\) & 4.298346167280743718826612134536 & 1 \\

\addlinespace
80  & \(0\)    & 4.365076286689542864035828077954 & 4 \\
80  & \(1/5\)  & 4.010739510791432389610499056194 & 8 \\
80  & \(1/6\)  & 4.046713573456469114235727637437 & 7 \\
80  & \(1/8\)  & 4.198433911029465756611969501053 & 6 \\
80  & \(1/12\) & 4.300715218425578908913133664543 & 5 \\
80  & \(1/16\) & 4.368846840474345630380842204395 & 3 \\
80  & \(1/25\) & 4.397641633229676353759931557760 & 2 \\
80  & \(1/50\) & 4.409932699888494327027154402066 & 1 \\

\addlinespace
90  & \(0\)    & 4.475394178930606300302930028315 & 3 \\
90  & \(1/5\)  & 4.051193459293446704008504024550 & 7 \\
90  & \(1/6\)  & 4.116620710367156118853873413986 & 6 \\
90  & \(1/8\)  & 4.026133861058892236358880667882 & 8 \\
90  & \(1/12\) & 4.277640742111484913450494485256 & 5 \\
90  & \(1/16\) & 4.426300947743744344607385400548 & 4 \\
90  & \(1/25\) & 4.502108322122445717967006768742 & 2 \\
90  & \(1/50\) & 4.514812668499281034377537700573 & 1 \\

\addlinespace
100 & \(0\)    & 4.571865161766055920977957342516 & 2 \\
100 & \(1/5\)  & 4.131174966202319094573252086929 & 8 \\
100 & \(1/6\)  & 4.228661930221300547990282980216 & 7 \\
100 & \(1/8\)  & 4.376164198382763272282279872439 & 6 \\
100 & \(1/12\) & 4.455747109584966143321615617427 & 5 \\
100 & \(1/16\) & 4.483072604316316381523061824105 & 4 \\
100 & \(1/25\) & 4.520592767325824976955743719607 & 3 \\
100 & \(1/50\) & 4.600398252159768512024806128765 & 1 \\

\end{longtable}
\endgroup

%% file: ncs-refs.bib
@article{maynard2015small,
  title={Small gaps between primes},
  author={Maynard, James},
  journal={Annals of mathematics},
  pages={383--413},
  year={2015},
  publisher={JSTOR}
}

@inproceedings{hales2017formal,
  title={A formal proof of the Kepler conjecture},
  author={Hales, Thomas and Adams, Mark and Bauer, Gertrud and Dang, Tat Dat and Harrison, John and Hoang, Le Truong and Kaliszyk, Cezary and Magron, Victor and McLaughlin, Sean and Nguyen, Tat Thang and others},
  booktitle={Forum of mathematics, Pi},
  volume={5},
  pages={e2},
  year={2017},
  organization={Cambridge University Press}
}

@inproceedings{tan2024formally,
  title={Formally certified approximate model counting},
  author={Tan, Yong Kiam and Yang, Jiong and Soos, Mate and Myreen, Magnus O and Meel, Kuldeep S},
  booktitle={International Conference on Computer Aided Verification},
  pages={153--177},
  year={2024},
  organization={Springer}
}

@article{trinh2024solving,
  title={Solving olympiad geometry without human demonstrations},
  author={Trinh, Trieu H and Wu, Yuhuai and Le, Quoc V and He, He and Luong, Thang},
  journal={Nature},
  volume={625},
  number={7995},
  pages={476--482},
  year={2024},
  publisher={Nature Publishing Group UK London}
}

@article{udrescu2020ai,
  title={AI Feynman: A physics-inspired method for symbolic regression},
  author={Udrescu, Silviu-Marian and Tegmark, Max},
  journal={Science advances},
  volume={6},
  number={16},
  pages={eaay2631},
  year={2020},
  publisher={American Association for the Advancement of Science}
}

@misc{Sutherland-database,
  author       = {Engelsma, Thomas J. and Sutherland, Andrew V.},
  title        = {Narrow Admissible Tuples},
  year         = {2013},
  howpublished = {\url{https://math.mit.edu/~primegaps/}},
  note         = {Entry for $k=3655$ submitted by A.~V. Sutherland,
                  27 June 2013.}
}

@misc{stadlmann2026bounded,
  author        = {Stadlmann, Julia},
  title         = {Bounded gaps between primes},
  year          = {2026},
  eprint        = {2608.31126},
  archivePrefix = {arXiv},
  primaryClass  = {math.NT},
  url           = {https://arxiv.org/abs/2608.31126},
  note          = {Submitted 31 August 2026}
}

@article{johansson2017arb,
  title={Arb: efficient arbitrary-precision midpoint-radius interval arithmetic},
  author={Johansson, Fredrik},
  journal={IEEE Transactions on Computers},
  volume={66},
  number={8},
  pages={1281--1292},
  year={2017},
  publisher={IEEE}
}

@misc{charton2026new,
  author       = {Charton, Fran\c{c}ois and Hong, Letong and Lau, Kenny and
                  Ono, Ken and Remy, Guillaume and Siu, Ho Chung and
                  Swaminathan, Ashvin A. and Thorner, Jesse and Xie, Yunzhou},
  title         = {A new bound for small gaps between primes},
  year          = {2026},
  howpublished  = {Preliminary draft, Axiom Math},
  note          = {Dated 3 September 2026. Lean formalization at
                   \url{https://github.com/AxiomMath/PrimeGapsLib}.
                   Accessed 14 September 2026}
}

@misc{openai2026improved,
  author        = {{OpenAI}},
  title         = {Improved short gaps between primes},
  year          = {2026},
  howpublished  = {\url{https://cdn.openai.com/pdf/51126fac-1b68-4128-9666-c908bcc16033/short_gaps.pdf}},
  note          = {Dated 30 August 2026. Proof attributed to GPT-6 Astra;
                   Lean~4 formalization and numerical certificate at
                   \url{https://github.com/openai/PrimeGaps186}.
                   Accessed 14 September 2026}
}

@article{raayoni2021generating,
  title={Generating conjectures on fundamental constants with the Ramanujan Machine},
  author={Raayoni, Gal and Gottlieb, Shahar and Manor, Yahel and Pisha, George and Harris, Yoav and Mendlovic, Uri and Haviv, Doron and Hadad, Yaron and Kaminer, Ido},
  journal={Nature},
  volume={590},
  number={7844},
  pages={67--73},
  year={2021},
  publisher={Nature Publishing Group UK London}
}

@article{cohn2022sign,
  title={Sign uncertainty principles and low-degree polynomials},
  author={Cohn, Henry and Dong, Dingding and Gon{\c{c}}alves, Felipe},
  journal={arXiv preprint arXiv:2210.01684},
  year={2022}
}

@article{coregliano2025higher,
  title={Higher-order delsarte dual lps: Lifting, constructions and completeness},
  author={Coregliano, Leonardo Nagami and Jeronimo, Fernando Granha and Jones, Chris and Linial, Nati and Loyfer, Elyassaf},
  journal={arXiv preprint arXiv:2501.04854},
  year={2025}
}

@article{loyfer2023new,
  title={New LP-based upper bounds in the rate-vs.-distance problem for binary linear codes},
  author={Loyfer, Elyassaf and Linial, Nati},
  journal={IEEE Transactions on Information Theory},
  volume={69},
  number={5},
  pages={2886--2899},
  year={2023},
  publisher={IEEE}
}

@article{Feynman1939,
  author  = {Feynman, Richard P.},
  title   = {Forces in Molecules},
  journal = {Physical Review},
  volume  = {56},
  pages   = {340--343},
  year    = {1939},
  doi     = {10.1103/PhysRev.56.340}
}

@article{LiuNocedal1989,
  author  = {Liu, Dong C. and Nocedal, Jorge},
  title   = {On the Limited Memory BFGS Method for Large Scale Optimization},
  journal = {Mathematical Programming},
  volume  = {45},
  pages   = {503--528},
  year    = {1989},
  doi     = {10.1007/BF01589116}
}

@article{flint,
  author  = {Hart, William and Johansson, Fredrik and Pancratz, Sebastian},
  title   = {FLINT: Fast Library for Number Theory},
  journal = {ACM Communications in Computer Algebra},
  volume  = {46},
  number  = {3/4},
  pages   = {88--93},
  year    = {2012}
}

@misc{Ghadimi2025,
  author        = {Ghadimi, Milad},
  title         = {Heuristic Bounded Prime Gaps via a Chaotic Multidimensional Sieve and Random Matrix Theory},
  year          = {2025},
  eprint        = {2507.17986},
  archivePrefix = {arXiv},
  primaryClass  = {math.NT},
  note          = {arXiv:2507.17986v1}
}

@book{GolubVanLoan2013,
  author    = {Golub, Gene H. and Van Loan, Charles F.},
  title     = {Matrix Computations},
  edition   = {4},
  publisher = {Johns Hopkins University Press},
  year      = {2013}
}

@article{ProbstAlagar1979,
  author  = {Probst, David K. and Alagar, V. S.},
  title   = {A Family of Algorithms for Powering Sparse Polynomials},
  journal = {SIAM Journal on Computing},
  volume  = {8},
  number  = {4},
  pages   = {626--644},
  year    = {1979},
  doi     = {10.1137/0208050}
}

@book{Parlett1998,
  author    = {Parlett, Beresford N.},
  title     = {The Symmetric Eigenvalue Problem},
  publisher = {SIAM},
  year      = {1998}
}

@article{GolubWelsch1969,
  author  = {Golub, Gene H. and Welsch, John H.},
  title   = {Calculation of Gauss Quadrature Rules},
  journal = {Mathematics of Computation},
  volume  = {23},
  number  = {106},
  pages   = {221--230},
  year    = {1969},
  doi     = {10.1090/S0025-5718-69-99647-1}
}

@inproceedings{KingmaBa2015,
  author    = {Kingma, Diederik P. and Ba, Jimmy},
  title     = {Adam: A Method for Stochastic Optimization},
  booktitle = {International Conference on Learning Representations},
  year      = {2015}
}

@article{semay2015hellmann,
  title={The Hellmann--Feynman theorem, the comparison theorem, and the envelope theory},
  author={Semay, Claude},
  journal={Results in Physics},
  volume={5},
  pages={322--323},
  year={2015},
  publisher={Elsevier}
}

@article{cohn2019optimal,
  title={An optimal uncertainty principle in twelve dimensions via modular forms: H. Cohn, F. Gon{\c{c}}alves},
  author={Cohn, Henry and Gon{\c{c}}alves, Felipe},
  journal={Inventiones mathematicae},
  volume={217},
  number={3},
  pages={799--831},
  year={2019},
  publisher={Springer}
}

@inproceedings{bourgain2010principe,
  title={Principe d’Heisenberg et fonctions positives},
  author={Bourgain, Jean and Clozel, Laurent and Kahane, Jean-Pierre},
  booktitle={Annales de l'institut Fourier},
  volume={60},
  number={4},
  pages={1215--1232},
  year={2010}
}

@misc{polymath2014variantsselbergsievebounded,
      title={Variants of the Selberg sieve, and bounded intervals containing many primes}, 
      author={D. H. J. Polymath},
      year={2014},
      eprint={1407.4897},
      archivePrefix={arXiv},
      primaryClass={math.NT},
      url={https://arxiv.org/abs/1407.4897}, 
}

@article{stadlmann2025primes,
  title={On primes in arithmetic progressions and bounded gaps between many primes},
  author={Stadlmann, Julia},
  journal={Advances in Mathematics},
  volume={468},
  pages={110190},
  year={2025},
  publisher={Elsevier}
}

@article{davies2021advancing,
  author  = {Davies, Alex and Veli{\v{c}}kovi{\'c}, Petar and Buesing, Lars
             and Blackwell, Sam and Zheng, Daniel and Toma{\v{s}}ev, Nenad
             and Tanburn, Richard and Battaglia, Peter and Blundell, Charles
             and Juh{\'a}sz, Andr{\'a}s and Lackenby, Marc and Williamson, Geordie},
  title   = {Advancing mathematics by guiding human intuition with {AI}},
  journal = {Nature},
  volume  = {600},
  pages   = {70--74},
  year    = {2021},
  doi     = {10.1038/s41586-021-04086-x}
}

@article{fawzi2022alphatensor,
  author  = {Fawzi, Alhussein and Balog, Matej and Huang, Aja
             and Hubert, Thomas and Romera-Paredes, Bernardino
             and Barekatain, Mohammadamin and Novikov, Alexander
             and Ruiz, Francisco J. R. and Schrittwieser, Julian
             and Swirszcz, Grzegorz and Silver, David
             and Hassabis, Demis and Kohli, Pushmeet},
  title   = {Discovering faster matrix multiplication algorithms
             with reinforcement learning},
  journal = {Nature},
  volume  = {610},
  pages   = {47--53},
  year    = {2022},
  doi     = {10.1038/s41586-022-05172-4}
}

@article{romeraparedes2024funsearch,
  author  = {Romera-Paredes, Bernardino and Barekatain, Mohammadamin
             and Novikov, Alexander and Balog, Matej
             and Kumar, M. Pawan and Dupont, Emilien
             and Ruiz, Francisco J. R. and Ellenberg, Jordan S.
             and Wang, Pengming and Fawzi, Omar
             and Kohli, Pushmeet and Fawzi, Alhussein},
  title   = {Mathematical discoveries from program search
             with large language models},
  journal = {Nature},
  volume  = {625},
  pages   = {468--475},
  year    = {2024},
  doi     = {10.1038/s41586-023-06924-6}
}

@inproceedings{mundinger2025neural,
  author    = {Mundinger, Konrad and Zimmer, Max and Kiem, Aldo
               and Spiegel, Christoph and Pokutta, Sebastian},
  title     = {Neural Discovery in Mathematics:
               Do Machines Dream of Colored Planes?},
  booktitle = {Proceedings of the 42nd International Conference
               on Machine Learning},
  series    = {Proceedings of Machine Learning Research},
  volume    = {267},
  year      = {2025},
  publisher = {PMLR}
}
